\documentclass[11pt]{article}
\usepackage{arxiv}
\usepackage[utf8]{inputenc}
\usepackage[T1]{fontenc}
\usepackage{microtype}
\usepackage{amsmath,amssymb,amsfonts,amsthm,mathtools}
\usepackage{enumitem}
\usepackage{booktabs}
\usepackage{tabularx}
\usepackage{array}
\usepackage{float}
\usepackage{xcolor}
\usepackage[colorlinks=true,linkcolor=blue,citecolor=blue,urlcolor=blue]{hyperref}
\usepackage[nameinlink,noabbrev]{cleveref}
\hypersetup{hidelinks}

\usepackage{todonotes}

\title{High-Probability PL-SGD with Markovian Noise: Optimal Mixing and Tail Dependence}

\author{%
  Dhruv Sarkar\textsuperscript{1,2} \quad
  Aprameyo Chakrabartty\textsuperscript{1} \quad
  Vaneet Aggarwal\textsuperscript{3} \\[3pt]
  \normalfont\small\textsuperscript{1}Indian Institute of Technology Kharagpur \\[-1pt]
  \normalfont\small\textsuperscript{2}Mohamed bin Zayed University of Artificial Intelligence \\[-1pt]
  \normalfont\small\textsuperscript{3} Purdue University \\[2pt]
  \normalfont\small\texttt{dhruv.sarkar223@gmail.com} \quad
  \normalfont\small\texttt{aprameyo8858@gmail.com} \quad\texttt{vaneet@purdue.edu} \\[2pt]
}

\date{}

\newtheorem{theorem}{Theorem}
\newtheorem{lemma}{Lemma}
\newtheorem{corollary}{Corollary}
\newtheorem{assumption}{Assumption}
\newtheorem{remark}{Remark}
\newtheorem{definition}{Definition}
\crefname{theorem}{Theorem}{Theorems}
\Crefname{theorem}{Theorem}{Theorems}
\crefname{corollary}{Corollary}{Corollaries}
\Crefname{corollary}{Corollary}{Corollaries}
\crefname{proposition}{Proposition}{Propositions}
\Crefname{proposition}{Proposition}{Propositions}
\crefname{lemma}{Lemma}{Lemmas}
\Crefname{lemma}{Lemma}{Lemmas}
\crefname{assumption}{Assumption}{Assumptions}
\Crefname{assumption}{Assumption}{Assumptions}
\crefname{remark}{Remark}{Remarks}
\Crefname{remark}{Remark}{Remarks}
\crefname{equation}{equation}{equations}
\Crefname{equation}{Equation}{Equations}
\newcommand{\appref}[1]{Appendix~\ref{#1}}

\newcommand{\R}{\mathbb{R}}
\newcommand{\N}{\mathbb{N}}
\newcommand{\E}{\mathbb{E}}
\newcommand{\Pp}{\mathbb{P}}
\newcommand{\Prob}{\mathbb{P}}
\newcommand{\calF}{\mathcal{F}}
\newcommand{\calE}{\mathcal{E}}
\newcommand{\calZ}{\mathcal{Z}}

\newcommand{\Zspace}{\mathsf Z}
\newcommand{\1}{\mathbf{1}}

\newcommand{\ip}[2]{\left\langle #1,#2\right\rangle}
\newcommand{\norm}[1]{\left\lVert #1\right\rVert}
\newcommand{\abs}[1]{\left\lvert #1\right\rvert}
\newcommand{\TV}{\mathrm{TV}}
\newcommand{\tv}{\mathrm{TV}}
\newcommand{\tmix}{t_{\mathrm{mix}}}
\newcommand{\Del}{\Delta}

\newcolumntype{Y}{>{\raggedright\arraybackslash}X}

\begin{document}
\maketitle

\begin{abstract}
We study first-order methods for smooth objectives satisfying the Polyak-\L{}ojasiewicz (PL) condition when gradient samples are generated by an exogenous Markov chain. In the light-tailed setting, prior uniform-in-time high-probability bounds for ordinary Stochastic Gradient Descent (SGD) under a standard growth envelope scale as \(\widetilde O(\tmix^2/k)\), leaving a gap with the \(\widetilde O(\tmix/k)\) expectation bounds. We close this gap using a lag-blocking argument to establish a uniform high-probability guarantee with a leading stochastic term of \(\widetilde O(\tmix/(k+K_0))\) under geometric mixing. We prove this linear dependence on the mixing time is optimal via a matching \(\Omega(\sigma^2 \tmix/k)\) lower bound on a quadratic objective driven by a persistent two-state chain.

We then extend this framework to heavy-tailed Markovian gradients satisfying a stationary finite-\(p\)-moment condition, \(p \in (1,2]\). We design an all-samples clipped block method that uses every Markov transition while mitigating Markovian bias. Under a transition budget \(T\), this algorithm achieves a high-probability stochastic error of \(\widetilde O(\sigma_p^2(\tmix/T)^{2(p-1)/p})\). We establish a matching lower bound by reducing PL optimization to heavy-tailed mean estimation for a sticky Markov chain. Ultimately, this work tightly characterizes the optimal polynomial dependence on mixing time for light-tailed PL-SGD, and the optimal heavy-tail exponent and effective-sample-size dependence in the robust regime.
\end{abstract}

\begin{center}
\end{center}

\section{Introduction}

Stochastic gradient methods are a central tool in stochastic approximation and large-scale
optimization.  Classical analyses of stochastic approximation and Stochastic Gradient Descent (SGD) typically rely on either
independent samples or conditionally unbiased martingale-difference gradient noise
\cite{robbins1951stochastic,kiefer1952stochastic,benveniste1990adaptive,
kushner2003stochastic,borkar2008stochastic,bottou2018optimization}.  Under these assumptions,
the stochastic error in the descent recursion can be controlled by martingale or concentration
arguments, and nonasymptotic convergence guarantees are now well understood in many convex,
strongly convex, and nonconvex regimes
\cite{moulines2011nonasymptotic,gower2019sgd,khaled2023better}.  In many modern applications,
however, the samples used by the gradient oracle are generated sequentially by a Markov chain rather
than independently.  This occurs in token-based and random-walk decentralized optimization
\cite{johansson2007randomized,mao2020walkman,hendrikx2023token}, Markov Chain Monte Carlo
gradient estimation, privacy-preserving subsampling schemes with temporal exclusion rules
\cite{abadi2016deep,choquettechoo2023bandmf,dong2026minsep}, online system identification
\cite{kowshik2021streaming}, and reinforcement-learning and temporal-difference algorithms
\cite{bhandari2018finite,srikant2019finite,kaledin2020finite,liu2025odemethod,bai2024regret,ganeshorder,ganeshsharper}.  In such settings
the gradient oracle is generally unbiased only after averaging with respect to the invariant
distribution of the chain.  At a finite time, the conditional law of the Markov state need not be
stationary, and the resulting bias must be controlled through the mixing behavior of the chain.

This paper studies this problem for smooth objectives satisfying the Polyak--\L{}ojasiewicz (PL)
inequality.  The PL condition, introduced in \cite{polyak1963gradient} and developed in its modern
optimization form in \cite{karimi2016linear}, is weaker than strong convexity but still yields global
linear convergence for deterministic gradient descent.  It is a natural condition in overparameterized
learning, least-squares problems, control, and other settings where global convergence can hold
without convexity.  We consider the Markovian SGD recursion
\begin{equation}\label{eq:sgd-intro}
    x_{k+1}=x_k-\alpha_kG_k,
    \qquad
    G_k=g(x_k,Z_k)+M_{k+1},
\end{equation}
where \((Z_k)_{k\ge 0}\) is an exogenous Markov chain with invariant distribution \(\pi\), the
stationary oracle satisfies
\[
    \int g(x,z)\,\pi(dz)=\nabla f(x),
\]
and \((M_{k+1})_{k\ge0}\) is a martingale-difference perturbation.  The oracle is allowed to satisfy a bounded \(A\)-\(B\)-\(C\) growth envelope, 
called the ABC condition because the constants \(A,B,C\) control the 
gradient-dependent term, the objective-gap-dependent term, and the additive 
noise floor:
\[
    \|G_k\|^2 \le A\|\nabla f(x_k)\|^2 + B(f(x_k)-f^\star)+C.
\]
This condition is used in prior PL-SGD analyses \cite{kar2026markovpl}. This growth model is
important because a uniform bounded-variance assumption can be too restrictive in PL problems,
least-squares models, interpolation regimes, and minibatch sampling.

There is now a substantial literature on first-order methods and stochastic approximation with
Markovian sampling.  Finite-time analyses of Markovian stochastic gradient methods were developed
in convex, strongly convex, and nonconvex settings by
\cite{sun2018markov,doan2020finite,doan2022markov,even2023sgd}.  More general Markovian
stochastic-approximation and reinforcement-learning problems are often handled using
Poisson-equation decompositions
\cite{benveniste1990adaptive,kushner2003stochastic,borkar2008stochastic,metivier1984applications,xu2026persistent}.
Recent work has also studied optimal mixing-time dependence for several first-order methods using
randomized batching, multilevel, or variance-reduction ideas
\cite{beznosikov2023markovian,even2023sgd}.  In the PL setting most closely related to ours,
Kar, Chandak, Singh, Moulines, Bhatnagar, and Bambos
\cite{kar2026markovpl} proved the first uniform-in-time high-probability guarantee for SGD with
Markovian plus martingale-difference noise under an ABC oracle envelope.  Their expectation bound
has leading stochastic order \(\widetilde{O}(\tmix/k)\), while their uniform high-probability bound
has leading stochastic order \(\widetilde{O}(\tmix^2/k)\).  Thus, before the present work, it was
unclear whether high-probability PL-SGD truly requires a quadratic polynomial dependence on the
mixing time, or whether the linear dependence suggested by expectation bounds is achievable.

Our first main result shows that the optimal polynomial dependence is linear.  For ordinary
one-sample SGD under the bounded ABC oracle envelope, we prove a uniform-in-time high-probability
bound whose leading stochastic term scales as
\[
    \widetilde{O}\!\left(\frac{\tmix}{k+K_0}\right)
\]
under geometric mixing, up to logarithmic factors and the usual optimization transient
\(K_0\{f(x_0)-f^\star\}/(k+K_0)\).  The hidden constants depend on the smoothness, PL, and bounded-ABC envelope constants, including the noise-scale term in \Cref{ass:abc}.  We also prove a matching lower bound on a one-dimensional
quadratic PL objective driven by a persistent two-state Markov chain:
\[
    f(x_k)-f^\star
    =
    \Omega\!\left(\frac{\sigma^2\tmix}{k}\right)
\]
in expectation and with constant probability.  Therefore no theorem valid over the same class of
instances can improve the polynomial dependence on \(\tmix\) below linear, apart from logarithmic
factors.

\subsection{Technical challenge and proof novelty}
\label{sec:technical_novelty}

The main technical obstacle is that the Markovian term in the PL descent recursion is an
\emph{adaptive} Markovian additive functional.  In the weighted descent recursion of
\Cref{lem:descent}, the stochastic part contains
\[
    \sum_{\ell=0}^{k-1} w_{\ell,k} h(x_\ell,Z_\ell),
    \qquad
    h(x,z)
    =
    \big\langle \nabla f(x),\,g(x,z)-\nabla f(x)\big\rangle,
\]
where \(h\) is defined in \eqref{eq:xi-h}.  For each fixed \(x\), the stationary-gradient identity
in \Cref{ass:abc} implies
\[
    \int h(x,z)\,\pi(dz)=0.
\]
However, conditionally on the past, \(Z_\ell\) need not be distributed according to \(\pi\), and
\(x_\ell\) is itself generated from the previous Markov states and martingale perturbations.
Consequently, \(h(x_\ell,Z_\ell)\) is neither a martingale difference nor a fixed additive
functional of the Markov chain.  Standard martingale concentration cannot be applied directly, and
classical concentration inequalities for fixed Markov-chain observables do not apply in a
black-box way.

\paragraph{The limitation of the Poisson equation.}
The standard mechanism for decoupling this temporal dependence relies on solutions to the Poisson equation. This framework yields a martingale difference sequence by representing the local gradient bias through the solution $V(x,z)$. For instance, in the recent high-probability PL-SGD analysis by Kar et al. \cite{kar2026markovpl}, this solution is defined (cf. Lemma 4.1 therein) as the infinite-horizon sum:
\begin{equation}
    V(x,z) := \mathbb{E}\left[ \sum_{j=0}^\infty \big( g(x,Z_j) - \nabla f(x) \big) \;\Big|\; Z_0 = z \right].
\end{equation}
This allows the Markovian noise to be perfectly decomposed as $g(x,z) - \nabla f(x) = V(x,z) - \int V(x,z')p(dz'|z)$. While this identity successfully isolates a strict martingale difference sequence, defined as $\tilde{M}_{\ell+1} := V(x_\ell, Z_{\ell+1}) - \int V(x_\ell, z) p(dz | Z_\ell)$, it achieves this by inadvertently amplifying the magnitude of the stochastic increments. 

Because the Markov chain requires $\mathcal{O}(\tmix)$ steps to contract toward the invariant measure $\pi$, the geometric series defining $V(x,z)$ forces its magnitude to scale linearly with the mixing time. Kar et al. formalize this dependence in their Lemma C.1, establishing the structural upper bound $\|V(x,z)\|^2 \le \mathcal{O}(\tmix^2 \cdot \Delta(x))$.

When uniform-in-time concentration inequalities, such as the Azuma-Hoeffding bound, are applied to this derived martingale, the variance proxy relies on the square of the almost-sure increment bounds. By measuring the deviations of a martingale whose increments have inherently bounded ranges of $\mathcal{O}(\tmix)$, the resulting concentration bound accumulates a $\mathcal{O}(\tmix^2)$ polynomial penalty. This quadratic dependence reflects a limitation of that proof route and leaves a gap relative to the optimal linear $\mathcal{O}(\tmix)$ scaling achievable in expectation bounds.

\paragraph{Our approach: lag-blocking and temporal conditioning.}
Our proof avoids this loss by changing the conditioning structure instead of changing the
observable through a Poisson corrector.  For a fixed target time \(k\), choose the analytical delay
\(m_k\) from \eqref{eq:delay-def}.  For \(\ell\ge m_k\), decompose
\[
    h(x_\ell,Z_\ell)
    =
    h(x_{\ell-m_k},Z_\ell)
    +
    \bigl\{
        h(x_\ell,Z_\ell)-h(x_{\ell-m_k},Z_\ell)
    \bigr\}.
\]
The first \(m_k\) indices are treated separately as an initial-window term, controlled in
\Cref{lem:initial}.  For the remaining terms, the delayed iterate \(x_{\ell-m_k}\) is
\(\calF_{\ell-m_k}\)-measurable, while by \Cref{ass:mixing} the chain has \(m_k\) transitions to
move toward stationarity before time \(\ell\).  With
\[
    Y_\ell
    :=
    w_{\ell,k} h(x_{\ell-m_k},Z_\ell),
    \qquad
    \overline Y_\ell
    :=
    \E[Y_\ell\mid \calF_{\ell-m_k}],
\]
as in \eqref{eq:Y-def}, the Markovian sum is split into an initial window, a centered delayed
martingale part, a mixing-bias part, and a replacement part:
\[
\begin{aligned}
\sum_{\ell=0}^{k-1} w_{\ell,k}h(x_\ell,Z_\ell)
&=
\sum_{\ell=0}^{m_k-1} w_{\ell,k}h(x_\ell,Z_\ell)
+
\sum_{\ell=m_k}^{k-1}(Y_\ell-\overline Y_\ell)  \\
&\quad
+
\sum_{\ell=m_k}^{k-1}\overline Y_\ell
+
\sum_{\ell=m_k}^{k-1}
    w_{\ell,k}
    \bigl\{h(x_\ell,Z_\ell)-h(x_{\ell-m_k},Z_\ell)\bigr\}.
\end{aligned}
\]

The centered delayed term is controlled by the residue-class martingale argument in
\Cref{lem:delayed-mart}.  The indices are split modulo \(m_k\).  Along a fixed residue class,
consecutive delayed terms are separated by exactly \(m_k\) Markov transitions, so after subtracting
\(\overline Y_\ell\) they form a martingale-difference sequence with respect to a down-sampled
filtration.  For the actual PL recursion, the relevant object is the weighted sum above, and the
residue-class bounds involve the weighted square sum \(\sum_{\ell<k}w_{\ell,k}^2\).  Summing over
the \(m_k\) residue classes produces a variance proxy of order
\[
    m_k\sum_{\ell<k}w_{\ell,k}^2
    \asymp
    \frac{m_k}{k+K_0},
\]
up to logarithmic and offset factors; the required weight estimates are collected in
\Cref{lem:weights}.  Under geometric mixing, \(m_k=\widetilde O(\tmix)\), which is the source of
the leading stochastic order \(\widetilde O(\tmix/(k+K_0))\) in
\Cref{cor:geometric-upper}.  No independence between different residue classes is required:
concentration is applied separately within each residue class, and the resulting high-probability
bounds are combined by a union bound.

The remaining terms are controlled separately.  The mixing bias
\(\sum_{\ell=m_k}^{k-1}\overline Y_\ell\) is small by the total-variation mixing condition in
\Cref{ass:mixing} and the bound in \Cref{lem:mixing-bias}.  The replacement term is deterministic
on the stopped good event: by \Cref{lem:h-envelope},
\[
    \bigl|h(x_\ell,z)-h(x_{\ell-m_k},z)\bigr|
    \le
    c_{\mathrm{lip}}
    \bigl(1+\sqrt{\Del(x_\ell)}+\sqrt{\Del(x_{\ell-m_k})}\bigr)
    \|x_\ell-x_{\ell-m_k}\|,
\]
and
\[
    \|x_\ell-x_{\ell-m_k}\|
    \le
    \sum_{i=\ell-m_k}^{\ell-1}\alpha_i\|G_i\|.
\]
The ABC envelope in \Cref{ass:abc} controls the gradients on the stopped event, and the resulting
replacement bound is stated in \Cref{lem:replacement}.  These ingredients are combined in
\Cref{lem:induction-step}, and the first-failure argument in \appref{app:main-proof} yields the
uniform-in-time guarantee of \Cref{thm:profile-upper}.  The delay \(m_k\) and residue classes are
only analytical devices; the algorithm remains the ordinary one-sample SGD recursion
\eqref{eq:sgd}.

We also study a heavy-tailed Markovian regime.  In many applications, individual gradient
samples may have no deterministic envelope and may not even have finite variance.  High-probability
guarantees under only finite \(p\)-th moments, \(p\in(1,2]\), therefore require robustification.
For independent or martingale-type noise, clipping, normalization, median-of-means, and related
robust methods have been used to obtain high-probability guarantees
\cite{cutkosky2021heavy,gorbunov2021nonsmooth,sadiev2023unbounded,nguyen2023clippedsgd,
puchkin2024barrier,hubler2025normalization,armacki2025nonlinear}.  Markovian heavy-tailed data are
more delicate because dependence reduces the effective sample size and nonstationary initialization
creates an additional bias.  We address this by holding the iterate fixed over a block, clipping
every consecutive Markovian gradient in the block, and averaging all clipped samples.  No samples
inside the block are discarded; the mixing window enters only through the analysis and the choice of
clipping scale.  Under geometric mixing and total transition budget \(T\), the resulting
high-probability stochastic term has order
\[
    \widetilde{O}\!\left(
        \sigma_p^2
        \left(\frac{\tmix}{T}\right)^{2(p-1)/p}
    \right),
\]
up to dimension and logarithmic factors.  A sticky-chain lower bound shows that the exponent
\(2(p-1)/p\) and the dependence on the Markovian effective sample size cannot be improved, up to
logarithmic and dimension factors.  The lower bound is one-dimensional, so we do not claim optimality
of the dimension dependence in the upper bound.

\subsection{Contributions.}
The paper makes the following contributions.

\begin{enumerate}
    \item We prove a uniform-in-time high-probability PL-SGD upper bound under a general uniform
    total-variation mixing profile on a general measurable Markov state space.  The result applies
    to ordinary one-sample SGD and does not require compactness of the state space.

    \item Under geometric mixing, the leading stochastic term in the bounded-ABC regime satisfies
    \[
        \widetilde{O}\!\left(\frac{\tmix}{k+K_0}\right),
    \]
    with no additional polynomial power of \(\tmix\).  This closes the gap between the
    \(\widetilde{O}(\tmix^2/k)\) high-probability behavior obtained by Poisson-equation-type
    analyses and the \(\widetilde{O}(\tmix/k)\) behavior suggested by expectation bounds.

    \item We prove a matching lower bound.  Even for the quadratic objective \(f(x)=x^2/2\) and a
    bounded two-state Markovian oracle, SGD satisfies
    \[
        \Pr\!\left(
            f(x_k)-f^\star
            \ge
            c\,\frac{\sigma^2\tmix}{k}
        \right)
        \ge c_0
    \]
    for universal constants \(c,c_0>0\).  Thus linear polynomial dependence on the mixing time is
    unavoidable at constant confidence.

    \item We extend the lag-blocking viewpoint to finite-\(p\) heavy-tailed Markovian gradients.
    The proposed all-samples clipped block method uses every Markov transition in each block and
    achieves the transition-budget rate
    \[
        \widetilde{O}\!\left(
            \sigma_p^2
            \left(\frac{\tmix}{T}\right)^{2(p-1)/p}
        \right)
    \]
    under geometric mixing, up to logarithmic and dimension factors.

    \item We prove a matching heavy-tailed lower bound by reducing PL optimization to mean
    estimation for a sticky Markov chain.  This shows that both the Markovian effective-sample-size
    dependence and the finite-moment exponent \(2(p-1)/p\) are unavoidable up to logarithmic
    and dimension factors; the result does not claim optimality of the dimension dependence.
\end{enumerate}

\section{Related work}\label{sec:related}

We summarize the closest comparisons in \Cref{tab:related-markovian-compact}.  
The table separates objective or problem-class assumptions from oracle/noise assumptions:
PL is an objective-geometry condition, whereas the ABC envelope is an oracle growth condition.  
The main comparison is with high-probability PL-SGD under Markovian sampling; broader background
on stochastic approximation, Poisson-equation methods, Markov-chain concentration, decentralized
and MCMC sampling, and robust heavy-tailed optimization is deferred to
\appref{app:extended-related}.

\begin{table}[h]
\centering
\scriptsize
\setlength{\tabcolsep}{2.2pt}
\renewcommand{\arraystretch}{1.18}
\caption{Compact comparison with the closest Markovian stochastic-optimization results.
Here Ach. denotes an upper-bound or algorithmic achievability result, LB denotes a lower bound
for the corresponding problem model, HP denotes high probability, and Exp. denotes expectation.
The notation \(\widetilde O(\cdot)\) hides logarithmic factors and problem-dependent constants.
Rows differ in objective class, oracle model, and algorithmic model, so the rates are not meant
as direct one-to-one comparisons.}
\label{tab:related-markovian-compact}
\begin{tabularx}{\textwidth}{@{}p{0.17\textwidth}p{0.105\textwidth}p{0.18\textwidth}p{0.115\textwidth}cccX@{}}
\toprule
\textbf{Work}
&
\textbf{Problem class}
&
\textbf{Oracle/noise model}
&
\textbf{Algorithm}
&
\shortstack{\textbf{Mark.}\\\textbf{samples}}
&
\textbf{Ach.}
&
\textbf{LB}
&
\textbf{Objective / rate}
\\
\midrule

Sun, Sun, and Yin~\cite{sun2018markov};
Doan et al.~\cite{doan2020finite};
Doan~\cite{doan2022markov}
&
C / SC / NC
&
bounded or variance-type Markovian noise
&
plain Markov-chain SGD
&
\(\checkmark\)
&
\(\checkmark\)
&
\(\times\)
&
finite-time convergence with mixing-dependent error
\\[0.35em]

Even~\cite{even2023sgd}
&
general Markovian sampling
&
Markovian sampling; hitting-time dependence
&
MC-SGD / MC-SAG
&
\(\checkmark\)
&
\(\checkmark\)
&
\(\checkmark\)
&
upper and lower bounds involving hitting-time quantities
\\[0.35em]

Beznosikov et al.~\cite{beznosikov2023markovian}
&
NC / SC / VI
&
strong-growth-type Markovian oracle
&
batching / multilevel estimators
&
\(\checkmark\)
&
\(\checkmark\)
&
\(\checkmark\)
&
linear mixing-time oracle complexity for modified estimators
\\[0.35em]

Kar et al.~\cite{kar2026markovpl}
&
PL
&
ABC growth envelope
&
ordinary one-sample SGD
&
\(\checkmark\)
&
\(\checkmark\)
&
\(\times\)
&
HP \(\widetilde O(\tmix^2/k)\); Exp. \(\widetilde O(\tmix/k)\)
\\[0.35em]

Agrawal, Maguluri, and Zubeldia~\cite{agrawal2026heavytailedmarkov}
&
SA
&
finite-state Markov component; heavy-tailed martingale noise
&
SA recursion
&
\(\checkmark\)
&
\(\checkmark\)
&
\(\times\)
&
concentration and tail classification for Markovian SA
\\[0.35em]

\textbf{This paper: light-tailed}
&
PL
&
ABC growth envelope
&
ordinary one-sample SGD
&
\(\checkmark\)
&
\(\checkmark\)
&
\(\checkmark\)
&
HP \(\widetilde O(\tmix/(k+K_0))\); LB \(\Omega(\sigma^2\tmix/k)\)
\\[0.35em]

\textbf{This paper: heavy-tailed}
&
PL
&
finite-\(p\) Markovian gradients
&
clipped Markovian blocks
&
\(\checkmark\)
&
\(\checkmark\)
&
\(\checkmark\)
&
HP \(\widetilde O\!\left(\sigma_p^2(\tmix/T)^{2(p-1)/p}\right)\); matching LB
\\

\bottomrule
\end{tabularx}

\vspace{0.25em}
\begin{minipage}{0.98\textwidth}
\scriptsize
\emph{Abbreviations:} C = convex, SC = strongly convex, NC = nonconvex,
VI = variational inequality, SA = stochastic approximation, Mark. = Markovian.
The LB column records lower bounds in the corresponding optimization/oracle-complexity model;
qualitative sharpness examples for broader stochastic-approximation tail behavior are not counted
as PL optimization lower bounds.
\end{minipage}
\end{table}

\paragraph{Closest PL-SGD comparison.}
The closest prior result is Kar et al.~\cite{kar2026markovpl}, which studies ordinary one-sample
SGD for smooth PL objectives with Markovian plus martingale-difference noise under an ABC oracle
envelope. Their expectation bound has leading stochastic order \(\widetilde O(\tmix/k)\), while
their uniform high-probability bound has leading stochastic order \(\widetilde O(\tmix^2/k)\).
Thus, before the present work, it was unclear whether the quadratic high-probability dependence on
the mixing time was intrinsic or an artifact of the proof technique. Our light-tailed theorem closes
this gap by proving a uniform high-probability bound with leading stochastic term
\(\widetilde O(\tmix/(k+K_0))\) for the same ordinary SGD recursion, under a general uniform
total-variation mixing profile. The matching two-state lower bound shows that the order
\(\Omega(\sigma^2\tmix/k)\) is unavoidable at constant confidence, up to logarithmic factors.

\paragraph{Other Markovian first-order methods.}
Finite-time analyses of stochastic gradient methods with Markovian samples were developed in
convex, strongly convex, and nonconvex settings by
Sun, Sun, and Yin~\cite{sun2018markov}, Doan et al.~\cite{doan2020finite}, and
Doan~\cite{doan2022markov}. Even~\cite{even2023sgd} studies general Markovian sampling schemes,
including MC-SGD and the variance-reduced MC-SAG method, and proves lower bounds involving
hitting-time quantities. Beznosikov et al.~\cite{beznosikov2023markovian} obtain linear
mixing-time dependence for several first-order methods using randomized batching and multilevel
gradient estimators, including nonconvex, strongly convex, and variational-inequality settings.
These works show that favorable mixing dependence is possible in several regimes, but they do not
resolve the high-probability behavior of ordinary one-sample PL-SGD under an ABC oracle envelope.

\paragraph{Heavy-tailed Markovian regimes.}
The heavy-tailed part of the paper is related to robust stochastic optimization and Markovian
stochastic approximation under finite-moment noise. Agrawal, Maguluri, and
Zubeldia~\cite{agrawal2026heavytailedmarkov} study concentration and tail behavior for general
stochastic approximation with finite-state Markovian components and heavy-tailed martingale
perturbations. Our robust result is different in two ways: it is a PL optimization theorem, and the
Markovian gradient component itself is allowed to have only a finite stationary \(p\)-th moment,
\(p\in(1,2]\). Using clipped Markovian blocks and a total transition budget \(T\), we obtain the
high-probability rate
\[
    \widetilde O\!\left(
        \sigma_p^2\left(\frac{\tmix}{T}\right)^{2(p-1)/p}
    \right),
\]
and prove a matching sticky-chain lower bound for the Markovian effective-sample-size dependence
and the finite-moment exponent, up to logarithmic and dimension factors.

\paragraph{Scope of the comparison.}
The rows in \Cref{tab:related-markovian-compact} are not directly interchangeable. Some works
study convex or nonconvex optimization rather than PL objectives; some use modified estimators or
variance-reduced methods rather than ordinary SGD; and some prove lower bounds for hitting-time or
oracle-complexity models rather than PL last-iterate optimization. The contribution of the present
paper is to identify the optimal polynomial mixing-time dependence for ordinary high-probability
PL-SGD in the bounded-ABC regime, and the optimal finite-moment exponent and Markovian
effective-sample-size dependence for the clipped-block heavy-tailed regime.

\section{Problem formulation and assumptions}\label{sec:setup}
Let $f:\R^d\to\R$ be differentiable and lower bounded.  Write
\[
        f^\star:=\inf_{x\in\R^d}f(x),
        \qquad
        \Del(x):=f(x)-f^\star,
        \qquad
        \Del_k:=\Del(x_k).
\]
The algorithm is
\begin{equation}\label{eq:sgd}
        x_{k+1}=x_k-\alpha_kG_k,
        \qquad
        G_k=g(x_k,Z_k)+M_{k+1}.
\end{equation}
The initial point $x_0$ is deterministic and $\Del_0<\infty$.  If $x_0$ is random, the same statements hold conditionally on $\calF_0$ on every realization with finite $\Del_0$.

\begin{assumption}[Objective]\label{ass:objective}\label{ass:pl-smooth}
The function $f$ is $L$-smooth and satisfies the $\mu$-PL inequality:
\begin{equation}\label{eq:pl}
        \|\nabla f(x)\|^2\ge 2\mu\Del(x),
        \qquad x\in\R^d.
\end{equation}
Since $f$ is lower bounded and $L$-smooth, the gradient upper bound
\begin{equation}\label{eq:grad-upper}
        \|\nabla f(x)\|^2\le 2L\Del(x)
\end{equation}
holds for all $x$.
\end{assumption}

\begin{remark}[Role, commonness, and scope of the objective assumption]
Smoothness and PL geometry are standard in the nonconvex-optimization and stochastic-gradient
literature.  The PL inequality goes back to Polyak~\cite{polyak1963gradient}, and its modern
optimization role, together with its relation to error bounds, quadratic growth, and related
conditions, is surveyed by Karimi, Nutini, and Schmidt~\cite{karimi2016linear}.  The same
smoothness-plus-PL framework is used in stochastic-gradient analyses with growth-type noise
conditions~\cite{gower2019sgd,khaled2023better} and in the closest prior Markovian PL-SGD work
\cite{kar2026markovpl}.  Thus \Cref{ass:objective} is not a special assumption introduced for our
lag-blocking argument; it is the standard geometry under which last-iterate objective convergence is
usually proved beyond convexity.

The assumption also holds for many functions of interest.  Every differentiable strongly convex
smooth function is PL, but PL is strictly weaker than strong convexity.  It also holds for full-rank
least-squares and more generally for quadratic objectives with a positive curvature direction on the
relevant range; it appears in several overparameterized models and neural-network loss landscapes
\cite{liu2022loss}; and it is used in control examples such as linear-quadratic regulator objectives
\cite{fazel2018global}.  In applications where the property is only valid on a sublevel set, the
localized version in \Cref{cor:local} allows the same proof to be run with local constants.

In the proof, \Cref{ass:objective} is used in two distinct ways.  Smoothness gives the one-step
descent inequality that leads to the weighted recursion in \Cref{lem:descent}.  The PL inequality
\eqref{eq:pl} converts gradient decrease into objective decrease, producing the contraction factors
in the weights \eqref{eq:weights}.  The smoothness upper bound \eqref{eq:grad-upper}, proved in
\Cref{lem:grad-upper}, is repeatedly used to express gradient magnitudes in terms of the
suboptimality \(\Delta(x)\).  The same geometry is used in the heavy-tailed inexact-gradient
argument, \Cref{lem:deterministic-inexact-pl}.  Thus \Cref{ass:objective} is the mechanism that
turns stochastic gradient control into last-iterate objective control.
\end{remark}

\begin{assumption}[Oracle, martingale noise, and ABC envelope]\label{ass:noise}\label{ass:oracle}\label{ass:abc}
Let
\[
        \calF_k:=\sigma(x_0,Z_0,\ldots,Z_k,M_1,\ldots,M_k).
\]
The sequence $(M_{k+1})_{k\ge0}$ is a martingale difference sequence with respect to $(\calF_k)$:
\[
        \E[M_{k+1}\mid\calF_k]=0.
\]
For every $x\in\R^d$,
\begin{equation}\label{eq:stationary-gradient}
        \int g(x,z)\,\pi(dz)=\nabla f(x).
\end{equation}
There are constants $A,B,C\ge0$ such that almost surely, for every $k$,
\begin{equation}\label{eq:abc-G}
        \|G_k\|^2\le A\|\nabla f(x_k)\|^2+B\Del_k+C,
\end{equation}
and, for every $x\in\R^d$ and $z\in\Zspace$,
\begin{equation}\label{eq:abc-g}
        \|g(x,z)\|^2\le A\|\nabla f(x)\|^2+B\Del(x)+C.
\end{equation}
Moreover $g(\cdot,z)$ is $L_g$-Lipschitz for every $z$:
\begin{equation}\label{eq:g-lip}
        \|g(x,z)-g(y,z)\|\le L_g\|x-y\|.
\end{equation}
\end{assumption}

\begin{remark}[Scope, examples, and role of the ABC envelope]
\Cref{ass:abc} is a pathwise ABC or weak-growth type envelope.  Conditions of this form are common
in modern SGD analyses under the names ABC, weak growth, strong growth, expected smoothness, or
growth-noise conditions; see, for example,
\cite{gower2019sgd,li2021second,khaled2023better,kar2026markovpl}.  The point of the assumption is
to allow the stochastic-gradient magnitude to grow with the local scale of the problem, rather than
requiring a uniform variance or uniform bounded-gradient bound over all \(x\).  This is important
for least squares, interpolation, minibatch gradients, and state-dependent oracle models, where the
sample gradient is typically small near a solution but can grow away from it.

The envelope contains several familiar bounded-oracle models as special cases.  Uniformly bounded
stochastic gradients correspond to \(A=B=0\).  Additive bounded noise,
\(G_k=\nabla f(x_k)+\xi_k\) with \(\|\xi_k\|\le\sigma\), gives
\(\|G_k\|^2\le 2\|\nabla f(x_k)\|^2+2\sigma^2\).  Relative-noise or strong-growth conditions
correspond to the case in which the stochastic-gradient magnitude is proportional to
\(\|\nabla f(x_k)\|\), possibly with an additive constant.  In finite-sum and minibatch problems,
such bounds hold whenever the component gradients are controlled by the objective residual or by the
full gradient on the sublevel set visited by the algorithm; in least-squares and generalized-linear
models this is typically verified from bounded features or from a local sublevel-set bound.  The
additional \(B\Delta_k\) term is natural under PL geometry because \eqref{eq:grad-upper} allows
\(\|\nabla f(x_k)\|^2\) and \(\Delta_k\) to be controlled together.  When the envelope holds only
locally, \Cref{cor:local} gives the corresponding localized statement.

The pathwise form is needed for high-probability control.  In \Cref{lem:M-noise}, it gives
deterministic martingale increment bounds for the martingale-difference perturbation.  In
\Cref{lem:h-envelope}, \eqref{eq:abc-g} gives a deterministic envelope for the Markovian descent
observable \(h\).  The Lipschitz condition \eqref{eq:g-lip} is used in \Cref{lem:replacement} to
control the error made by replacing \(x_\ell\) with the delayed iterate \(x_{\ell-m}\).  Finite
moment assumptions that do not imply such deterministic envelopes are treated separately in
\Cref{sec:heavy-tail} through clipping and blockwise averaging.
\end{remark}

\begin{assumption}[Exogenous Markov chain and uniform mixing profile]\label{ass:mixing}
The process $(Z_k)_{k\ge0}$ is a time-homogeneous Markov chain on a measurable space $(\Zspace,\calZ)$ with transition kernel $P$ and invariant law $\pi$.  It is Markov with respect to the full filtration:
\begin{equation}\label{eq:full-markov}
        \Pp(Z_{k+1}\in B\mid\calF_k)=P(Z_k,B),
        \qquad B\in\calZ.
\end{equation}
Define the uniform total-variation mixing profile
\begin{equation}\label{eq:mixing-profile}
        \varepsilon(m):=\sup_{z\in\Zspace}\|P^m(z,\cdot)-\pi\|_{\TV},
        \qquad m\ge0.
\end{equation}
We assume $\varepsilon(m)\to0$ as $m\to\infty$.
Whenever a single geometric mixing time is used, we say that $\tmix\ge1$ is valid if
\begin{equation}\label{eq:mixing-def}
        \varepsilon(m)\le 2^{-\lfloor m/\tmix\rfloor},
        \qquad m\ge0.
\end{equation}
\end{assumption}

\begin{remark}[Why exogeneity and uniform mixing are assumed]
Uniform total-variation mixing is a standard assumption in finite-time analyses and concentration
results for Markovian stochastic algorithms.  It is used, either directly or through closely related
mixing-time parameters, in Markovian-gradient and first-order analyses such as
\cite{sun2018markov,doan2020finite,doan2022markov,even2023sgd,beznosikov2023markovian,kar2026markovpl}
and in concentration inequalities for uniformly ergodic Markov chains
\cite{glynn2002hoeffding,paulin2015concentration}.  The assumption holds for every finite-state
irreducible aperiodic chain, and more generally for uniformly ergodic chains satisfying a Doeblin or
minorization condition.  Thus it covers many fixed-sampling mechanisms used in applications, such as
finite-state random walks with laziness, token chains on finite connected networks, fixed MCMC
samplers after a uniform-ergodicity verification, and privacy-subsampling chains with a fixed
transition rule.  The theorem is stated in terms of the whole profile \(\varepsilon(m)\) so that the
analysis can also use nonasymptotic mixing estimates when a single closed-form \(\tmix\) is not the
most convenient description.

\Cref{ass:mixing} has two proof roles.  First, the full-filtration Markov property
\eqref{eq:full-markov} makes the chain exogenous: the transition kernel of \(Z_k\) does not depend
on the current iterate.  This is essential for the lag argument, because once \(x_{\ell-m}\) is
\(\calF_{\ell-m}\)-measurable, the conditional law of \(Z_\ell\) is exactly
\(P^m(Z_{\ell-m},\cdot)\).  Second, the uniform total-variation profile
\eqref{eq:mixing-profile}, together with \eqref{eq:tv-dual}, controls the conditional bias of the
delayed observable.  This is the bound used in \Cref{lem:mixing-bias}.

The main theorem is stated for the full profile \(\varepsilon(m)\), not for a single scalar mixing
time.  The geometric form \eqref{eq:mixing-def} is used only in \Cref{cor:geometric-upper}, where
\Cref{lem:window} verifies that the analysis windows are admissible with
\(m_k=\widetilde O(\tmix)\).  Parameter-dependent Markov kernels, such as adaptive MCMC or
policy-dependent reinforcement-learning chains, would require additional control of the change in
the transition kernel or invariant law and are outside the present exogenous model.
\end{remark}

We use the stepsize
\begin{equation}\label{eq:stepsize}
        \alpha_k=\frac{a}{k+K_0},
        \qquad a\ge\frac{2}{\mu}.
\end{equation}
Set
\begin{equation}\label{eq:basic-K}
        K_0\ge K_{\mathrm{base}}
        :=\max\left\{1,\mu a,\frac{aL}{2}\left(2A+\frac{B}{\mu}\right)\right\}.
\end{equation}
For $0\le \ell\le k-1$, define
\begin{equation}\label{eq:weights}
        \zeta_{\ell+1,k-1}:=\prod_{j=\ell+1}^{k-1}(1-\mu\alpha_j),
        \qquad
        w_{\ell,k}:=\alpha_\ell\zeta_{\ell+1,k-1},
\end{equation}
with empty products equal to one.  We also define
\begin{equation}\label{eq:xi-h}
        \xi(x,z):=g(x,z)-\nabla f(x),
        \qquad
        h(x,z):=\ip{\nabla f(x)}{\xi(x,z)}.
\end{equation}
Then $\int h(x,z)\pi(dz)=0$ for each fixed $x$.

\paragraph{Notation.}
For probability measures $\nu$ and $\nu'$ on $(\Zspace,\calZ)$, we use the probability total-variation distance
\[
        \|\nu-\nu'\|_{\TV}:=\sup_{A\in\calZ}|\nu(A)-\nu'(A)|
        =\frac12 |\nu-\nu'|(\Zspace).
\]
Thus the largest possible distance between two probability laws is one.  With this convention, for every bounded measurable scalar function $\varphi$,
\begin{equation}\label{eq:tv-dual}
        \left|\int \varphi\,d(\nu-\nu')\right|
        \le 2\|\varphi\|_\infty\|\nu-\nu'\|_{\TV}.
\end{equation}
All norms on vectors are Euclidean.

\section{Sharp upper and lower bounds}\label{sec:main-results}

\subsection{Upper bound for a general mixing profile}\label{sec:sgd-upper-bound}
The main theorem is stated for a general mixing profile.  Fix $\delta\in(0,e^{-1})$ and define, for $k\ge1$,
\begin{equation}\label{eq:delay-def}
        N_k:=k+K_0,
        \qquad
        \widehat m_k:=\min\left\{m\in\N: m\ge1,\ \varepsilon(m)\le \frac{\delta}{32N_k^4}\right\},
        \qquad
        m_k:=\min\{k,\widehat m_k\}.
\end{equation}
The minimum is finite by Assumption~\ref{ass:mixing}.  The profile $\varepsilon(m)$ is nonincreasing in integer $m$: indeed, Markov kernels contract total variation and
$P^{m+1}(z,\cdot)-\pi=(P^m(z,\cdot)-\pi)P$.  Since the threshold $\delta/(32N_k^4)$ is nonincreasing in $k$, the sequences $\widehat m_k$ and $m_k$ are nondecreasing.  Define
\begin{equation}\label{eq:u-theta}
        u_k:=\log\left(\frac{16(m_k+1)(k+1)^2}{\delta}\right),
        \qquad
        q_k:=1+\frac{m_k}{K_0},
        \qquad
        \Theta_k:=m_ku_kq_k^2.
\end{equation}
Consequently $u_k,q_k,$ and $\Theta_k$ are also nondecreasing in $k$.

\begin{theorem}[Mixing-profile high-probability upper bound]\label{thm:profile-upper}
Suppose Assumptions~\ref{ass:objective}--\ref{ass:mixing} hold and let $(x_k)$ be generated by \eqref{eq:sgd} with stepsizes \eqref{eq:stepsize}.  There exist constants
\[
        \rho_0\in(0,1),\qquad K_D<\infty,\qquad C_\star<\infty,
\]
depending only on $a,\mu,L,L_g,A,B,C,\Del_0$, such that the following holds.  Assume $K_0\ge K_{\mathrm{base}}\vee K_D$ and that the analysis windows satisfy, for every $k\ge1$,
\begin{align}
        \frac{\Theta_k}{N_k}&\le \rho_0,\label{eq:admissible-1}\\
        \frac{m_k}{N_k}\left(1+\log\frac{N_k}{K_0}+\frac{m_k}{K_0}\right)&\le \rho_0.\label{eq:admissible-2}
\end{align}
Then, with probability at least $1-\delta$, for all $k\ge1$,
\begin{equation}\label{eq:profile-bound}
        \Del_k\le \frac{2K_0\Del_0+C_\star(1+\Theta_k)}{k+K_0}.
\end{equation}
\end{theorem}

The proof is given in \appref{app:main-proof}.

\begin{remark}[Constants and admissible windows]\label{rem:profile-constants}
The constant \(C_\star\) in \Cref{thm:profile-upper} is problem-dependent: it may depend on the
ABC constants \(A,B,C\), the geometry constants \(L,\mu,L_g\), the stepsize parameter \(a\),
and the initial suboptimality \(\Del_0\).  Thus rate statements such as
\(\widetilde O(\tmix/(k+K_0))\) are always meant after fixing these problem constants; in
particular, the noise-size parameter \(C\) in \Cref{ass:abc} is part of the hidden constant.
This is consistent with the lower-bound normalization in \Cref{thm:linear-lower}, where the
corresponding envelope scale is proportional to \(\sigma^2\).

The theorem is also an admissible-window result for a general mixing profile.  It becomes a closed
form rate only after the window conditions \eqref{eq:admissible-1}--\eqref{eq:admissible-2} are
verified.  \Cref{cor:geometric-upper} performs this verification for geometric mixing.
\end{remark}

\begin{corollary}[Geometric mixing rate]\label{cor:geometric-upper}
Suppose, in addition to the assumptions of \Cref{thm:profile-upper}, that the chain has a valid geometric mixing time $\tmix$ in the sense of \eqref{eq:mixing-def}.  There exists $K_{\mathrm{geo}}<\infty$, depending only on $a,\mu,L,L_g,A,B,C,\Del_0$, such that if
\begin{equation}\label{eq:K-geometric}
        K_0\ge K_{\mathrm{geo}}\,\tmix\left(1+\log\frac{e\tmix}{\delta}\right)^3,
\end{equation}
then \eqref{eq:profile-bound} holds with probability at least $1-\delta$.  More explicitly, for every $k\ge1$,
\begin{align}
        m_k&\le \left\lceil \tmix\left(1+\log_2\left(\frac{32(k+K_0)^4}{\delta}\right)\right)\right\rceil
        \le 1+\tmix\left(1+\log_2\left(\frac{32(k+K_0)^4}{\delta}\right)\right),\label{eq:mk-geometric-explicit}\\
        \Theta_k
        &\le M_k\,\log\left(\frac{16(k+1)^2(M_k+1)}{\delta}\right)
        \left(1+\frac{M_k}{K_0}\right)^2,\label{eq:theta-geometric-explicit}\\
        M_k&:=1+\tmix\left(1+\log_2\left(\frac{32(k+K_0)^4}{\delta}\right)\right).\notag
\end{align}
Consequently, \eqref{eq:profile-bound} gives a leading stochastic contribution of order
\[
    \widetilde O\!\left(\frac{\tmix}{k+K_0}\right),
\]
where the hidden factors are logarithmic in \(k\), \(1/\delta\), and \(\tmix\), and the
prefactor is the problem-dependent constant described in \Cref{rem:profile-constants}; there is
no additional polynomial power of \(\tmix\).
\end{corollary}

The proof is given in \appref{app:window}.

\begin{remark}[Comparison with Poisson-equation bounds]
Poisson-equation high-probability analyses typically control a derived martingale whose increment
range scales with the mixing time.  Applying worst-case martingale range bounds to such increments
leads to a leading stochastic term of order \(\widetilde O(\tmix^2/k)\).  The lag-blocking
analysis instead applies concentration to delayed residue-class sums whose increments are bounded
by the original oracle envelope from \Cref{ass:abc}.  The number of residue classes is \(m_k\), and
under \eqref{eq:mixing-def}, \Cref{lem:window} gives \(m_k=\widetilde O(\tmix)\).  This yields the
order \(\widetilde O(\tmix/(k+K_0))\) in \Cref{cor:geometric-upper}.
\end{remark}

\subsection{A matching lower bound}\label{sec:sgd-lower-bound}
The upper bound has the optimal polynomial dependence on mixing time.  Let $\Zspace=\{-1,+1\}$ and fix $\varepsilon\in(0,1/8]$.  Consider the persistent two-state chain
\begin{equation}\label{eq:two-state-kernel}
      P=
      \begin{pmatrix}
      1-\varepsilon & \varepsilon\\
      \varepsilon & 1-\varepsilon
      \end{pmatrix},
\end{equation}
whose invariant law is uniform.  Let
\begin{equation}\label{eq:lower-instance}
      f(x)=\frac{x^2}{2},
      \qquad
      g(x,z)=x+\sigma z,
      \qquad \sigma>0.
\end{equation}
With $M_{k+1}\equiv0$, SGD becomes
\begin{equation}\label{eq:lower-sgd-recursion}
      x_{n+1}=\left(1-\frac{a}{n+K}\right)x_n
      -\frac{a\sigma}{n+K}Z_n,
      \qquad x_0=0.
\end{equation}
The chain is started in stationarity.

\begin{theorem}[Linear mixing-time lower bound]\label{thm:linear-lower}
Fix $a\ge2$. There are positive constants $c_a$ and $c_0$ depending at most on $a$ such that the following holds. Let $K\ge \max\{4a,4a^2\}$ and let $\varepsilon\in(0,1/8]$. For the instance \eqref{eq:two-state-kernel}-\eqref{eq:lower-sgd-recursion}, if
\begin{equation}\label{eq:k-condition-main}
      k\ge \max\left\{8K,\frac{16}{\varepsilon},16a^2\right\},
\end{equation}
then
\begin{equation}\label{eq:expectation-lower-main}
      \E\left[f(x_k)-f^\star\right]
      \ge c_a\,\frac{\sigma^2}{\varepsilon k},
\end{equation}
and
\begin{equation}\label{eq:probability-lower-main}
      \Pp\!\left(
      f(x_k)-f^\star\ge c_a\,\frac{\sigma^2}{\varepsilon k}
      \right)
      \ge c_0.
\end{equation}
One may take $c_0=1/96$ and
\begin{equation}\label{eq:c-a-explicit}
      c_a=\frac{a^2 e^{-2}}{2^{14}\,3^{2a+2}}.
\end{equation}
The mixing time of the chain is $\Theta(1/\varepsilon)$. Consequently, for another constant $c'_a>0$ depending only on $a$,
\begin{equation}\label{eq:tmix-lower-main}
      \E\left[f(x_k)-f^\star\right]
      \ge c'_a\,\frac{\sigma^2\tmix}{k},
      \qquad
      \Pp\!\left(f(x_k)-f^\star\ge c'_a\,\frac{\sigma^2\tmix}{k}\right)
      \ge c_0.
\end{equation}
\end{theorem}

The proof is given in \appref{app:lower-proof}; the final moment-to-probability argument is in \appref{app:complete-lower}.

\begin{corollary}[No sublinear order in \(\tmix/k\) at constant confidence]\label{cor:no-sublinear}
Any high-probability theorem that is valid for all instances satisfying
\Cref{ass:objective,ass:abc,ass:mixing} cannot have leading stochastic order
\(o(\tmix/k)\) at constant confidence.  More precisely, if the failure probability is smaller than
the constant \(c_0\) in \Cref{thm:linear-lower}, then the right-hand side of a uniform upper bound
cannot be \(o(\tmix/k)\) on the family \eqref{eq:two-state-kernel}--\eqref{eq:lower-sgd-recursion}.
\end{corollary}

\begin{proof}
This follows directly from \Cref{thm:linear-lower}.  The lower-bound family
\eqref{eq:two-state-kernel}--\eqref{eq:lower-sgd-recursion} satisfies
\Cref{ass:objective,ass:abc,ass:mixing}.  On this family, \Cref{thm:linear-lower}
shows that, with probability at least \(c_0\),
\[
        f(x_k)-f^\star
        \ge c'_a\frac{\sigma^2\tmix}{k}.
\]
Therefore any uniform high-probability upper bound over the same class, with failure probability
strictly smaller than \(c_0\), cannot have a leading stochastic term \(o(\tmix/k)\), since such a
bound would eventually be smaller than the lower-bound threshold on this family, contradicting
\Cref{thm:linear-lower}.
\end{proof}

\section{Heavy-tailed Markovian gradients}\label{sec:heavy-tail}
The bounded-ABC theorem is sharp for light-tailed or almost-sure controlled gradients.  We now treat a genuinely heavy-tailed regime in which the Markovian gradient component has only a finite $p$th moment under the invariant law.  Ordinary one-sample SGD is not stable enough to provide logarithmic-confidence high-probability guarantees under this assumption alone.  The robust method below therefore uses clipping and blockwise averaging.  During each block the iterate is held fixed, the next consecutive Markovian gradients are all used, and the update is made with the average of the clipped samples.  The mixing window is a design and proof parameter used to choose the clipping scale and to analyze the consecutive block; it is not a spacing rule and no observations inside the block are discarded.

\begin{remark}[Why clipping is used]\label{rem:why-clipping}
The clipping step in \Cref{def:robust-batch-method} is not intended to improve the
constant-confidence exponent in the number of samples.  Even for independent samples, ordinary
averaging can have the same \(T^{-2(p-1)/p}\) scaling at constant probability under a finite
\(p\)-moment condition.  The obstruction is the dependence on the confidence parameter: under only
finite \(p\)-moments, a single rare outlier can dominate an unclipped average, and worst-case
high-probability bounds for ordinary averaging or unclipped SGD generally have polynomial rather
than logarithmic dependence on \(1/\delta\).  Robust mean-estimation and heavy-tailed stochastic
optimization methods therefore use clipping, truncation, normalization, median-of-means, or related
transformations to obtain logarithmic-confidence guarantees
\cite{catoni2012challenging,devroye2016subgaussian,lugosi2019mean,cutkosky2021heavy,
gorbunov2021nonsmooth,sadiev2023unbounded,nguyen2023clippedsgd,puchkin2024barrier,
hubler2025normalization,armacki2025nonlinear}.  Our contribution in this section is the Markovian
extension: clipped finite-\(p\) high-probability control for consecutive Markovian samples while
using every transition in each block, together with a sticky-chain lower bound for the Markovian
effective sample size.
\end{remark}

For $\lambda>0$, define the coordinatewise clipping map
\begin{equation}\label{eq:coord-clip-def}
        [\mathsf T_\lambda(v)]_j:=\max\{-\lambda,\min\{v_j,\lambda\}\},
        \qquad j=1,\ldots,d.
\end{equation}
The heavy-tailed oracle is written as
\begin{equation}\label{eq:heavy-oracle}
        Y(x,z,u)={\mathcal G}(x,z,u),
\end{equation}
where $u$ denotes auxiliary randomness conditionally independent of the past given the current iterate and Markov state.

\begin{assumption}[Stationary finite-moment Markovian oracle]\label{ass:heavy-tail}
Fix \(p\in(1,2]\).  At each oracle call, the auxiliary variable \(U\) is sampled conditionally
independently of the past given the queried point and current Markov state, with the same
conditional law used in \eqref{eq:heavy-unbiased} and \eqref{eq:stationary-p-moment}.  The oracle
satisfies the stationary unbiasedness condition
\begin{equation}\label{eq:heavy-unbiased}
        \int \E_U[{\mathcal G}(x,z,U)]\,\pi(dz)=\nabla f(x),
        \qquad x\in\R^d.
\end{equation}
Let $\eta(x,z,u):={\mathcal G}(x,z,u)-\nabla f(x)$ and
\begin{equation}\label{eq:Rp-def}
        R_p(x):=1+\|\nabla f(x)\|+\sqrt{\Delta(x)}.
\end{equation}
There is a scale $\sigma_p\ge1$ such that, for every $x\in\R^d$,
\begin{equation}\label{eq:stationary-p-moment}
        \int \E_U\left[\|\eta(x,z,U)\|^p\right] \pi(dz)
        \le \sigma_p^p R_p(x)^p.
\end{equation}
\end{assumption}

\begin{remark}[Why the heavy-tailed moment condition is stationary]
Finite-moment assumptions are standard in robust stochastic optimization and robust mean estimation
under heavy-tailed data; see, for example,
\cite{cutkosky2021heavy,gorbunov2021nonsmooth,sadiev2023unbounded,nguyen2023clippedsgd,
puchkin2024barrier,hubler2025normalization,armacki2025nonlinear}.  Markovian stochastic
approximation with heavy-tailed perturbations has also been studied in \cite{agrawal2026heavytailedmarkov}.
\Cref{ass:heavy-tail} is the Markovian optimization analogue of these conditions: it allows the raw
oracle to have no deterministic envelope and, when \(p<2\), no finite variance.

The moment bound is imposed under the invariant law rather than uniformly over all Markov states.
This is natural for Markovian data because rare states may produce very large gradients, but their
contribution is controlled by their stationary probability under \(\pi\).  Such a condition holds
whenever the stationary distribution has a finite \(p\)-moment for the gradient noise scale.  This
includes bounded-gradient models, sub-Weibull or polynomial-tail models with finite \(p\)-moment,
mixtures with rare outliers, and Markov chains whose high-gradient states have sufficiently small
stationary mass.  In particular, the bounded ABC model in \Cref{ass:abc} implies a finite-\(p\)
stationary moment bound for every \(p\in(1,2]\), after adjusting the scale \(\sigma_p\); the
heavy-tailed theorem is nevertheless complementary because it analyzes a clipped block method rather
than ordinary one-sample SGD.

Nonstationary initialization is handled by clipping and lag-blocking.  The clipping bias and
stationary second-moment estimates are isolated in \Cref{lem:clipped-stationary-bias-var}; the
Markovian consecutive-block concentration is proved in \Cref{lem:consecutive-block-concentration}.
The first \(m\) observations of each block are controlled directly by clipping, while the remaining
observations are centered relative to the sigma-field one lag in the past.  Thus every Markov
transition in a block is used by the algorithm; the mixing window enters only through the proof and
the clipping scale.
\end{remark}

\begin{definition}[All-samples clipped block method]\label{def:robust-batch-method}
Fix integers $N,b\ge1$ and a clipping level $\lambda$.  Set $\tau_0=0$.  At outer iteration $r=0,\ldots,N-1$, hold $x_r$ fixed and observe the next $b$ consecutive Markovian gradients.  Specifically, for $i=1,\ldots,b$, after one Markov transition from time $\tau_r+i-1$ to $\tau_r+i$, draw $U_{r,i}$ and observe
\[
        Y_{r,i}:={\mathcal G}(x_r,Z_{\tau_r+i},U_{r,i}).
\]
After the block is collected, set $\tau_{r+1}:=\tau_r+b$ and update
\begin{equation}\label{eq:robust-batch-update}
        \widehat g_r:=\frac1b\sum_{i=1}^b\mathsf T_\lambda(Y_{r,i}),
        \qquad
        x_{r+1}=x_r-\frac1{4L}\widehat g_r .
\end{equation}
Thus each Markov transition in the block contributes exactly one clipped gradient sample to the update.
\end{definition}

Let
\begin{equation}\label{eq:theta-p-def}
        \vartheta_p:=\frac{2(p-1)}{p}.
\end{equation}
For a radius $\mathcal R\ge1$, define
\begin{equation}\label{eq:BR-def}
        B_{\mathcal R}:=1+\sqrt{2L\mathcal R}+\sqrt{\mathcal R}.
\end{equation}
For confidence level $\delta\in(0,e^{-1})$ and an integer analysis lag $m\in\{1,\ldots,b\}$, set
\begin{equation}\label{eq:heavy-u-def}
        u_{\rm ht}:=\log\left(\frac{64dN(m+1)}{\delta}\right),
        \qquad
        s_{\rm ht}:=\frac{m u_{\rm ht}}{b},
\end{equation}
and choose
\begin{equation}\label{eq:heavy-fixed-lambda}
        \lambda:=c_{\lambda,p}\,\sigma_pB_{\mathcal R}\,s_{\rm ht}^{-1/p},
\end{equation}
where $c_{\lambda,p}$ is a sufficiently large numerical constant depending only on $p$.  The following theorem is stated with explicit admissibility conditions.  They require the block to be long enough for the clipping bias, sampling fluctuation, and residual mixing bias to be below the radius used in the induction.

\begin{theorem}[Heavy-tailed robust Markovian PL optimization]\label{thm:heavy-upper}
Suppose Assumptions~\ref{ass:objective}, \ref{ass:mixing}, and \ref{ass:heavy-tail} hold.  Let $(x_r)_{r=0}^N$ be generated by \Cref{def:robust-batch-method} with clipping level \eqref{eq:heavy-fixed-lambda}.  Let $\mathcal R\ge 4\Delta_0+4$ and assume
\begin{align}
        \varepsilon(m)&\le \frac{\delta}{64dNb},\label{eq:heavy-mixing-condition}\\
        \lambda&\ge 2\sqrt{2L\mathcal R},\label{eq:heavy-lambda-condition}\\
        \frac{c_p d\sigma_p^2B_{\mathcal R}^2}{\mu}
        s_{\rm ht}^{\vartheta_p}&\le \frac{\mathcal R}{4},\label{eq:heavy-batch-size-condition}
\end{align}
where $c_p$ is a numerical constant depending only on $p$.  Then, with probability at least $1-\delta$,
\begin{equation}\label{eq:heavy-profile-bound}
        \Delta(x_N)
        \le
        \left(1-\frac{\mu}{16L}\right)^N\Delta_0
        +\frac{c_p d\sigma_p^2B_{\mathcal R}^2}{\mu}
        s_{\rm ht}^{\vartheta_p}.
\end{equation}
Moreover, on the same event all outer iterates satisfy $\Delta(x_r)\le\mathcal R$ for $0\le r\le N$.
\end{theorem}

The proof is given in \appref{app:heavy-upper-proof}.

\begin{corollary}[Geometric mixing and total transition budget]\label{cor:heavy-geometric}
Suppose, in addition, that \eqref{eq:mixing-def} holds with a valid mixing time $\tmix$.  For a total Markov-transition budget $T$, choose
\begin{equation}\label{eq:heavy-geo-choices}
        N=\left\lceil \frac{32L}{\mu}\log(eT)\right\rceil,
        \qquad
        b=\left\lfloor\frac{T}{N}\right\rfloor,
        \qquad
        m=\left\lceil \tmix\left(1+\log_2\left(\frac{64dNT}{\delta}\right)\right)\right\rceil .
\end{equation}
If $1\le m\le b$ and the admissibility conditions in \Cref{thm:heavy-upper} hold for a radius $\mathcal R\ge4\Delta_0+4$, then the robust method uses at most $T$ Markov transitions and, with probability at least $1-\delta$,
\begin{equation}\label{eq:heavy-geo-rate}
        \Delta(x_N)
        \le
        \widetilde O\left(\frac{\Delta_0}{T}\right)
        +\widetilde O\left(
        \frac{d\sigma_p^2B_{\mathcal R}^2}{\mu}
        \left(\frac{\tmix}{T}\right)^{\frac{2(p-1)}{p}}
        \right),
\end{equation}
where the hidden factors are polylogarithmic in $T$, $1/\delta$, $d$, $L/\mu$, and $\tmix$, and contain no additional polynomial power of $\tmix$.
\end{corollary}

The proof is given in \appref{app:heavy-geometric-proof}.

\begin{remark}[Radius admissibility and dimension dependence]\label{rem:heavy-radius-dimension}
\Cref{thm:heavy-upper,cor:heavy-geometric} are stated with an explicit induction radius
\(\mathcal R\).  In applications one may choose any \(\mathcal R\ge 4\Delta_0+4\) for which
\eqref{eq:heavy-lambda-condition} and \eqref{eq:heavy-batch-size-condition} hold; increasing the
block size \(b\) decreases \(s_{\rm ht}\) and makes these conditions easier to satisfy.  The factor
\(d\) in \eqref{eq:heavy-profile-bound} comes from coordinatewise clipping and a union bound over
coordinates in \Cref{lem:consecutive-block-concentration}.  We do not claim this dimension
dependence is optimal.  The lower bound in \Cref{thm:heavy-lower} is one-dimensional and matches
only the finite-moment exponent and the Markovian effective-sample-size dependence.
\end{remark}

The exponent in \Cref{cor:heavy-geometric} and the dependence on the Markovian effective sample size cannot be improved in general.  The lower bound below reduces PL optimization to estimating the mean of a heavy-tailed sticky Markov chain.

\begin{theorem}[Heavy-tail and mixing lower bound]\label{thm:heavy-lower}
Fix $p\in(1,2]$, $\tau\ge2$, and $\sigma_p\ge1$.  There are universal constants $c,c_0>0$ such that the following holds for every $n\ge 8\tau$.  Consider all algorithms which, after observing $n$ oracle outputs generated during $n$ Markov transitions starting from the fixed initial state $Z_0=0$, output an estimate $\widehat x_n$.  There exists a one-dimensional quadratic PL objective
\[
        f_\theta(x)=\frac12(x-\theta)^2
\]
and a Markovian oracle satisfying Assumption~\ref{ass:heavy-tail} with a valid geometric mixing time at most $c\tau$ and moment scale at most a constant multiple of $\sigma_p$, such that
\begin{equation}\label{eq:heavy-lower-prob}
        \Prob\left(
        f_\theta(\widehat x_n)-f_\theta^\star
        \ge
        c\,\sigma_p^2\left(\frac{\tau}{n}\right)^{\frac{2(p-1)}{p}}
        \right)
        \ge c_0.
\end{equation}
Hence the heavy-tail exponent and the polynomial dependence on the number of effective Markovian samples in \Cref{cor:heavy-geometric} are unavoidable up to logarithmic factors.  Since the construction is one-dimensional, this lower bound does not address the dimension factor in the upper bound.
\end{theorem}

The proof is given in \appref{app:heavy-lower-proof}.

\section{Relaxing global assumptions}\label{sec:relaxations}
The main theorem already removes compactness of the Markov state space and replaces a single geometric mixing-time parameter by a general mixing profile.  The remaining global assumptions can also be localized in the standard way.  We record a formal version because it is useful in overparameterized models, where PL and oracle Lipschitzness often hold only near the initialization.

For $R>0$, write
\[
        \mathcal S_R:=\{x\in\R^d:\Del(x)\le R\}.
\]
Assume that $f$ is globally $L$-smooth, while the PL inequality \eqref{eq:pl}, the stationary-gradient identity \eqref{eq:stationary-gradient}, the pointwise oracle envelope \eqref{eq:abc-g}, and the Lipschitz condition \eqref{eq:g-lip} hold only for $x,y\in\mathcal S_R$.  The pathwise ABC bound \eqref{eq:abc-G} is required only at times for which $x_k\in\mathcal S_R$.

\begin{corollary}[Local PL and local oracle regularity]\label{cor:local}
Let the local assumptions above hold on $\mathcal S_R$.  Let the windows satisfy the admissibility conditions \eqref{eq:admissible-1}-\eqref{eq:admissible-2}, and let $C_\star$ and $\rho_0$ be the constants from \Cref{thm:profile-upper} computed using the local constants on $\mathcal S_R$.  If
\begin{equation}\label{eq:R-local-condition}
        R\ge 2\Del_0+C_\star(1+\rho_0),
\end{equation}
then the conclusion of \Cref{thm:profile-upper} holds.  In particular, with probability at least $1-\delta$, all iterates remain in $\mathcal S_R$ and satisfy \eqref{eq:profile-bound} for every $k\ge1$.
\end{corollary}

The proof is given in \appref{app:local-proof}.

\section{Proof roadmap}\label{sec:proof-roadmap}

We summarize the proof architecture and point to the specific appendix results used in each step.
The paper has three proof mechanisms.  The first is a weighted PL descent recursion for ordinary
SGD.  The second is the lag-blocking argument, which converts the adaptive Markovian descent
observable into residue-class martingale differences plus controlled bias and replacement terms.
The third is the robust clipped-block argument for finite-\(p\) heavy-tailed Markovian gradients.
The appendices are organized so that each of these mechanisms is isolated in a sequence of lemmas
before being combined in the corresponding theorem.

\paragraph{Light-tailed upper bound: \Cref{thm:profile-upper}.}
The proof begins with the deterministic descent structure induced by PL geometry.  In
\Cref{lem:descent}, smoothness from \Cref{ass:objective} is applied to the SGD update
\[
    x_{k+1}=x_k-\alpha_kG_k,
    \qquad
    G_k=g(x_k,Z_k)+M_{k+1}.
\]
The PL inequality in \Cref{ass:objective} then turns the gradient-norm term into a contraction in
the suboptimality \(\Delta_k=f(x_k)-f^\star\).  The ABC envelope in \Cref{ass:abc} is used at this
stage to control the quadratic term \(\alpha_k^2\|G_k\|^2\).  This gives a one-step inequality of
the form
\[
    \Delta_{k+1}
    \le
    (1-\mu\alpha_k)\Delta_k
    +\hbox{controlled deterministic terms}
    -\alpha_k\langle \nabla f(x_k),M_{k+1}\rangle
    -\alpha_k h(x_k,Z_k).
\]
Iterating this recursion produces the weighted representation in \Cref{lem:descent}, with weights
\(w_{\ell,k}\).  The deterministic estimates on these weights are collected in
\Cref{lem:weights}.  These estimates are used repeatedly to show that weighted square sums have
the correct order, for example
\[
    \sum_{\ell<k} w_{\ell,k}^2
    \asymp
    \frac{1}{k+K_0}
\]
up to constants and logarithmic factors.

The stochastic terms in the weighted recursion are controlled on stopped good events.  The
martingale-difference perturbation
\[
    \sum_{\ell=0}^{k-1}
    w_{\ell,k}
    \big\langle \nabla f(x_\ell),M_{\ell+1}\big\rangle
\]
is handled in \Cref{lem:M-noise}.  The stopped event ensures that
\(\Delta_\ell\) remains inside the induction envelope, so the ABC envelope in
\Cref{ass:abc} gives deterministic bounds on the martingale increments.  A union bound over times
then gives the uniform-in-time control needed in the final first-failure argument.

The Markovian descent term
\[
    \sum_{\ell=0}^{k-1} w_{\ell,k}h(x_\ell,Z_\ell),
    \qquad
    h(x,z)=\langle \nabla f(x),g(x,z)-\nabla f(x)\rangle,
\]
is the main difficulty.  It is not a martingale sum because \(x_\ell\) is generated from the past
Markov trajectory, and \(Z_\ell\) need not be stationary conditionally on the past.  For a fixed
terminal time \(k\), the proof introduces the analytical lag \(m=m_k\).  The first \(m\) terms are
kept separate and controlled by \Cref{lem:initial}.  For \(\ell\ge m\), the decomposition
\[
    h(x_\ell,Z_\ell)
    =
    h(x_{\ell-m},Z_\ell)
    +
    \bigl\{h(x_\ell,Z_\ell)-h(x_{\ell-m},Z_\ell)\bigr\}
\]
separates the Markovian term into a delayed part and a replacement error.

The delayed part is where \Cref{ass:mixing} is used.  Since \(x_{\ell-m}\) is
\(\calF_{\ell-m}\)-measurable, the conditional law of \(Z_\ell\) given \(\calF_{\ell-m}\) is
\(P^m(Z_{\ell-m},\cdot)\).  Thus the uniform total-variation mixing profile
\(\varepsilon(m)\) controls the conditional bias of the delayed observable.  After subtracting this
conditional mean, the delayed sum is split into residue classes modulo \(m\).  Along each residue
class, consecutive indices are separated by exactly \(m\) Markov transitions, and the recentered
delayed terms form a martingale-difference sequence with respect to a down-sampled filtration.
This is formalized in \Cref{lem:delayed-mart}.  The key quantitative point is that the residue-class
martingale bounds involve
\[
    m\sum_{\ell<k}w_{\ell,k}^2,
\]
rather than \(m^2\sum_{\ell<k}w_{\ell,k}^2\).  This single power of \(m\) is the source of the
linear polynomial dependence on the mixing window.

The remaining pieces of the Markovian term are controlled separately.  The conditional mixing bias
is bounded in \Cref{lem:mixing-bias} by combining \Cref{ass:mixing} with the envelope for
\(h\).  The replacement error is controlled in \Cref{lem:replacement}.  This uses the Lipschitz
condition in \Cref{ass:abc}, the envelope estimates for \(h\) from \Cref{lem:h-envelope}, and the
pathwise movement bound
\[
    \|x_\ell-x_{\ell-m}\|
    \le
    \sum_{i=\ell-m}^{\ell-1}\alpha_i\|G_i\|.
\]
On the stopped event, the ABC envelope controls the gradients in this movement bound.  Thus the
replacement term is deterministic after conditioning on the good event, and the admissibility
condition \eqref{eq:admissible-2} ensures that it is small enough to be absorbed into the induction
envelope.

The one-step stopped induction estimate is \Cref{lem:induction-step}.  This lemma combines
\Cref{lem:M-noise}, \Cref{lem:delayed-mart}, \Cref{lem:mixing-bias},
\Cref{lem:replacement}, and \Cref{lem:initial}.  The first admissibility condition
\eqref{eq:admissible-1} controls the size of the stochastic concentration terms, while the second
admissibility condition \eqref{eq:admissible-2} controls the drift accumulated by replacing
\(x_\ell\) with \(x_{\ell-m}\).  The proof of \Cref{thm:profile-upper}, given in
\appref{app:main-proof}, then applies a first-failure argument: assume the induction envelope holds
up to time \(k-1\), apply \Cref{lem:induction-step} at time \(k\), and show that the envelope cannot
be violated on the high-probability event.  A union bound over the failure probabilities gives the
simultaneous guarantee \eqref{eq:profile-bound} for all \(k\ge1\).

\paragraph{Geometric mixing corollary: \Cref{cor:geometric-upper}.}
\Cref{thm:profile-upper} is stated for a general uniform total-variation profile
\(\varepsilon(m)\).  Therefore it does not by itself choose a closed-form delay.  Under the
geometric condition \eqref{eq:mixing-def}, the delay \(m_k\) is chosen so that the mixing bias is
small compared with the weighted descent scale.  The elementary logarithm-over-linear estimate
needed for this calculation is \Cref{lem:log-linear}, and the admissibility of the geometric
windows is verified in \Cref{lem:window}.

The role of \Cref{lem:window} is to check that the choices of \(m_k\), \(u_k\), \(q_k\), and
\(\Theta_k\) satisfy \eqref{eq:admissible-1}--\eqref{eq:admissible-2} once \(K_0\) satisfies
\eqref{eq:K-geometric}.  In particular, it shows that
\[
    m_k=\widetilde O(\tmix)
\]
under geometric mixing.  Since the leading variance term in \Cref{lem:delayed-mart} scales like
\(m_k\sum_{\ell<k}w_{\ell,k}^2\), substituting the geometric-window estimates into
\eqref{eq:profile-bound} gives the leading stochastic order
\[
    \widetilde O\!\left(\frac{\tmix}{k+K_0}\right),
\]
with only logarithmic factors hidden in \(\widetilde O(\cdot)\).  This proves
\Cref{cor:geometric-upper}.

\paragraph{Light-tailed lower bound: \Cref{thm:linear-lower}.}
The lower bound is designed to show that the linear polynomial dependence on \(\tmix\) cannot be
improved, even in the simplest PL geometry.  The construction uses the one-dimensional quadratic
objective
\[
    f(x)=\frac12x^2
\]
and a persistent two-state Markov chain.  The lower-bound instance is verified to satisfy
\Cref{ass:objective,ass:abc,ass:mixing} in \Cref{lem:lower-instance-assumptions}.  The mixing time
of the two-state chain is computed in \Cref{lem:two-state-mixing}; the key fact is that the
autocorrelation decays geometrically at rate \(1-2\varepsilon\), so the mixing time is of order
\(1/\varepsilon\).

The SGD recursion on this instance is linear.  \Cref{lem:linear-filter} writes the final iterate as
a weighted linear filter of the Markov chain:
\[
    x_k=-\sigma\sum_{i=0}^{k-1}b_{i,k}Z_i .
\]
The weights corresponding to the recent half of the trajectory are lower bounded in
\Cref{lem:recent-weight-lower}.  Because the chain is persistent, the covariance
\(\mathbb E[Z_iZ_j]\) remains positive for lags of order \(1/\varepsilon\).  Combining this
autocorrelation with the recent-weight lower bound yields the variance lower bound in
\Cref{lem:variance-lower}:
\[
    \mathbb E[x_k^2]
    \gtrsim
    \frac{\sigma^2}{\varepsilon k}.
\]
Since \(t_{\mathrm{mix}}\asymp 1/\varepsilon\), this is exactly the desired expectation lower
bound.

To convert the second-moment lower bound into a constant-probability lower bound,
\Cref{lem:fourth-moment} proves a fourth-moment upper bound for the same linear filter.  The proof
of \Cref{thm:linear-lower} then applies Paley--Zygmund to \(x_k^2\).  This yields both
\eqref{eq:expectation-lower-main} and \eqref{eq:probability-lower-main}; combining these estimates
with \Cref{lem:two-state-mixing} gives \eqref{eq:tmix-lower-main}.  Thus any high-probability
upper bound valid over \Cref{ass:objective,ass:abc,ass:mixing} must have leading stochastic order at
least \(\Omega(\sigma^2\tmix/k)\) at constant confidence.

\paragraph{Heavy-tailed upper bound: \Cref{thm:heavy-upper}.}
The heavy-tailed proof uses a different oracle model and a different algorithm.  Instead of the
pathwise ABC envelope in \Cref{ass:abc}, it uses the stationary finite-\(p\) moment condition in
\Cref{ass:heavy-tail}.  Because this condition gives no deterministic envelope and may not imply a
finite variance, ordinary martingale concentration cannot be applied directly to the raw gradients.
The algorithm therefore holds the iterate fixed over a block, clips each consecutive Markovian
gradient, and averages all clipped samples.  The coordinatewise clipping map is defined in
\eqref{eq:coord-clip-def}, and the resulting all-samples clipped block method is
\Cref{def:robust-batch-method}.

The first ingredient is the stationary clipping calculation in
\Cref{lem:clipped-stationary-bias-var}.  This lemma records two standard consequences of a
finite-\(p\) moment assumption.  First, clipping creates a bias of order
\[
    v^p\lambda^{1-p},
\]
where \(v\) is the local \(p\)-moment scale and \(\lambda\) is the clipping level.  Second, the
clipped variable has second moment of order
\[
    v^p\lambda^{2-p}.
\]
These are the heavy-tailed analogues of bounded variance estimates.  They are applied coordinate by
coordinate and then unioned over the dimension.

The second ingredient is the Markovian clipped-block concentration lemma,
\Cref{lem:consecutive-block-concentration}.  This lemma is the heavy-tailed analogue of the
lag-blocking concentration argument used in the light-tailed proof.  For a block of length \(b\),
the first \(m\) samples are controlled directly by the clipping level.  For the remaining samples,
the proof delays by \(m\) steps, subtracts the conditional mean given the sigma-field one lag in the
past, and splits the centered terms into residue classes modulo \(m\).  Freedman's inequality is
then applied within each residue class.  Summing the residue-class bounds gives the effective
variance term
\[
    \sqrt{
        mbu_{\rm ht}
        \bigl(v^p\lambda^{2-p}+\lambda^2\varepsilon(m)\bigr)
    },
\]
and after division by the block length \(b\), balancing this term with the clipping bias yields the
rate factor
\[
    \left(\frac{mu_{\rm ht}}{b}\right)^{(p-1)/p}.
\]

The third ingredient is deterministic PL descent with inexact gradients,
\Cref{lem:deterministic-inexact-pl}.  Once the clipped block average \(\widehat g_r\) satisfies
\[
    \|\widehat g_r-\nabla f(x_r)\|\le e_r,
\]
this lemma shows that the update with step size \(1/(4L)\) contracts the objective up to an error
of order \(e_r^2/\mu\).  Therefore the heavy-tailed concentration bound controls optimization error
after squaring the gradient-estimation error.

The proof of \Cref{thm:heavy-upper}, given in \appref{app:heavy-upper-proof}, combines these three
ingredients through a stopped induction.  The stopped event has two purposes.  First, the radius
condition \(\Delta(x_r)\le \mathcal R\) ensures that the local moment scale is bounded by
\(B_{\mathcal R}\), so \Cref{lem:consecutive-block-concentration} can be applied at the next
iterate.  Second, the gradient-accuracy event ensures that
\Cref{lem:deterministic-inexact-pl} can be applied to perform one PL descent step.  Thus the proof
alternates between concentration and deterministic descent: concentration gives a sufficiently
accurate clipped gradient estimate at \(x_r\), and deterministic PL descent keeps \(x_{r+1}\) inside
the radius while reducing the suboptimality up to the statistical error floor.

The geometric transition-budget rate in \Cref{cor:heavy-geometric} is obtained by choosing the
number of outer blocks \(N\), the block length \(b\), and the analytical lag \(m\) as functions of
the total transition budget \(T\).  The choice of \(N\) makes the optimization transient negligible,
the choice of \(b\) ensures that the block average has enough effective samples, and the choice of
\(m\) makes the Markovian bias \(\varepsilon(m)\) smaller than the target confidence level.  Under
geometric mixing, \(m=\widetilde O(\tmix)\), so the statistical floor in \Cref{thm:heavy-upper}
becomes
\[
    \widetilde O\!\left(
        \sigma_p^2
        \left(\frac{\tmix}{T}\right)^{2(p-1)/p}
    \right),
\]
up to logarithmic and dimension factors.

\paragraph{Heavy-tailed lower bound: \Cref{thm:heavy-lower}.}
The proof of \Cref{thm:heavy-lower}, given in \appref{app:heavy-lower-proof}, embeds finite-\(p\) mean
estimation into one-dimensional PL optimization.  The algorithm is allowed to be arbitrary, subject
only to observing the oracle transcript and its own internal randomness; it is not told which of the
two lower-bound instances is in force.  The construction uses two objectives with minimizers
separated by a distance of order
\[
    \sigma_p\left(\frac{\tau}{n}\right)^{(p-1)/p}.
\]
Distinguishing these two objectives is therefore necessary for returning a point with smaller
suboptimality than the claimed lower bound.

The Markov chain used in the construction is sticky: it refreshes only once every \(O(\tau)\)
transitions on average.  Hence \(n\) observed Markov transitions contain only \(O(n/\tau)\)
effectively independent opportunities to see an informative heavy-tailed sample.  With constant
probability, the transcript contains no informative refresh.  On this event, the two instances
produce identical oracle transcripts, so any algorithm coupled with the same internal randomness
must output the same point under both instances.  Since the two minimizers are separated, this
common output must have large error for at least one of the two objectives.

The finite-\(p\) moment scale determines the separation between the two instances.  The rare
refresh distribution is chosen so that the stationary \(p\)-moment is bounded by \(\sigma_p^p\),
while the mean shift is of order
\[
    \sigma_p\left(\frac{\tau}{n}\right)^{(p-1)/p}.
\]
Squaring this separation gives the PL suboptimality lower bound
\[
    \Delta(\widehat x_n)
    \gtrsim
    \sigma_p^2
    \left(\frac{\tau}{n}\right)^{2(p-1)/p}
\]
with constant probability.  Since the constructed sticky chain has geometric mixing time
\(O(\tau)\), this matches the dependence in \Cref{cor:heavy-geometric} up to logarithmic and
dimension factors.

\paragraph{Local assumptions: \Cref{cor:local}.}
The localization result in \Cref{cor:local} is proved in \appref{app:local-proof}. The point of this
corollary is that the global smoothness, PL, and oracle-envelope assumptions need only hold on a
sublevel set that contains the stopped trajectory.  The proof repeats the first-failure induction
from \appref{app:main-proof}, but all estimates are restricted to
\[
    \mathcal S_R:=\{x:\Delta(x)\le R\}.
\]
The condition \eqref{eq:R-local-condition} ensures that the induction envelope remains inside
\(\mathcal S_R\).  Consequently, every time the proof invokes \Cref{lem:descent},
\Cref{lem:h-envelope}, \Cref{lem:M-noise}, \Cref{lem:replacement}, or
\Cref{lem:induction-step}, the required constants are the local constants on \(\mathcal S_R\).
Thus the global theorem transfers directly to the local setting once the initial point and the
induction envelope are contained in the sublevel region.

\section{Conclusion}

We developed a sharp high-probability theory for stochastic optimization under the
Polyak--\L{}ojasiewicz condition when gradient samples are generated along an
exogenous Markov chain. In the bounded-oracle regime, we showed that ordinary SGD
achieves a uniform-in-time rate whose leading stochastic term has order
\(\widetilde O(\tmix/(k+K_0))\) under geometric mixing, up to logarithmic factors.
This closes the gap between the \(\widetilde O(\tmix^2/k)\) leading stochastic
order produced by Poisson-equation martingale-range analyses and the
\(\widetilde O(\tmix/k)\) order suggested by expectation bounds. The main
technical device is a lag-blocking argument: the Markovian descent observable is
delayed until the chain has mixed, split into residue classes, and then
controlled by martingale concentration without changing the SGD algorithm.

We also proved that this order is unavoidable. A one-dimensional quadratic PL
objective driven by a persistent two-state Markov chain already exhibits
suboptimality of order \(\sigma^2\tmix/k\) in expectation and with constant
probability. Thus the upper bound identifies the optimal polynomial dependence on
the mixing time in the light-tailed Markovian PL-SGD setting.

Beyond bounded gradients, we studied a finite-\(p\)-moment heavy-tailed regime.
Here robustification must be algorithmic: the method holds the iterate fixed over
short blocks, clips every Markovian gradient in the block, and averages all of
the clipped samples. The resulting high-probability bound has the
effective-sample-size rate
\[
    \widetilde O\!\left(
    \sigma_p^2
    \left(\frac{\tmix}{T}\right)^{2(p-1)/p}
    \right)
\]
under geometric mixing, where \(T\) is the Markov-transition budget. A matching
sticky-chain lower bound shows that both the Markovian effective-sample-size dependence and the
finite-moment exponent are unavoidable up to logarithmic factors. The lower bound is
one-dimensional, so the dimension dependence in the upper bound is not claimed to be optimal.
Together, the light-tailed and heavy-tailed results characterize the statistical price of
Markovian dependence for PL stochastic optimization.

Several directions remain open. First, the heavy-tailed method uses blockwise
updates; an important question is whether one can design a fully online robust
method that updates after every Markovian sample while retaining the same
effective-sample-size rate under minimal nonstationary assumptions. Second, our
analysis treats exogenous Markov chains. Extending the theory to
parameter-dependent kernels, as arise in policy-gradient, actor--critic, and
adaptive-MCMC methods, would require controlling the interaction between iterate
movement, invariant-distribution drift, and mixing. Third, the present results use
uniform total-variation mixing. Developing analogous sharp high-probability
bounds under weaker drift/minorization or \(V\)-uniform ergodicity assumptions
would broaden the scope to noncompact state spaces such as linear dynamical
systems with sub-Gaussian innovations. Finally, it would be valuable to combine
the lag-blocking viewpoint with variance reduction, acceleration, and decentralized
asynchronous updates, where Markovian dependence is intrinsic to the algorithmic
architecture.

\newpage
\bibliographystyle{unsrt}
\bibliography{refs}
\newpage

\appendix

\section{Extended Related work}\label{app:extended-related}

This section reviews the closest work on stochastic approximation, PL optimization, Markovian
gradient methods, concentration for Markov chains, and robust optimization under heavy tails.  The
most relevant comparison for our main theorem is summarized in \Cref{tab:related-markovian};
the paragraphs below then give broader context.

\begin{table}[H]
\centering
\scriptsize
\setlength{\tabcolsep}{3pt}
\renewcommand{\arraystretch}{1.18}
\caption{Closest results on Markovian stochastic optimization. Here \(k\) is the number of SGD
iterations, \(T\) is the total number of Markov transitions, \(t_{\mathrm{mix}}\) or \(\tau\) is a
mixing-time parameter, and \(\tau_{\mathrm{hit}}\) is a hitting-time parameter. The notation
\(\widetilde O(\cdot)\) hides logarithmic factors and problem-dependent constants. Rows are not
directly comparable when they solve a different objective class, use a different algorithmic model,
or prove a different type of lower bound.}
\label{tab:related-markovian}
\begin{tabularx}{\textwidth}{@{}p{0.17\textwidth}p{0.28\textwidth}p{0.25\textwidth}Y@{}}
\toprule
\textbf{Work}
&
\textbf{Setting / method}
&
\textbf{Markovian guarantee or lower bound}
&
\textbf{Relation to this paper}
\\
\midrule

Sun, Sun, and Yin~\cite{sun2018markov}; 
Doan et al.~\cite{doan2020finite}; 
Doan~\cite{doan2022markov}
&
Convex, strongly convex, and nonconvex stochastic optimization with Markovian samples; plain Markov-chain gradient methods.
&
Finite-time convergence bounds with additional dependence on the mixing behavior of the chain.
&
Foundational finite-time analyses for Markovian gradients, but not uniform-in-time high-probability PL last-iterate bounds under an ABC envelope.
\\[0.45em]

Even~\cite{even2023sgd}
&
General Markovian sampling schemes; MC-SGD and the variance-reduced MC-SAG method.
&
Provides lower bounds involving a hitting-time parameter \(\tau_{\mathrm{hit}}\); MC-SAG also depends on \(\tau_{\mathrm{hit}}\).
&
Gives important lower bounds for Markovian sampling, but the parameter and objective regime differ from our PL high-probability \(t_{\mathrm{mix}}/k\) setting.
\\[0.45em]

Beznosikov et al.~\cite{beznosikov2023markovian}
&
First-order methods with Markovian noise for nonconvex optimization, strongly convex optimization, and variational inequalities; randomized batching and MLMC-type estimators.
&
Achieves optimal linear dependence on the mixing time \(\tau\) for several modified first-order methods; also gives matching oracle-complexity lower bounds in some regimes.
&
Shows that a leading order with one polynomial power of the mixing parameter is achievable with altered estimators or batching, but does not analyze ordinary one-sample PL-SGD under an ABC oracle envelope.
\\[0.45em]

Kar et al.~\cite{kar2026markovpl}
&
Smooth PL objectives with Markovian plus martingale-difference noise under an ABC envelope; ordinary one-sample SGD.
&
Uniform high-probability leading stochastic term
\(\widetilde O(t_{\mathrm{mix}}^2/k)\); expectation leading stochastic term
\(\widetilde O(t_{\mathrm{mix}}/k)\).
&
Closest prior work. Leaves open whether the quadratic high-probability mixing dependence is intrinsic or an artifact of the Poisson-equation analysis.
\\[0.45em]

Agrawal, Maguluri, and Zubeldia~\cite{agrawal2026heavytailedmarkov}
&
General stochastic approximation with finite-state Markovian components and heavy-tailed martingale perturbations.
&
Gives concentration and tail classifications for Markovian stochastic approximation, with qualitative sharpness examples.
&
Broader stochastic-approximation result, but not a PL optimization theorem and not a clipped-block effective-sample-size rate.
\\[0.45em]

\textbf{This paper: light-tailed}
&
Smooth PL objectives; exogenous uniformly mixing Markov chain; martingale perturbation; bounded ABC envelope; ordinary one-sample SGD.
&
Uniform high-probability leading stochastic term
\(\widetilde O(t_{\mathrm{mix}}/(k+K_0))\), with matching lower bound
\(\Omega(\sigma^2 t_{\mathrm{mix}}/k)\) in expectation and with constant probability.
&
Closes the high-probability gap for ordinary PL-SGD and proves that the order \(\Omega(t_{\mathrm{mix}}/k)\) is unavoidable at constant confidence, up to logarithmic factors.
\\[0.45em]

\textbf{This paper: heavy-tailed}
&
Smooth PL objectives with Markovian gradients having only finite \(p\)-th stationary moment, \(p\in(1,2]\); all-samples clipped block method.
&
High-probability transition-budget rate
\(\widetilde O\!\bigl(\sigma_p^2(t_{\mathrm{mix}}/T)^{2(p-1)/p}\bigr)\), with a matching sticky-chain lower bound.
&
Shows that the Markovian effective-sample-size dependence and the finite-moment exponent are unavoidable up to logarithmic and dimension factors; dimension optimality is not claimed.
\\

\bottomrule
\end{tabularx}
\end{table}

We next review the surrounding literature in more detail.

\paragraph{Stochastic approximation and SGD.}
The recursion studied here belongs to the classical stochastic approximation lineage initiated by Robbins and Monro \cite{robbins1951stochastic} and Kiefer and Wolfowitz \cite{kiefer1952stochastic}.  The asymptotic ODE viewpoint and martingale methods for recursive algorithms were developed by Ljung \cite{ljung1977analysis}, Kushner and Yin \cite{kushner2003stochastic}, Borkar \cite{borkar2008stochastic}, and Benveniste, Metivier, and Priouret \cite{benveniste1990adaptive}.  Nonasymptotic analyses for stochastic gradient and stochastic approximation methods include Moulines and Bach \cite{moulines2011nonasymptotic}, Rakhlin, Shamir, and Sridharan \cite{rakhlin2011making}, Gower et al.\ \cite{gower2019sgd}, and the broad optimization perspective of Bottou, Curtis, and Nocedal \cite{bottou2018optimization}.  Many of these works assume conditionally unbiased martingale-difference noise or independent sampling.  Our work instead focuses on temporally dependent gradient samples generated by a Markov chain and seeks uniform-in-time high-probability bounds.

\paragraph{PL geometry and SGD under growth conditions.}
The Polyak-\L{}ojasiewicz inequality originated in Polyak's work on gradient methods \cite{polyak1963gradient}; the modern optimization formulation and its relationship to error bounds, restricted secant inequalities, quadratic growth, and related conditions were developed systematically by Karimi, Nutini, and Schmidt \cite{karimi2016linear}.  The condition is weaker than strong convexity and appears in overparameterized learning and control; examples include wide neural networks \cite{liu2022loss}, linear-quadratic control \cite{fazel2018global}, and composed strongly convex objectives. For stochastic gradients under PL geometry, several analyses allow the noise 
magnitude to grow with the local optimization scale rather than requiring a 
uniform bounded-variance assumption; examples include Gower et al.~\cite{gower2019sgd}, 
Li, Zhuang, and Orabona~\cite{li2021second}, and Khaled and Richtarik~\cite{khaled2023better}. 
The ABC envelope is one such growth condition: it bounds the stochastic-gradient 
magnitude by a combination of \(\|\nabla f(x)\|^2\), the objective gap 
\(f(x)-f^\star\), and an additive noise floor. Such growth conditions are important for least squares, interpolation, and 
minibatch sampling because uniform bounded variance is often too restrictive.  High-probability PL results under light-tailed martingale noise have been obtained by Madden, Dall'Anese, and Becker \cite{madden2024subweibull}, while Karandikar and Vidyasagar \cite{karandikar2024rates} study almost-sure rates under biased noise and unbounded variance.  These results do not address the Markovian sampling bias that is central in this paper.

\paragraph{SGD and first-order methods with Markovian sampling.}
Markovian gradient noise has been studied in convex, nonconvex, strongly convex, and PL settings.  Sun, Sun, and Yin \cite{sun2018markov} analyze Markov chain gradient descent for convex problems and inexact gradients.  Doan, Nguyen, Pham, and Romberg \cite{doan2020finite} and Doan \cite{doan2022markov} provide finite-time analyses for stochastic gradient algorithms under Markov randomness.  Even \cite{even2023sgd} studies Markov-chain SGD under mild assumptions, develops lower bounds involving hitting-time quantities, and introduces MC-SAG as a variance-reduced method for Markovian sampling.  Beznosikov et al.\ \cite{beznosikov2023markovian} obtain optimal one-power mixing-time dependence for several first-order methods through randomized batching and multilevel ideas, and also treat variational inequalities.  Kar et al.\ \cite{kar2026markovpl} establish the first uniform-in-time high-probability PL-SGD theorem with Markovian plus martingale-difference noise under an ABC envelope; their high-probability theorem has leading stochastic order \(\widetilde O(\tmix^2/k)\), whereas their expectation theorem has leading stochastic order \(\widetilde O(\tmix/k)\).  The light-tailed part of our paper closes this high-probability mixing gap for the ordinary one-sample SGD recursion, up to logarithmic factors, and proves that the resulting order is sharp.

\paragraph{Markovian stochastic approximation and reinforcement learning.}
A large body of stochastic approximation and reinforcement-learning theory uses Markovian noise.  Poisson-equation decompositions for Markov-driven stochastic approximation go back at least to Metivier and Priouret \cite{metivier1984applications} and are standard in the monographs of Benveniste et al.\ \cite{benveniste1990adaptive}, Kushner and Yin \cite{kushner2003stochastic}, and Borkar \cite{borkar2008stochastic}.  In reinforcement learning, temporal-difference learning and linear stochastic approximation under Markovian data have been studied through finite-sample, asymptotic, and bootstrap viewpoints; representative works include Bhandariet al.\ \cite{bhandari2018finite}, Srikant and Ying \cite{srikant2019finite}, Kaledin et al. \cite{kaledin2020finite},  Liu et al. \cite{liu2025odemethod}, and Ganesh et al. \cite{ganeshsharper}.  Recent work also treats nonexpansive or two-time-scale stochastic approximation under Markovian data \cite{blaser2026nonexpansive,chandak2025twotimescale}.  Agrawal et al. \cite{agrawal2026heavytailedmarkov} study concentration for general stochastic approximation with a finite-state Markovian component and heavy-tailed martingale noise.  Our focus is different: we study last-iterate optimization under PL geometry, allow the Markovian gradient component itself to be heavy-tailed in the robust regime, and prove matching optimization lower bounds.

\paragraph{Concentration for Markov chains.}
Classical concentration inequalities for Markov chains include Hoeffding-type bounds for uniformly ergodic chains \cite{glynn2002hoeffding}, Chernoff and spectral-gap bounds for finite chains \cite{lezaud1998chernoff}, and the coupling and spectral inequalities of Paulin \cite{paulin2015concentration}.  These inequalities show that additive functionals of a Markov chain often behave like averages over an effective sample size reduced by a mixing parameter.  They are not directly plug-and-play for adaptive SGD, however, because the observable at time $k$ depends on the iterate $x_k$, which is itself a function of the past data.  Poisson-equation methods handle this adaptivity but can introduce a squared mixing factor when combined with worst-case martingale range bounds.  Our lag-blocking proof is a direct concentration argument for the weighted, iterate-adapted additive functional: the delay creates approximate independence, and splitting into residue classes converts the delayed sum into martingale differences without changing the algorithm.

\paragraph{Decentralized, MCMC, privacy, and system-identification motivations.}
Markovian sampling appears naturally in token-based and random-walk decentralized optimization.  Early randomized incremental and distributed methods include Johansson, Rabi, and Johansson \cite{johansson2007randomized} and Duchi, Agarwal, and Wainwright \cite{duchi2012dual}; more recent token frameworks include Walkman \cite{mao2020walkman} and the principled token-algorithm design of Hendrikx \cite{hendrikx2023token}.  Markov chains also arise when gradients are estimated using MCMC samples, where mixing controls the bias and variance of the gradient estimator.  In privacy-preserving learning, subsampling mechanisms for privacy amplification, including differentially private SGD \cite{abadi2016deep}, banded matrix-factorization schemes \cite{choquettechoo2023bandmf}, and minimum-separation subsampling \cite{dong2026minsep}, can produce temporally dependent minibatches.  In online system identification, observations from stable dynamical systems generate Markovian regressors, as in streaming linear system identification \cite{kowshik2021streaming}.  These examples motivate a theory that treats Markovian dependence as a first-order parameter rather than as a lower-order perturbation.

\paragraph{Heavy-tailed stochastic optimization.}
Heavy-tailed gradient noise has motivated robust stochastic methods that avoid sub-Gaussian or bounded-variance assumptions.  For SGD under infinite variance, Wang et al.\ \cite{wang2021infinite} derive convergence rates in expectation.  High-probability nonconvex stochastic optimization with finite $p$th moments has been studied by Cutkosky and Mehta \cite{cutkosky2021heavy}, Gorbunov et al.\ \cite{gorbunov2021nonsmooth}, Sadiev et al.\ \cite{sadiev2023unbounded}, and Nguyen et al.\ \cite{nguyen2023clippedsgd}, with clipping, normalization, momentum, or mirror-descent variants playing a central role.  Puchkin et al.\ \cite{puchkin2024barrier} use smoothed median-of-means ideas to stabilize gradients, Huebler, Fatkhullin, and He \cite{hubler2025normalization} study normalized SGD under heavy tails, and Armacki et al.\ \cite{armacki2025nonlinear} give high-probability guarantees for a class of nonlinear SGD transformations including clipping, quantization, and sign methods.  These works largely concern independent, online, or martingale-type noise rather than Markovian gradient samples.  The heavy-tailed part of this paper combines clipping with consecutive Markovian blocks and obtains the effective-sample-size rate dictated by the mixing profile while using all samples inside each block.

\paragraph{Robust statistics, heavy-tailed mean estimation, and lower bounds.}
Our heavy-tailed lower bound is based on the classical connection between stochastic optimization over quadratics and mean estimation.  Robust mean estimation under heavy tails is a mature subject, including Catoni's estimator \cite{catoni2012challenging}, median-of-means and tournament estimators, and the survey of Lugosi and Mendelson \cite{lugosi2019mean}; Devroye et al.\ \cite{devroye2016subgaussian} give sub-Gaussian mean estimators under finite variance.  Robust stochastic approximation ideas also appear in Juditsky and Nemirovski \cite{juditsky2008large} and Nemirovski et al.\ \cite{nemirovski2009robust}.  For Markovian data, recent robust reinforcement-learning work analyzes median-of-means or clipping under time-correlated samples, including robust TD learning under contamination \cite{maity2025robusttd} and heavy-tailed rewards \cite{cayci2023heavytd}.  We use a sticky-chain construction to show that $k$ Markovian samples can contain only $k/\tmix$ effective refreshes, and therefore the finite-$p$ mean-estimation exponent $2(p-1)/p$ is unavoidable for PL optimization as well.

\paragraph{Positioning.}
The present paper therefore has two complementary roles.  In the light-tailed regime, it sharpens the high-probability Markovian PL-SGD theory by matching the order suggested by expectation bounds and by proving a corresponding lower bound.  In the heavy-tailed regime, it extends robust clipped-gradient guarantees from independent or martingale sampling to Markovian gradient oracles through all-samples clipped blocks, and proves that both the Markovian effective-sample-size dependence and finite-moment exponent are unavoidable, up to logarithmic and dimension factors.

\section{Deterministic backbone for the light-tailed upper bound}
\label{app:deterministic}

This appendix collects deterministic estimates used throughout the proof of 
\Cref{thm:profile-upper}. Lemma~\ref{lem:grad-upper} converts smoothness and 
lower boundedness into the gradient upper bound \eqref{eq:grad-upper}. 
Lemma~\ref{lem:descent} derives the weighted PL descent recursion. 
Lemma~\ref{lem:weights} gives the weight estimates used in all concentration 
bounds, and Lemma~\ref{lem:h-envelope} controls the Markovian descent observable 
\(h(x,z)\). Together these lemmas form the deterministic backbone of the 
lag-blocking proof.

\begin{lemma}[Gradient upper bound under smoothness]\label{lem:grad-upper}
Under Assumption \ref{ass:objective}, \eqref{eq:grad-upper} holds.
\end{lemma}

\begin{proof}
Fix $x\in\R^d$.  By $L$-smoothness, for every $y$,
\[
        f(y)\le f(x)+\ip{\nabla f(x)}{y-x}+\frac{L}{2}\norm{y-x}^2.
\]
Take $y=x-L^{-1}\nabla f(x)$.  Then
\[
        f\left(x-L^{-1}\nabla f(x)\right)
        \le f(x)-\frac{1}{2L}\norm{\nabla f(x)}^2.
\]
Since $f^\star\le f(x-L^{-1}\nabla f(x))$, rearranging gives
\[
        \norm{\nabla f(x)}^2\le 2L(f(x)-f^\star)=2L\Del(x).
\]
\end{proof}

\begin{lemma}[Weighted descent recursion]\label{lem:descent}
Under Assumptions \ref{ass:objective} and \ref{ass:noise}, if \eqref{eq:basic-K} holds, then for every $k\ge1$,
\begin{align}\label{eq:descent}
\Del_k
&\le
\Del_0\zeta_{0,k-1}
+\frac{LC}{2}\sum_{\ell=0}^{k-1}\alpha_\ell^2\zeta_{\ell+1,k-1}
-\sum_{\ell=0}^{k-1}w_{\ell,k}\ip{\nabla f(x_\ell)}{M_{\ell+1}}
-\sum_{\ell=0}^{k-1}w_{\ell,k}h(x_\ell,Z_\ell).
\end{align}
\end{lemma}

\begin{proof}
By $L$-smoothness and the update $x_{n+1}=x_n-\alpha_nG_n$,
\begin{align*}
f(x_{n+1})
&\le f(x_n)-\alpha_n\ip{\nabla f(x_n)}{G_n}+\frac{L\alpha_n^2}{2}\norm{G_n}^2 \\
&= f(x_n)-\alpha_n\norm{\nabla f(x_n)}^2
        -\alpha_n h(x_n,Z_n)
        -\alpha_n\ip{\nabla f(x_n)}{M_{n+1}}
        +\frac{L\alpha_n^2}{2}\norm{G_n}^2.
\end{align*}
Using \eqref{eq:abc-G}, the PL inequality \eqref{eq:pl}, and the gradient upper bound \eqref{eq:grad-upper},
\begin{align*}
\Del_{n+1}
&\le
\Del_n
-\left(\alpha_n-\frac{AL\alpha_n^2}{2}\right)\norm{\nabla f(x_n)}^2
+\frac{BL\alpha_n^2}{2}\Del_n
+\frac{LC\alpha_n^2}{2} \\
&\quad
-\alpha_n h(x_n,Z_n)
-\alpha_n\ip{\nabla f(x_n)}{M_{n+1}} \\
&\le
\left(1-2\mu\alpha_n+\frac{(2\mu A+B)L\alpha_n^2}{2}\right)\Del_n
+\frac{LC\alpha_n^2}{2}
-\alpha_n h(x_n,Z_n)
-\alpha_n\ip{\nabla f(x_n)}{M_{n+1}}.
\end{align*}
The condition \eqref{eq:basic-K} implies
\[
        \frac{(2\mu A+B)L\alpha_n^2}{2}\le \mu\alpha_n,
\]
so
\[
\Del_{n+1}
\le (1-\mu\alpha_n)\Del_n
+\frac{LC\alpha_n^2}{2}
-\alpha_n h(x_n,Z_n)
-\alpha_n\ip{\nabla f(x_n)}{M_{n+1}}.
\]
Iterating this affine recursion from $n=0$ to $n=k-1$ gives \eqref{eq:descent}.
\end{proof}

\begin{lemma}[Weight estimates]\label{lem:weights}
Let $\beta:=\mu a\ge2$ and assume $K_0\ge \beta$.  There is a constant $c_w<\infty$, depending only on $a$ and $\mu$, such that for every $k\ge1$,
\begin{align}
\zeta_{0,k-1}&\le \frac{K_0}{k+K_0},\label{eq:weight-init}\\
\sum_{\ell=0}^{k-1}\alpha_\ell^2\zeta_{\ell+1,k-1}&\le \frac{c_w}{k+K_0},\label{eq:weight-alpha2}\\
w_{\ell,k}&\le c_w\frac{\ell+K_0}{(k+K_0)^2},\qquad 0\le \ell\le k-1,\label{eq:weight-point}\\
\sum_{\ell=0}^{k-1}\frac{w_{\ell,k}^2}{\ell+K_0}&\le \frac{c_w}{(k+K_0)^2},\label{eq:weight-sum1}\\
\sum_{\ell=0}^{k-1}\frac{w_{\ell,k}^2}{(\ell+K_0)^2}&\le \frac{c_w}{(k+K_0)^3}.\label{eq:weight-sum2}
\end{align}
Moreover, for every integer $m\in\{1,\ldots,k\}$,
\begin{align}
\sum_{\ell=m}^{k-1}\frac{w_{\ell,k}^2}{\ell-m+K_0}
&\le \left(1+\frac{m}{K_0}\right)\frac{c_w}{(k+K_0)^2},\label{eq:weight-delay1}\\
\sum_{\ell=m}^{k-1}\frac{w_{\ell,k}^2}{(\ell-m+K_0)^2}
&\le \left(1+\frac{m}{K_0}\right)^2\frac{c_w}{(k+K_0)^3}.\label{eq:weight-delay2}
\end{align}
\end{lemma}

\begin{proof}
Since $K_0\ge\beta$, we have $0\le \mu\alpha_j\le1$.  For $0\le\ell\le k-1$,
\begin{align*}
\zeta_{\ell+1,k-1}
&=\prod_{j=\ell+1}^{k-1}\left(1-\frac{\beta}{j+K_0}\right)
\le \exp\left(-\beta\sum_{j=\ell+1}^{k-1}\frac{1}{j+K_0}\right).
\end{align*}
Using the integral comparison
\[
        \sum_{j=\ell+1}^{k-1}\frac{1}{j+K_0}
        \ge \int_{\ell+1}^{k}\frac{dt}{t+K_0}
        =\log\frac{k+K_0}{\ell+K_0+1},
\]
we get
\[
        \zeta_{\ell+1,k-1}
        \le \left(\frac{\ell+K_0+1}{k+K_0}\right)^\beta.
\]
Because $(\ell+K_0+1)/(\ell+K_0)\le 1+K_0^{-1}$ and $\beta/K_0\le1$,
\[
        \left(\frac{\ell+K_0+1}{k+K_0}\right)^\beta
        \le e\left(\frac{\ell+K_0}{k+K_0}\right)^\beta.
\]
Increasing constants from $e$ to $e^2$ if needed, this bound also covers the empty-product edge cases.  Since $\beta\ge2$,
\[
        w_{\ell,k}
        =\frac{a}{\ell+K_0}\zeta_{\ell+1,k-1}
        \le ae^2\frac{(\ell+K_0)^{\beta-1}}{(k+K_0)^\beta}
        \le ae^2\frac{\ell+K_0}{(k+K_0)^2},
\]
which proves \eqref{eq:weight-point}.  The initial product similarly satisfies
\[
        \zeta_{0,k-1}\le \left(\frac{K_0}{k+K_0}\right)^\beta
        \le \frac{K_0}{k+K_0},
\]
which proves \eqref{eq:weight-init}.

For the first square sum, \eqref{eq:weight-point} yields
\[
\sum_{\ell=0}^{k-1}\frac{w_{\ell,k}^2}{\ell+K_0}
\le \frac{a^2e^4}{(k+K_0)^4}\sum_{\ell=0}^{k-1}(\ell+K_0)
\le \frac{a^2e^4}{(k+K_0)^4}(k+K_0)^2.
\]
This is \eqref{eq:weight-sum1} after enlarging $c_w$.  The second square sum follows from
\[
\sum_{\ell=0}^{k-1}\frac{w_{\ell,k}^2}{(\ell+K_0)^2}
\le \frac{a^2e^4}{(k+K_0)^4}\sum_{\ell=0}^{k-1}1
\le \frac{a^2e^4}{(k+K_0)^3}.
\]
Also,
\[
\sum_{\ell=0}^{k-1}\alpha_\ell^2\zeta_{\ell+1,k-1}
=\sum_{\ell=0}^{k-1}\alpha_\ell w_{\ell,k}
\le \frac{a^2e^2}{(k+K_0)^2}\sum_{\ell=0}^{k-1}1
\le \frac{a^2e^2}{k+K_0},
\]
which proves \eqref{eq:weight-alpha2}.  Finally, for $\ell\ge m$,
\[
        \frac{1}{\ell-m+K_0}
        =\frac{\ell+K_0}{\ell-m+K_0}\frac{1}{\ell+K_0}
        \le\left(1+\frac{m}{K_0}\right)\frac{1}{\ell+K_0}.
\]
Squaring the displayed inequality gives the corresponding squared denominator bound.  Combining these with \eqref{eq:weight-sum1} and \eqref{eq:weight-sum2} proves \eqref{eq:weight-delay1} and \eqref{eq:weight-delay2}.
\end{proof}

\begin{lemma}[Envelope and Lipschitz bounds for the descent observable]\label{lem:h-envelope}
Let
\begin{equation}\label{eq:r-def}
        r:=2AL+B,
\end{equation}
and, for $D\ge0$, define
\begin{equation}\label{eq:H-def}
        H(D):=\sqrt{2LD}\left(\sqrt{rD+C}+\sqrt{2LD}\right).
\end{equation}
There are constants $c_H,c_{\mathrm{lip}}<\infty$, depending only on $L,L_g,A,B,C$, such that:
\begin{enumerate}[label=(\roman*)]
\item if $\Del(x)\le D$, then $|h(x,z)|\le H(D)$ for every $z$;
\item $H(D)^2\le c_H(D+D^2)$ for every $D\ge0$;
\item if $\Del(x)\le D_x$ and $\Del(y)\le D_y$, then
\begin{equation}\label{eq:h-lip-bound}
        |h(x,z)-h(y,z)|
        \le c_{\mathrm{lip}}(1+\sqrt{D_x}+\sqrt{D_y})\norm{x-y}.
\end{equation}
\end{enumerate}
\end{lemma}

\begin{proof}
If $\Del(x)\le D$, then \eqref{eq:grad-upper} gives $\norm{\nabla f(x)}\le\sqrt{2LD}$.  Moreover, \eqref{eq:abc-g} and \eqref{eq:grad-upper} give
\[
        \norm{g(x,z)}\le \sqrt{A\norm{\nabla f(x)}^2+B\Del(x)+C}
        \le \sqrt{rD+C}.
\]
Thus
\[
        \norm{\xi(x,z)}\le \norm{g(x,z)}+\norm{\nabla f(x)}
        \le \sqrt{rD+C}+\sqrt{2LD}.
\]
The Cauchy-Schwarz inequality proves $|h(x,z)|\le H(D)$.

Next,
\begin{align*}
H(D)^2
&=2LD\left(\sqrt{rD+C}+\sqrt{2LD}\right)^2 \\
&\le 4LD(rD+C+2LD)
\le c_H(D+D^2)
\end{align*}
for a finite $c_H$ depending only on $L,A,B,C$.

For the Lipschitz bound, write
\begin{align*}
|h(x,z)-h(y,z)|
&\le \norm{\nabla f(x)-\nabla f(y)}\norm{\xi(x,z)}
    +\norm{\nabla f(y)}\norm{\xi(x,z)-\xi(y,z)}.
\end{align*}
By $L$-smoothness, $\norm{\nabla f(x)-\nabla f(y)}\le L\norm{x-y}$.  The previous envelope gives $\norm{\xi(x,z)}\le c(1+\sqrt{D_x})$.  Also,
\[
        \norm{\xi(x,z)-\xi(y,z)}
        \le \norm{g(x,z)-g(y,z)}+\norm{\nabla f(x)-\nabla f(y)}
        \le (L_g+L)\norm{x-y},
\]
and $\norm{\nabla f(y)}\le\sqrt{2LD_y}$.  Combining these estimates proves \eqref{eq:h-lip-bound}.
\end{proof}

\section{Lag-blocking concentration estimates}
\label{app:blocking}

This appendix proves the probabilistic estimates that control the stochastic 
terms in the weighted descent recursion of Lemma~\ref{lem:descent}. 
Lemma~\ref{lem:M-noise} handles the martingale-difference perturbation. 
Lemmas~\ref{lem:delayed-mart}--\ref{lem:initial} implement the lag-blocking 
decomposition of the adaptive Markovian term into a delayed martingale part, 
a mixing-bias part, a replacement error, and an initial-window term. These 
bounds are combined in \appref{app:main-proof}.

The lemmas in this section are stated for a fixed target time $k$.  Let $N:=k+K_0$ and let $m\in\{1,\ldots,k\}$.  Let $(\mathcal K_j)_{j=0}^{k-1}$ be a nonincreasing family of events, meaning $\mathcal K_{j+1}\subseteq\mathcal K_j$, with $\mathcal K_j\in\calF_j$.  Assume that on $\mathcal K_j$,
\begin{equation}\label{eq:local-envelope}
        \Del_j\le \frac{\Lambda}{j+K_0},
        \qquad 0\le j\le k-1,
\end{equation}
for some deterministic $\Lambda\ge1$.

\begin{lemma}[Martingale-difference noise]\label{lem:M-noise}
For every $\eta\in(0,1)$, with probability at least $1-\eta$,
\begin{equation}\label{eq:M-noise-final}
\1_{\mathcal K_{k-1}}
\left|\sum_{\ell=0}^{k-1}w_{\ell,k}\ip{\nabla f(x_\ell)}{M_{\ell+1}}\right|
\le
\frac{c_M}{N}
\sqrt{\log\frac{2}{\eta}\left(\Lambda+\frac{\Lambda^2}{N}\right)},
\end{equation}
where $c_M<\infty$ depends only on $a,\mu,L,A,B,C$.
\end{lemma}

\begin{proof}
Define
\[
        Q_{\ell+1}:=w_{\ell,k}\ip{\nabla f(x_\ell)}{M_{\ell+1}}\1_{\mathcal K_\ell},
        \qquad 0\le \ell\le k-1.
\]
Since $\mathcal K_\ell\in\calF_\ell$, $w_{\ell,k}$ is deterministic, and $\E[M_{\ell+1}\mid\calF_\ell]=0$, the sequence $(Q_{\ell+1})$ is a martingale difference sequence.  On $\mathcal K_\ell$, \eqref{eq:abc-G}, \eqref{eq:abc-g}, and \eqref{eq:grad-upper} imply
\[
        \norm{M_{\ell+1}}
        \le \norm{G_\ell}+\norm{g(x_\ell,Z_\ell)}
        \le 2\sqrt{r\frac{\Lambda}{\ell+K_0}+C},
\]
where $r=2AL+B$.  Also $\norm{\nabla f(x_\ell)}\le\sqrt{2L\Lambda/(\ell+K_0)}$.  Hence
\[
        |Q_{\ell+1}|
        \le c w_{\ell,k}
        \sqrt{\frac{\Lambda}{\ell+K_0}}
        \sqrt{1+\frac{\Lambda}{\ell+K_0}}
\]
for a constant $c$ depending only on $L,A,B,C$.  Therefore
\[
        |Q_{\ell+1}|^2
        \le c w_{\ell,k}^2
        \left(\frac{\Lambda}{\ell+K_0}+\frac{\Lambda^2}{(\ell+K_0)^2}\right).
\]
Define the deterministic bounds
\[
        b_{\ell,k}^2
        :=c w_{\ell,k}^2
        \left(\frac{\Lambda}{\ell+K_0}+\frac{\Lambda^2}{(\ell+K_0)^2}\right),
        \qquad 0\le \ell\le k-1,
\]
so that $|Q_{\ell+1}|\le b_{\ell,k}$ almost surely.  Using \eqref{eq:weight-sum1}--\eqref{eq:weight-sum2},
\[
        \sum_{\ell=0}^{k-1}b_{\ell,k}^2
        \le \frac{c}{N^2}\left(\Lambda+\frac{\Lambda^2}{N}\right).
\]
The Azuma-Hoeffding inequality applied with the deterministic increment bounds $b_{\ell,k}$ gives
\[
        \Pp\left(\left|\sum_{\ell=0}^{k-1}Q_{\ell+1}\right|
        >\sqrt{2\log(2/\eta)\sum_{\ell=0}^{k-1}b_{\ell,k}^2}\right)
        \le \eta.
\]
Combining the last two displays yields the bound in \eqref{eq:M-noise-final}.  On $\mathcal K_{k-1}$, monotonicity gives $\1_{\mathcal K_\ell}=1$ for every $\ell\le k-1$, so $\sum Q_{\ell+1}$ equals the martingale term in \eqref{eq:M-noise-final}.  This proves the result after absorbing numerical constants into $c_M$.
\end{proof}

\begin{lemma}[Delayed residue-class martingale]\label{lem:delayed-mart}
For $\ell\ge m$, define
\begin{equation}\label{eq:Y-def}
        Y_\ell:=w_{\ell,k}h(x_{\ell-m},Z_\ell),
        \qquad
        \overline Y_\ell:=\E[Y_\ell\mid\calF_{\ell-m}].
\end{equation}
For every $\eta\in(0,1)$, with probability at least $1-\eta$,
\begin{align}\label{eq:delayed-mart-final}
\1_{\mathcal K_{k-1}}
\left|\sum_{\ell=m}^{k-1}(Y_\ell-\overline Y_\ell)\right|
&\le
\frac{c_B}{N}
\sqrt{m\log\left(\frac{2m}{\eta}\right)
\left(1+\frac{m}{K_0}\right)
\left(\Lambda+\frac{\Lambda^2}{N}\left(1+\frac{m}{K_0}\right)\right)},
\end{align}
where $c_B<\infty$ depends only on $a,\mu,L,A,B,C$.
\end{lemma}

\begin{proof}
For $\ell\ge m$, set
\[
        D_\ell:=(Y_\ell-\overline Y_\ell)\1_{\mathcal K_{\ell-m}}.
\]
Since $x_{\ell-m}$ and $\1_{\mathcal K_{\ell-m}}$ are $\calF_{\ell-m}$-measurable,
\[
        \E[D_\ell\mid\calF_{\ell-m}]=0.
\]
On $\mathcal K_{\ell-m}$, \eqref{eq:local-envelope} and Lemma \ref{lem:h-envelope} give
\[
        |Y_\ell|
        \le w_{\ell,k}H\left(\frac{\Lambda}{\ell-m+K_0}\right)
        =:b_\ell.
\]
The same bound holds for $|\overline Y_\ell|$ on $\mathcal K_{\ell-m}$, because conditional expectation is dominated by the conditional expectation of the absolute value.  Hence $|D_\ell|\le2b_\ell$.

We now spell out the filtration structure, since this is the point at which the lag creates martingale differences.  Fix a residue class $r\in\{0,1,\ldots,m-1\}$.  List the indices $\ell=r+jm$ that lie in $\{m,m+1,\ldots,k-1\}$ in increasing order, say $\ell_1<\cdots<\ell_s$.  Define $\mathcal H_0:=\calF_{\ell_1-m}$ and $\mathcal H_i:=\calF_{\ell_i}$ for $1\le i\le s$.  For the first index, $D_{\ell_1}$ has conditional mean zero given $\mathcal H_0=\calF_{\ell_1-m}$.  For $i\ge2$, the residue-class spacing gives $\ell_i-m=\ell_{i-1}$, and hence
\[
        \E[D_{\ell_i}\mid\mathcal H_{i-1}]
        =\E[D_{\ell_i}\mid\calF_{\ell_i-m}]=0.
\]
Thus $(D_{\ell_i})_{i=1}^s$ is a martingale-difference sequence with respect to the down-sampled filtration $(\mathcal H_i)_{i=0}^s$.  No independence across different residue classes is used; each class is controlled separately, and the resulting high-probability bounds are combined by a union bound.  Azuma-Hoeffding gives
\[
        \Pp\left(\left|\sum_{i=1}^sD_{\ell_i}\right|
        >\sqrt{8u\sum_{i=1}^sb_{\ell_i}^2}\right)
        \le 2e^{-u}.
\]
Take $u=\log(2m/\eta)$ and union bound over the $m$ residue classes.  With probability at least $1-\eta$, all residue-class inequalities hold.  On that event,
\begin{align*}
\left|\sum_{\ell=m}^{k-1}D_\ell\right|
&\le \sum_{r=0}^{m-1}\left|\sum_{\ell\equiv r\, (\mathrm{mod}\,m)}D_\ell\right| \\
&\le \sum_{r=0}^{m-1}\sqrt{8u\sum_{\ell\equiv r\, (\mathrm{mod}\,m)}b_\ell^2} \\
&\le \sqrt{8um\sum_{\ell=m}^{k-1}b_\ell^2},
\end{align*}
where the last step is Cauchy-Schwarz.

It remains to bound the square sum.  Lemma \ref{lem:h-envelope} gives
\[
        b_\ell^2
        \le c_H w_{\ell,k}^2
        \left(\frac{\Lambda}{\ell-m+K_0}+\frac{\Lambda^2}{(\ell-m+K_0)^2}\right).
\]
Using \eqref{eq:weight-delay1}--\eqref{eq:weight-delay2},
\[
        \sum_{\ell=m}^{k-1}b_\ell^2
        \le \frac{c}{N^2}\left(1+\frac{m}{K_0}\right)\Lambda
        +\frac{c}{N^3}\left(1+\frac{m}{K_0}\right)^2\Lambda^2.
\]
Substituting this estimate into the previous display proves \eqref{eq:delayed-mart-final}.  On $\mathcal K_{k-1}$, all indicators $\1_{\mathcal K_{\ell-m}}$ are one, so the sum of $D_\ell$ is the displayed delayed centered sum.
\end{proof}

\begin{lemma}[Mixing bias]\label{lem:mixing-bias}
Assume $m<k$ and set $\varepsilon_m:=\varepsilon(m)$, where $\varepsilon(\cdot)$ is the uniform total-variation mixing profile in Assumption~\ref{ass:mixing}.  Then, on $\mathcal K_{k-1}$,
\begin{equation}\label{eq:mixing-bias-final}
\left|\sum_{\ell=m}^{k-1}\overline Y_\ell\right|
\le
\frac{c_{\mathrm{mix}}\varepsilon_m}{N}
\sqrt{N\left(1+\frac{m}{K_0}\right)
\left(\Lambda+\frac{\Lambda^2}{N}\left(1+\frac{m}{K_0}\right)\right)},
\end{equation}
where $c_{\mathrm{mix}}<\infty$ depends only on $a,\mu,L,A,B,C$.
\end{lemma}

\begin{proof}
Because $x_{\ell-m}$ is $\calF_{\ell-m}$-measurable and $\int h(x_{\ell-m},z)\pi(dz)=0$, the full-filtration Markov property gives
\begin{align*}
\overline Y_\ell
&=w_{\ell,k}\int h(x_{\ell-m},z)\bigl(P^m(Z_{\ell-m},dz)-\pi(dz)\bigr).
\end{align*}
On $\mathcal K_{\ell-m}$, Lemma \ref{lem:h-envelope} and the dual total-variation bound \eqref{eq:tv-dual} imply
\[
        |\overline Y_\ell|
        \le 2\varepsilon_m w_{\ell,k}H\left(\frac{\Lambda}{\ell-m+K_0}\right)
        =2\varepsilon_m b_\ell.
\]
Therefore, on $\mathcal K_{k-1}$,
\[
        \left|\sum_{\ell=m}^{k-1}\overline Y_\ell\right|
        \le 2\varepsilon_m\sum_{\ell=m}^{k-1}b_\ell
        \le 2\varepsilon_m\sqrt{k\sum_{\ell=m}^{k-1}b_\ell^2}.
\]
The square-sum estimate from the proof of Lemma \ref{lem:delayed-mart} and $k\le N$ give \eqref{eq:mixing-bias-final}.
\end{proof}

\begin{lemma}[Replacement error]\label{lem:replacement}
On $\mathcal K_{k-1}$,
\begin{align}\label{eq:replacement-final}
&\left|\sum_{\ell=m}^{k-1}w_{\ell,k}\bigl(h(x_\ell,Z_\ell)-h(x_{\ell-m},Z_\ell)\bigr)\right| \\
&\quad\le c_R\Bigg[
\frac{m}{N}
+\frac{m^2}{N^2}\left(1+\log\frac{N}{K_0}\right)
+\frac{m\Lambda}{N^2}\left(1+\log\frac{N}{K_0}+\frac{m}{K_0}\right) \notag\\
&\hspace{4.8cm}
+\frac{m\sqrt\Lambda}{N^{3/2}}
+\frac{m^2\sqrt\Lambda}{\sqrt{K_0}N^2}
\Bigg],
\end{align}
where $c_R<\infty$ depends only on $a,\mu,L,L_g,A,B,C$.
\end{lemma}

\begin{proof}
On $\mathcal K_j$, the ABC bound and \eqref{eq:grad-upper} imply
\[
        \norm{G_j}
        \le \sqrt{r\frac{\Lambda}{j+K_0}+C}
        \le c_G\left(1+\sqrt{\frac{\Lambda}{j+K_0}}\right)
\]
for a constant $c_G$ depending only on $L,A,B,C$.  Hence, for $\ell\ge m$,
\begin{align}\label{eq:movement-proof}
\norm{x_\ell-x_{\ell-m}}
&\le \sum_{j=\ell-m}^{\ell-1}\alpha_j\norm{G_j} \\
&\le ac_G\sum_{j=\ell-m}^{\ell-1}\frac{1}{j+K_0}
\left(1+\sqrt{\frac{\Lambda}{j+K_0}}\right) \notag\\
&\le c\frac{m}{\ell-m+K_0}
\left(1+\sqrt{\frac{\Lambda}{\ell-m+K_0}}\right).\notag
\end{align}
The last inequality uses that $j+K_0\ge \ell-m+K_0$ over the summation range.

By Lemma \ref{lem:h-envelope}, \eqref{eq:local-envelope}, and \eqref{eq:movement-proof},
\begin{align*}
&|h(x_\ell,Z_\ell)-h(x_{\ell-m},Z_\ell)| \\
&\quad\le c_{\mathrm{lip}}
\left(1+\sqrt{\frac{\Lambda}{\ell+K_0}}+\sqrt{\frac{\Lambda}{\ell-m+K_0}}\right)
\norm{x_\ell-x_{\ell-m}} \\
&\quad\le c\frac{m}{\ell-m+K_0}
\left(1+\sqrt{\frac{\Lambda}{\ell-m+K_0}}+\frac{\Lambda}{\ell-m+K_0}\right).
\end{align*}
Thus
\begin{align}\label{eq:replacement-sum-proof}
&\left|\sum_{\ell=m}^{k-1}w_{\ell,k}\bigl(h(x_\ell,Z_\ell)-h(x_{\ell-m},Z_\ell)\bigr)\right| \\
&\quad\le cm\sum_{\ell=m}^{k-1}w_{\ell,k}\left(
\frac{1}{t_\ell}+\frac{\sqrt\Lambda}{t_\ell^{3/2}}+\frac{\Lambda}{t_\ell^2}
\right),\notag
\end{align}
where $t_\ell:=\ell-m+K_0$.  Since $\ell+K_0=t_\ell+m$, \eqref{eq:weight-point} gives
\[
        w_{\ell,k}\le c\frac{t_\ell+m}{N^2}.
\]
The elementary estimates
\begin{align*}
\sum_{\ell=m}^{k-1}1&\le k\le N,
&\sum_{\ell=m}^{k-1}\frac1{t_\ell}&\le 1+\log\frac{N}{K_0},
&\sum_{\ell=m}^{k-1}\frac1{t_\ell^2}&\le \frac1{K_0},\\
\sum_{\ell=m}^{k-1}\frac1{\sqrt{t_\ell}}&\le 2\sqrt N,
&\sum_{\ell=m}^{k-1}\frac1{t_\ell^{3/2}}&\le \frac2{\sqrt{K_0}}
\end{align*}
therefore imply
\begin{align*}
\sum_{\ell=m}^{k-1}\frac{w_{\ell,k}}{t_\ell}
&\le \frac{c}{N^2}\left(N+m\left(1+\log\frac{N}{K_0}\right)\right),\\
\sum_{\ell=m}^{k-1}\frac{w_{\ell,k}}{t_\ell^{3/2}}
&\le \frac{c}{N^2}\left(\sqrt N+\frac{m}{\sqrt{K_0}}\right),\\
\sum_{\ell=m}^{k-1}\frac{w_{\ell,k}}{t_\ell^2}
&\le \frac{c}{N^2}\left(1+\log\frac{N}{K_0}+\frac{m}{K_0}\right).
\end{align*}
Substituting these three estimates into \eqref{eq:replacement-sum-proof} proves \eqref{eq:replacement-final}.
\end{proof}

\begin{lemma}[Initial window]\label{lem:initial}
For every $m\in\{1,\ldots,k\}$, on $\mathcal K_{k-1}$,
\begin{equation}\label{eq:initial-final}
\left|\sum_{\ell=0}^{m-1}w_{\ell,k}h(x_\ell,Z_\ell)\right|
\le
c_I\left(\frac{\sqrt{m\Lambda}}{N}+\frac{\Lambda\sqrt m}{N^{3/2}}\right),
\end{equation}
where $c_I<\infty$ depends only on $a,\mu,L,A,B,C$.
\end{lemma}

\begin{proof}
On $\mathcal K_\ell$, Lemma \ref{lem:h-envelope} gives
\[
        |h(x_\ell,Z_\ell)|
        \le H\left(\frac{\Lambda}{\ell+K_0}\right).
\]
Because $\mathcal K_{k-1}\subseteq\mathcal K_\ell$ for every $\ell\le k-1$, Cauchy-Schwarz gives, on $\mathcal K_{k-1}$,
\begin{align*}
\sum_{\ell=0}^{m-1}w_{\ell,k}|h(x_\ell,Z_\ell)|
&\le \sqrt{m\sum_{\ell=0}^{m-1}w_{\ell,k}^2H\left(\frac{\Lambda}{\ell+K_0}\right)^2}.
\end{align*}
Using $H(D)^2\le c_H(D+D^2)$ and \eqref{eq:weight-sum1}--\eqref{eq:weight-sum2},
\begin{align*}
\sum_{\ell=0}^{m-1}w_{\ell,k}^2H\left(\frac{\Lambda}{\ell+K_0}\right)^2
&\le c\sum_{\ell=0}^{k-1}w_{\ell,k}^2
\left(\frac{\Lambda}{\ell+K_0}+\frac{\Lambda^2}{(\ell+K_0)^2}\right) \\
&\le c\left(\frac{\Lambda}{N^2}+\frac{\Lambda^2}{N^3}\right).
\end{align*}
Taking square roots proves \eqref{eq:initial-final}.
\end{proof}

\section{Proof of Theorem~\ref{thm:profile-upper}: mixing-profile upper bound}\label{app:main-proof}

\begin{lemma}[One-step induction estimate]\label{lem:induction-step}
There are constants $c_0,c_1,K_D<\infty$ and a number $\rho_0\in(0,1)$, depending only on $a,\mu,L,L_g,A,B,C$, such that the following holds.  Fix $k\ge1$, let $N=k+K_0$, and choose $m\in\{1,\ldots,k\}$.  Set
\[
        q:=1+\frac{m}{K_0},
        \qquad
        u:=\log\left(\frac{16(m+1)(k+1)^2}{\delta}\right),
        \qquad
        \Theta:=muq^2.
\]
Assume $K_0\ge K_D$ and
\begin{equation}\label{eq:small-window-step}
        \frac{\Theta}{N}\le\rho_0,
        \qquad
        \frac{m}{N}\left(1+\log\frac{N}{K_0}+\frac{m}{K_0}\right)\le\rho_0.
\end{equation}
If $m<k$, also assume
\begin{equation}\label{eq:mixing-small-step}
        \varepsilon(m)\le \frac{\delta}{32N^4}.
\end{equation}
Let $\Lambda\ge1$, and suppose that the good-event family $(\mathcal K_j)_{j=0}^{k-1}$ satisfies \eqref{eq:local-envelope}.  Let
\[
        \eta_k:=\frac{\delta}{8(k+1)^2}.
\]
Then, outside an event of probability at most $2\eta_k$,
\begin{align}\label{eq:one-step-bound}
\1_{\mathcal K_{k-1}}
\Bigg(
&\left|\sum_{\ell=0}^{k-1}w_{\ell,k}\ip{\nabla f(x_\ell)}{M_{\ell+1}}\right|
+\left|\sum_{\ell=0}^{k-1}w_{\ell,k}h(x_\ell,Z_\ell)\right|
\Bigg) \\
&\le \frac{\Lambda}{4N}+\frac{c_0(1+\Theta)}{N}+\frac{c_1}{N}.\notag
\end{align}
\end{lemma}

\begin{proof}
Throughout the proof, \(c\) denotes a finite constant depending only on
\(a,\mu,L,L_g,A,B,C\); its value may change from line to line.  The goal is to prove a
single-time bound at the fixed terminal time \(k\).  The good-event indicators are included only to
make all increment bounds deterministic; on \(\mathcal K_{k-1}\) all earlier good events also hold.

\paragraph{Martingale-difference perturbation.}
The martingale part is controlled by \Cref{lem:M-noise}.  Since \(m\ge1\), \(q\ge1\), and
\(\Theta=muq^2\), we have \(\Theta\ge u\).  Also
\[
        \log\frac{2}{\eta_k}
        =\log\frac{16(k+1)^2}{\delta}
        \le
        \log\frac{16(m+1)(k+1)^2}{\delta}=u .
\]
Therefore, outside an event of probability at most \(\eta_k\),
\begin{align}\label{eq:M-close-new}
\1_{\mathcal K_{k-1}}
\left|\sum_{\ell=0}^{k-1}w_{\ell,k}\ip{\nabla f(x_\ell)}{M_{\ell+1}}\right|
&\le \frac{c}{N}\sqrt{u\left(\Lambda+\frac{\Lambda^2}{N}\right)}  \\
&\le \frac{c}{N}\sqrt{(1+\Theta)\Lambda}
   +\frac{c\Lambda}{N}\sqrt{\frac{1+\Theta}{N}} .\notag
\end{align}
For the first term we use Young's inequality in the form
\(c\sqrt{(1+\Theta)\Lambda}\le \Lambda/32+c(1+\Theta)\).  For the second term,
\[
        \sqrt{\frac{1+\Theta}{N}}
        \le \sqrt{\frac1{K_0}+\frac{\Theta}{N}}.
\]
By increasing \(K_D\) and then choosing \(\rho_0\) sufficiently small, the condition
\(K_0\ge K_D\) and \(\Theta/N\le\rho_0\) make the last display small enough that
\[
        \frac{c\Lambda}{N}\sqrt{\frac{1+\Theta}{N}}
        \le \frac{\Lambda}{32N}+\frac{c}{N}.
\]
Consequently the martingale-difference contribution satisfies
\begin{equation}\label{eq:M-final-use}
\1_{\mathcal K_{k-1}}
\left|\sum_{\ell=0}^{k-1}w_{\ell,k}\ip{\nabla f(x_\ell)}{M_{\ell+1}}\right|
\le
        \frac{\Lambda}{16N}+\frac{c(1+\Theta)}{N}+\frac{c}{N}.
\end{equation}

\paragraph{Decomposition of the Markovian sum.}
If \(m<k\), define \(Y_\ell\) and \(\overline Y_\ell\) by \eqref{eq:Y-def}.  Then
\[
\sum_{\ell=0}^{k-1}w_{\ell,k}h(x_\ell,Z_\ell)
=
\sum_{\ell=0}^{m-1}w_{\ell,k}h(x_\ell,Z_\ell)
+
\sum_{\ell=m}^{k-1}(Y_\ell-\overline Y_\ell)
+
\sum_{\ell=m}^{k-1}\overline Y_\ell
+
\sum_{\ell=m}^{k-1}w_{\ell,k}\{h(x_\ell,Z_\ell)-h(x_{\ell-m},Z_\ell)\}.
\]
When \(m=k\), the three sums over \(\ell=m,\ldots,k-1\) are empty, so only the initial-window
term remains.

\paragraph{Centered delayed term.}
Assume first that \(m<k\).  Since
\[
        \log\frac{2m}{\eta_k}
        =\log\frac{16m(k+1)^2}{\delta}
        \le u,
\]
\Cref{lem:delayed-mart} gives, outside another event of probability at most \(\eta_k\),
\begin{align}\label{eq:delayed-close-new}
\1_{\mathcal K_{k-1}}
\left|\sum_{\ell=m}^{k-1}(Y_\ell-\overline Y_\ell)\right|
&\le
\frac{c}{N}
\sqrt{muq\left(\Lambda+\frac{\Lambda^2q}{N}\right)}   \\
&\le
\frac{c}{N}\sqrt{\Theta\Lambda}
+c\frac{\Lambda}{N}\sqrt{\frac{\Theta}{N}} .\notag
\end{align}
Here we used \(q\ge1\), \(muq\le muq^2=\Theta\), and \(muq\cdot q=muq^2=\Theta\).
The first term is bounded by \(\Lambda/(64N)+c\Theta/N\) by Young's inequality.  The second
term is at most \(\Lambda/(64N)\) after decreasing \(\rho_0\), because \(\Theta/N\le\rho_0\).
Thus the centered delayed term is bounded by
\begin{equation}\label{eq:delayed-close-final-expanded}
\1_{\mathcal K_{k-1}}
\left|\sum_{\ell=m}^{k-1}(Y_\ell-\overline Y_\ell)\right|
\le
        \frac{\Lambda}{32N}+\frac{c\Theta}{N}.
\end{equation}

\paragraph{Mixing bias.}
By \Cref{lem:mixing-bias} and \eqref{eq:mixing-small-step}, on \(\mathcal K_{k-1}\),
\[
\left|\sum_{\ell=m}^{k-1}\overline Y_\ell\right|
\le
\frac{c\delta}{N^5}
\sqrt{Nq\left(\Lambda+\frac{\Lambda^2q}{N}\right)} .
\]
Since \(\delta<1\), \(q=1+m/K_0\le1+k/K_0\le N\), and \(\Lambda\ge1\), the right-hand side is
at most
\[
        \frac{c}{N^5}\{N\sqrt\Lambda+N\Lambda\}
        \le \frac{c(1+\Lambda)}{N^4}.
\]
Increasing \(K_D\) if necessary gives, for all \(N\ge K_D\),
\begin{equation}\label{eq:bias-close-new}
\1_{\mathcal K_{k-1}}
\left|\sum_{\ell=m}^{k-1}\overline Y_\ell\right|
\le
        \frac{\Lambda}{64N}+\frac{c}{N}.
\end{equation}

\paragraph{Replacement term.}
We now simplify the deterministic estimate in \Cref{lem:replacement}.  The first term in
\eqref{eq:replacement-final} is bounded by \(cm/N\le c\Theta/N\), since \(u,q\ge1\).  For the
second term, use
\[
        1+\log\frac{N}{K_0}
        \le c\log\{16(m+1)(k+1)^2/\delta\}=cu
        \le c\frac{\Theta}{m},
\]
which gives
\[
        \frac{m^2}{N^2}\left(1+\log\frac{N}{K_0}\right)
        \le c\frac{m\Theta}{N^2}
        \le c\frac{\Theta}{N}.
\]
The \(\Lambda\)-linear replacement term is controlled directly by the second small-window condition:
\[
        c\frac{m\Lambda}{N^2}
        \left(1+\log\frac{N}{K_0}+\frac{m}{K_0}\right)
        \le c\rho_0\frac{\Lambda}{N}
        \le \frac{\Lambda}{64N},
\]
after decreasing \(\rho_0\).  The term \(cm\sqrt\Lambda/N^{3/2}\) satisfies
\[
        c\frac{m\sqrt\Lambda}{N^{3/2}}
        \le \frac{\Lambda}{128N}+c\frac{m^2}{N^2}
        \le \frac{\Lambda}{128N}+c\frac{\Theta}{N}.
\]
For the final term, Young's inequality gives
\[
        c\frac{m^2\sqrt\Lambda}{\sqrt{K_0}N^2}
        \le \frac{\Lambda}{128N}+c\frac{m^4}{K_0N^3}.
\]
Because \(m\le N\) and
\[
        \frac{m^2}{K_0}
        =m\frac{m}{K_0}
        \le m\left(1+\frac{m}{K_0}\right)
        \le mq
        \le mq^2
        \le \Theta,
\]
we have \(m^4/(K_0N^3)\le \Theta/N\).  Therefore
\begin{equation}\label{eq:replacement-close-new}
\1_{\mathcal K_{k-1}}
\left|\sum_{\ell=m}^{k-1}w_{\ell,k}\{h(x_\ell,Z_\ell)-h(x_{\ell-m},Z_\ell)\}\right|
\le
        \frac{\Lambda}{32N}+\frac{c\Theta}{N}.
\end{equation}

\paragraph{Initial window.}
Finally, \Cref{lem:initial} gives
\[
\1_{\mathcal K_{k-1}}
\left|\sum_{\ell=0}^{m-1}w_{\ell,k}h(x_\ell,Z_\ell)\right|
\le
c\frac{\sqrt{m\Lambda}}{N}+c\frac{\Lambda}{N}\sqrt{\frac{m}{N}} .
\]
The first term is at most \(\Lambda/(64N)+c m/N\le \Lambda/(64N)+c\Theta/N\).  The second
term is at most \(\Lambda/(64N)\) after decreasing \(\rho_0\), because
\(m/N\le\Theta/N\le\rho_0\).  Hence the initial window is bounded by
\begin{equation}\label{eq:initial-close-expanded}
\1_{\mathcal K_{k-1}}
\left|\sum_{\ell=0}^{m-1}w_{\ell,k}h(x_\ell,Z_\ell)\right|
\le
        \frac{\Lambda}{32N}+\frac{c\Theta}{N}.
\end{equation}

Before combining the estimates, we record where the two admissibility requirements are used.  The condition
\[
        \Theta/N\le\rho_0
\]
is used to absorb the stochastic square-root terms in \eqref{eq:M-close-new},
\eqref{eq:delayed-close-new}, and \eqref{eq:initial-close-expanded} into a small multiple of
\(\Lambda/N\), with the remainder of order \((1+\Theta)/N\).  The second condition,
\[
        \frac{m}{N}\left(1+\log\frac{N}{K_0}+\frac{m}{K_0}\right)\le\rho_0,
\]
is used for the replacement term, where the movement of the iterate over one lag window produces the
additional factor involving \(m\), \(K_0\), and \(\log(N/K_0)\).  Thus the first admissibility
condition controls concentration size, while the second controls the deterministic lag-replacement
drift.

Combining \eqref{eq:delayed-close-final-expanded}, \eqref{eq:bias-close-new},
\eqref{eq:replacement-close-new}, and \eqref{eq:initial-close-expanded} controls the Markovian
sum.  If \(m=k\), the delayed centered, mixing-bias, and replacement bounds are omitted and the
same conclusion follows from \eqref{eq:initial-close-expanded}.  Adding the martingale bound
\eqref{eq:M-final-use} proves \eqref{eq:one-step-bound} after enlarging \(c_0\) and \(c_1\).
The only probabilistic failures used are the martingale-difference event and the delayed martingale
event, each of probability at most \(\eta_k\), so the total failure probability is at most
\(2\eta_k\).
\end{proof}

\begin{proof}[Proof of \Cref{thm:profile-upper}]
Let $\rho_0$ and $K_D$ be the constants from \Cref{lem:induction-step}.  Increase $K_D$ if necessary so that $K_D\ge K_{\mathrm{base}}$.  Let $\Gamma\ge1$ be chosen below and define
\begin{equation}\label{eq:lambda-def}
        \Lambda_k:=2K_0\Del_0+\Gamma(1+\Theta_k).
\end{equation}
By the monotonicity noted after \eqref{eq:u-theta}, the sequences $m_k,u_k,q_k,\Theta_k$ are nondecreasing in $k$, so $\Lambda_k$ is nondecreasing.

Define the single-time and cumulative good events by
\begin{equation}\label{eq:E-def}
        \mathsf E_k:=\left\{\Del_k\le \frac{\Lambda_k}{k+K_0}\right\},
        \qquad
        \mathsf E_0:=\Omega,
        \qquad
        \calE_k:=\bigcap_{j=0}^k\mathsf E_j.
\end{equation}
For target time $k\ge1$, on $\calE_{k-1}$ we have the envelope required in \eqref{eq:local-envelope} with $\Lambda=\Lambda_k$: for $j=0$,
\[
        \Del_0\le \Lambda_k/K_0,
\]
because $\Lambda_k\ge2K_0\Del_0$, and for $1\le j\le k-1$,
\[
        \Del_j\le \frac{\Lambda_j}{j+K_0}\le \frac{\Lambda_k}{j+K_0}.
\]
Thus \Cref{lem:induction-step} applies with $m=m_k$ and $\mathcal K_j=\calE_j$.  If $m_k<k$, then by definition of $m_k$ it equals $\widehat m_k$, and \eqref{eq:mixing-small-step} holds.  The admissibility assumptions are exactly \eqref{eq:small-window-step}.

By \Cref{lem:descent}, \Cref{lem:weights}, and \Cref{lem:induction-step}, outside an event of probability at most $2\eta_k$,
\begin{align}\label{eq:pre-close-new}
\1_{\calE_{k-1}}\Del_k
&\le
\frac{K_0\Del_0+c_{\mathrm{det}}}{N_k}
+\frac{\Lambda_k}{4N_k}
+\frac{c(1+\Theta_k)}{N_k},
\end{align}
where $c_{\mathrm{det}}=LCc_w/2$ and $c<\infty$ depends only on the model constants.  Choose
\begin{equation}\label{eq:Gamma-choice-new}
        \Gamma\ge4(c+c_{\mathrm{det}}+1).
\end{equation}
Then the numerator on the right side of \eqref{eq:pre-close-new} is bounded by
\begin{align*}
K_0\Del_0+c_{\mathrm{det}}+c(1+\Theta_k)+\frac14\{2K_0\Del_0+\Gamma(1+\Theta_k)\}
&\le \frac32K_0\Del_0+\frac12\Gamma(1+\Theta_k)\\
&\le 2K_0\Del_0+\Gamma(1+\Theta_k)=\Lambda_k.
\end{align*}
Therefore
\[
        \Pp(\calE_{k-1}\cap\mathsf E_k^c)
        \le 2\eta_k
        =\frac{\delta}{4(k+1)^2}.
\]
The first-failure union bound gives
\begin{align*}
\Pp\left(\bigcup_{k\ge1}\mathsf E_k^c\right)
&=\Pp\left(\bigcup_{k\ge1}(\calE_{k-1}\cap\mathsf E_k^c)\right) \\
&\le \sum_{k=1}^\infty\frac{\delta}{4(k+1)^2}
\le \delta.
\end{align*}
Thus $\mathsf E_k$ holds for all $k\ge1$ with probability at least $1-\delta$.  Taking $C_\star=\Gamma$ proves \eqref{eq:profile-bound}.
\end{proof}

\section{Proof of Corollary~\ref{cor:geometric-upper}: geometric window calculus}\label{app:window}

\begin{lemma}[A logarithm-over-linear bound]
\label{lem:log-linear}
Let $r\ge1$ be an integer.  Let $b,s>0$ and assume
\begin{equation}\label{eq:log-linear-condition}
        1+\log(b/s)\ge r.
\end{equation}
Then
\begin{equation}\label{eq:log-linear}
        \sup_{N\ge b}\frac{(1+\log(N/s))^r}{N}
        =\frac{(1+\log(b/s))^r}{b}.
\end{equation}
\end{lemma}

\begin{proof}
Set $y(N)=1+\log(N/s)$.  Under \eqref{eq:log-linear-condition}, $y(N)\ge r$ for every $N\ge b$.  For
\[
        \phi(N)=\frac{y(N)^r}{N}
\]
we have
\[
        \frac{d}{dN}\log\phi(N)
        =\frac{r}{Ny(N)}-\frac1N
        =\frac{1}{N}\left(\frac{r}{y(N)}-1\right)\le0.
\]
Thus $\phi$ is nonincreasing on $[b,\infty)$, and the supremum is attained at $N=b$.
\end{proof}

\begin{lemma}[Geometric windows are admissible]\label{lem:window}
Fix $\rho\in(0,1)$.  There exists $K_{\mathrm{win}}(\rho)<\infty$ such that the following holds.  Let $\delta\in(0,e^{-1})$, let $\tmix\ge1$, and assume \eqref{eq:mixing-def}.  If
\begin{equation}\label{eq:window-K-condition}
        K_0\ge K_{\mathrm{win}}(\rho)\,\tmix\left(1+\log\frac{e\tmix}{\delta}\right)^3,
\end{equation}
then the windows \eqref{eq:delay-def}-\eqref{eq:u-theta} satisfy, for every $k\ge1$,
\begin{align}
        \frac{\Theta_k}{N_k}&\le\rho,\label{eq:window-theta}\\
        \frac{m_k}{N_k}\left(1+\log\frac{N_k}{K_0}+\frac{m_k}{K_0}\right)&\le\rho.\label{eq:window-replacement}
\end{align}
\end{lemma}

\begin{proof}
Let
\[
        \ell_0:=1+\log\frac{e\tmix}{\delta},
        \qquad
        y_k:=1+\log\frac{N_k}{\delta}.
\]
Under \eqref{eq:mixing-def}, the minimal window in \eqref{eq:delay-def} satisfies
\begin{equation}\label{eq:m-upper-window-new}
        m_k\le c\tmix y_k
\end{equation}
for a numerical constant $c$.  Indeed, if
\[
        m\ge \left\lceil \tmix\left(1+\log_2(32N_k^4/\delta)\right)\right\rceil,
\]
then $\lfloor m/\tmix\rfloor\ge \log_2(32N_k^4/\delta)$, and hence $\varepsilon(m)\le\delta/(32N_k^4)$.  Also,
\begin{equation}\label{eq:u-upper-window-new}
        u_k\le c y_k.
\end{equation}
Indeed, $m_k+1\le c\tmix y_k+1\le \exp(cy_k)$ since $N_k\ge K_0\ge\tmix$ and $\delta<e^{-1}$; substituting this into the definition of $u_k$ gives \eqref{eq:u-upper-window-new}.

Combining \eqref{eq:m-upper-window-new} and \eqref{eq:u-upper-window-new},
\begin{align}\label{eq:theta-pre-window}
\frac{\Theta_k}{N_k}
&=\frac{m_ku_k}{N_k}\left(1+\frac{m_k}{K_0}\right)^2 \\
&\le c\left\{
\frac{\tmix y_k^2}{N_k}
+\frac{\tmix^2 y_k^3}{K_0N_k}
+\frac{\tmix^3 y_k^4}{K_0^2N_k}
\right\}.\notag
\end{align}
By \Cref{lem:log-linear}, each term is maximized, up to a numerical constant depending only on the logarithmic power, at $N_k=K_0$.  Let $K_0=K\tmix\ell_0^3$ with $K\ge K_{\mathrm{win}}(\rho)$.  Then
\[
        1+\log\frac{K_0}{\delta}
        \le c\{\ell_0+\log K+\log\ell_0\}
        \le c_K\ell_0,
\]
where $c_K$ depends on $K$ but not on $\tmix$ or $\delta$.  Substituting this bound into \eqref{eq:theta-pre-window} at $N_k=K_0$ yields
\[
\frac{\Theta_k}{N_k}
\le
c\left(
\frac{c_K^2}{K\ell_0}
+\frac{c_K^3}{K^2\ell_0^3}
+\frac{c_K^4}{K^3\ell_0^5}
\right).
\]
The right side tends to zero as $K\to\infty$, uniformly over $\ell_0\ge1$, because $\log K/K^a\to0$ for every $a>0$.  Choosing $K_{\mathrm{win}}(\rho)$ sufficiently large proves \eqref{eq:window-theta}.

For \eqref{eq:window-replacement}, use \eqref{eq:m-upper-window-new} and $\log(N_k/K_0)\le y_k$:
\[
\frac{m_k}{N_k}\left(1+\log\frac{N_k}{K_0}+\frac{m_k}{K_0}\right)
\le c\left\{\frac{\tmix y_k^2}{N_k}+\frac{\tmix^2y_k^2}{K_0N_k}\right\}.
\]
The same logarithm-over-linear argument bounds the right side by
\[
        c\left(\frac{c_K^2}{K\ell_0}+\frac{c_K^2}{K^2\ell_0^4}\right),
\]
which is at most $\rho$ after increasing $K_{\mathrm{win}}(\rho)$.  This proves the lemma.
\end{proof}

\begin{proof}[Proof of \Cref{cor:geometric-upper}]
Let \(\rho_0,K_D,C_\star\) be the constants from \Cref{thm:profile-upper}.  Choose
\(K_{\mathrm{geo}}\) large enough so that \eqref{eq:K-geometric} implies both
\(K_0\ge K_{\mathrm{base}}\vee K_D\) and the hypothesis of \Cref{lem:window} with
\(\rho=\rho_0\).  Then \Cref{lem:window} gives
\[
        \frac{\Theta_k}{N_k}\le\rho_0,
        \qquad
        \frac{m_k}{N_k}
        \left(1+\log\frac{N_k}{K_0}+\frac{m_k}{K_0}\right)
        \le\rho_0
\]
for every \(k\ge1\).  These are exactly the admissibility conditions
\eqref{eq:admissible-1}--\eqref{eq:admissible-2}, while the definition of \(m_k\) gives the
mixing condition needed in \Cref{thm:profile-upper}.  Hence \Cref{thm:profile-upper} applies and
proves \eqref{eq:profile-bound}.

It remains only to justify the displayed closed-form estimates.  Set
\[
        m=\left\lceil\tmix\left(1+\log_2\left(\frac{32N_k^4}{\delta}\right)\right)\right\rceil .
\]
Then
\(\lfloor m/\tmix\rfloor\ge \log_2(32N_k^4/\delta)\).  The geometric mixing assumption
\eqref{eq:mixing-def} therefore gives
\(\varepsilon(m)\le \delta/(32N_k^4)\).  Since \(m_k\) is the minimum of this admissible value and
\(k\), we have \(m_k\le m\), which proves \eqref{eq:mk-geometric-explicit}.  The quantity
\(\Theta_k=m_ku_k(1+m_k/K_0)^2\) is monotone in \(m_k\).  Replacing \(m_k\) by the upper bound
\(M_k\) in both \(u_k=\log(16(k+1)^2(m_k+1)/\delta)\) and \((1+m_k/K_0)^2\) gives
\eqref{eq:theta-geometric-explicit}.

Finally, \eqref{eq:m-upper-window-new} and \eqref{eq:u-upper-window-new} show that
\(m_k=O(\tmix\log(N_k/\delta))\) and \(u_k=O(\log(N_k/\delta))\).  Hence
\(\Theta_k=m_ku_k(1+m_k/K_0)^2\) contributes one polynomial power of \(\tmix\), with additional
logarithmic factors from \(u_k\) and \(q_k=1+m_k/K_0\).  The offset condition
\eqref{eq:K-geometric} prevents \(q_k\) from creating any further polynomial power of \(\tmix\).
Substituting this into \eqref{eq:profile-bound} gives the stated
\(\widetilde O(\tmix/(k+K_0))\) leading stochastic order.
\end{proof}

\section{Proof of Corollary~\ref{cor:local}: local PL and local oracle regularity}\label{app:local-proof}
\begin{proof}[Proof of Corollary~\ref{cor:local}]
The proof is the same first-failure induction as in \Cref{thm:profile-upper}, but we spell out why
only local assumptions are used.  Define \(\Lambda_k=2K_0\Delta_0+C_\star(1+\Theta_k)\) and the
events \(\mathsf E_k,\mathcal E_k\) as in \eqref{eq:E-def}.  Up to the first failure of these
events, the induction hypothesis gives, for every earlier index \(j\),
\[
        \Delta_j\le \frac{2K_0\Delta_0+C_\star(1+\Theta_j)}{j+K_0}.
\]
Using the admissibility condition \(\Theta_j/(j+K_0)\le\rho_0\) at the same index, and
\(K_0\ge1\), we obtain
\[
        \Delta_j
        \le 2\Delta_0+C_\star\left(\frac1{K_0}+\frac{\Theta_j}{j+K_0}\right)
        \le 2\Delta_0+C_\star(1+\rho_0)
        \le R.
\]
Thus every iterate that appears before the first failure lies in the sublevel set
\(\mathcal S_R\).  On this stopped path, all uses of smoothness, PL, the ABC envelope, and the
Lipschitz bound for the oracle are uses of their local versions on \(\mathcal S_R\).  The Markov
mixing assumption is unchanged because it concerns the exogenous chain only.  Consequently the
proofs of \Cref{lem:descent,lem:h-envelope,lem:M-noise,lem:delayed-mart,lem:mixing-bias,lem:replacement,lem:initial,lem:induction-step}
apply with the local constants.  The same summable first-failure union bound as in
\Cref{thm:profile-upper} gives the stated high-probability inequality.  Since the inequality itself
keeps \(\Delta_k\le R\), the stopped argument closes and all iterates remain in \(\mathcal S_R\) on
the resulting event.
\end{proof}
\section{Proof of Theorem~\ref{thm:linear-lower}: linear mixing-time lower bound}\label{app:lower-proof}

This appendix proves \Cref{thm:linear-lower}. The proof has four steps. 
First, Appendix~\ref{app:verify-assumptions} verifies that the two-state 
quadratic construction satisfies the assumptions of the upper-bound theorem. 
Second, Appendix~\ref{app:linear-representation} writes the SGD iterate as a 
linear filter of the Markov chain. Third, Appendix~\ref{app:moments} proves a 
variance lower bound and a fourth-moment upper bound for this filter. Finally, 
Appendix~\ref{app:complete-lower} combines these moment estimates with a 
constant-probability argument and translates \(\varepsilon^{-1}\) into the 
mixing time.

\subsection{Verification of the assumptions}\label{app:verify-assumptions}

\begin{lemma}[The lower-bound instance satisfies the standard assumptions]\label{lem:lower-instance-assumptions}
The instance \eqref{eq:two-state-kernel}-\eqref{eq:lower-instance} satisfies Assumptions~\ref{ass:objective}--\ref{ass:mixing}. Specifically, $f$ is $1$-smooth and $1$-PL, $g(\cdot,z)$ is $1$-Lipschitz for both states, the stationary gradient identity holds, and
\begin{equation}\label{eq:lower-abc-bound}
      \abs{g(x,z)}^2\le 2\abs{x}^2+2\sigma^2
      =2\norm{\nabla f(x)}^2+2\sigma^2.
\end{equation}
Thus Assumption~\ref{ass:abc} holds with $A=2$, $B=0$, and $C=2\sigma^2$.
\end{lemma}

\begin{proof}
For $f(x)=x^2/2$, we have $\nabla f(x)=x$, so
\[
      \abs{\nabla f(x)-\nabla f(y)}=\abs{x-y}
\]
and $f$ is $1$-smooth. Also
\[
      \norm{\nabla f(x)}^2=x^2=2(f(x)-f^\star),
\]
so $f$ is $1$-PL. The invariant law of \eqref{eq:two-state-kernel} is uniform on $\{-1,+1\}$, hence
\[
      \int g(x,z)\,\pi(dz)
      =\frac{x+\sigma}{2}+\frac{x-\sigma}{2}=x=\nabla f(x).
\]
For each fixed $z$, $g(x,z)=x+\sigma z$ is $1$-Lipschitz in $x$. Finally,
\[
      \abs{x+\sigma z}^2\le 2x^2+2\sigma^2.
\]
for $z\in\{-1,+1\}$, proving \eqref{eq:lower-abc-bound}. The mixing assertion is proved in \cref{lem:two-state-mixing} below.
\end{proof}

\begin{lemma}[Mixing time of the two-state chain]\label{lem:two-state-mixing}
Let $\rho=1-2\varepsilon$ with $\varepsilon\in(0,1/8]$. For the chain \eqref{eq:two-state-kernel},
\begin{equation}\label{eq:two-state-tv}
      \sup_{z\in\{-1,+1\}}\norm{P^t(\cdot\mid z)-\pi}_{\tv}=\frac{\rho^t}{2},
      \qquad t\ge0.
\end{equation}
Consequently,
\[
      \tau_\rho=\left\lceil\frac{\log 2}{-\log\rho}\right\rceil
\]
is a valid mixing time in the sense of \eqref{eq:mixing-def}. Moreover,
\begin{equation}\label{eq:tau-rho-bounds}
      \frac{\log 2}{4\varepsilon}\le \tau_\rho\le \frac{\log 2}{2\varepsilon}+1.
\end{equation}
\end{lemma}

\begin{proof}
The transition matrix has eigenvalues $1$ and $\rho=1-2\varepsilon$. Starting from $+1$,
\[
      \Prob(Z_t=+1\mid Z_0=+1)=\frac{1+\rho^t}{2},
      \qquad
      \Prob(Z_t=-1\mid Z_0=+1)=\frac{1-\rho^t}{2}.
\]
The invariant law is $\pi(+1)=\pi(-1)=1/2$. Therefore the signed measure
$P^t(\cdot\mid +1)-\pi$ has masses $\rho^t/2$ and $-\rho^t/2$ on $+1$ and $-1$. Under the probability total-variation convention used in this paper, its norm is $\rho^t/2$. The same calculation holds from $-1$, proving \eqref{eq:two-state-tv}.

By the definition of $\tau_\rho$, $\rho^{\tau_\rho}\le 1/2$. If $t\ge0$ and $q=\lfloor t/\tau_\rho\rfloor$, then $t\ge q\tau_\rho$ and hence
\[
      \frac{\rho^t}{2}\le \rho^{q\tau_\rho}\le 2^{-q}=2^{-\lfloor t/\tau_\rho\rfloor}.
\]
Thus $\tau_\rho$ is a valid mixing time.

It remains to prove \eqref{eq:tau-rho-bounds}. For $u\in(0,1/2]$,
\[
      u\le -\log(1-u)\le 2u.
\]
With $u=2\varepsilon\in(0,1/4]$, this gives
\[
      2\varepsilon\le -\log(1-2\varepsilon)\le 4\varepsilon.
\]
Thus
\[
      \frac{\log 2}{4\varepsilon}\le \frac{\log 2}{-\log\rho}\le \frac{\log 2}{2\varepsilon}.
\]
Since $\tau_\rho=\lceil \log 2/(-\log\rho)\rceil$, the stated bounds follow.
\end{proof}

\subsection{Linear representation of the SGD iterate}\label{app:linear-representation}

Define
\begin{equation}\label{eq:b-ik-def}
      b_{i,k}=\alpha_i\prod_{j=i+1}^{k-1}(1-\alpha_j),
      \qquad
      \alpha_j=\frac{a}{j+K},
      \qquad
      0\le i\le k-1,
\end{equation}
with the convention that an empty product equals one.

\begin{lemma}[Linear filter representation]\label{lem:linear-filter}
For the recursion \eqref{eq:lower-sgd-recursion},
\begin{equation}\label{eq:xk-linear-filter}
      x_k=-\sigma\sum_{i=0}^{k-1}b_{i,k}Z_i.
\end{equation}
\end{lemma}

\begin{proof}
The proof is by induction. For $k=1$,
\[
      x_1=\left(1-\frac{a}{K}\right)x_0-\frac{a\sigma}{K}Z_0
      =-\sigma\alpha_0Z_0,
\]
because $x_0=0$, which agrees with \eqref{eq:xk-linear-filter}. Assume \eqref{eq:xk-linear-filter} holds at time $k$. Then
\begin{align*}
      x_{k+1}
      &= (1-\alpha_k)x_k-\alpha_k\sigma Z_k \\
      &= -\sigma\sum_{i=0}^{k-1}\alpha_i
         \left\{\prod_{j=i+1}^{k-1}(1-\alpha_j)\right\}(1-\alpha_k)Z_i
         -\sigma\alpha_k Z_k \\
      &= -\sigma\sum_{i=0}^{k}\alpha_i
         \left\{\prod_{j=i+1}^{k}(1-\alpha_j)\right\}Z_i.
\end{align*}
This is \eqref{eq:xk-linear-filter} with $k$ replaced by $k+1$.
\end{proof}

\begin{lemma}[Lower bound on recent weights]\label{lem:recent-weight-lower}
Assume $a\ge2$, $K\ge\max\{4a,4a^2\}$, and $k\ge\max\{8K,16a^2\}$. Let
\[
      I_k=\{\lceil k/2\rceil,\lceil k/2\rceil+1,\ldots,k-1\}.
\]
Then for every $i\in I_k$,
\begin{equation}\label{eq:recent-weight-lower}
      b_{i,k}\ge \frac{\beta_a}{k},
      \qquad
      \beta_a=\frac{2a e^{-1}}{3^{a+1}}.
\end{equation}
\end{lemma}

\begin{proof}
Since $K\ge4a$, for every $j\ge0$,
\[
      0\le \alpha_j=\frac{a}{j+K}\le \frac{a}{K}\le \frac14.
\]
For $0\le u\le1/2$, $\log(1-u)\ge -u-u^2$. Therefore, for $i\le k-1$,
\begin{align}\label{eq:product-log-lower}
      \log\prod_{j=i+1}^{k-1}(1-\alpha_j)
      &\ge -\sum_{j=i+1}^{k-1}\frac{a}{j+K}
           -\sum_{j=i+1}^{k-1}\frac{a^2}{(j+K)^2}.
\end{align}
The first sum is bounded as
\begin{equation}\label{eq:harmonic-bound}
      \sum_{j=i+1}^{k-1}\frac{1}{j+K}
      \le \int_{i+K}^{k+K}\frac{du}{u}
      =\log\frac{k+K}{i+K}.
\end{equation}
For the second sum, using $i\in I_k$ and $k\ge16a^2$,
\begin{equation}\label{eq:square-sum-bound}
      \sum_{j=i+1}^{k-1}\frac{a^2}{(j+K)^2}
      \le \int_{i+K}^{\infty}\frac{a^2\,du}{u^2}
      =\frac{a^2}{i+K}
      \le \frac{2a^2}{k}
      \le \frac18
      \le 1.
\end{equation}
Combining \eqref{eq:product-log-lower}-\eqref{eq:square-sum-bound} gives
\begin{equation}\label{eq:product-lower}
      \prod_{j=i+1}^{k-1}(1-\alpha_j)
      \ge e^{-1}\left(\frac{i+K}{k+K}\right)^a.
\end{equation}
Since $i\in I_k$ and $k\ge8K$,
\[
      i+K\ge \frac{k}{2},
      \qquad
      k+K\le \frac{9k}{8}<\frac{3k}{2},
      \qquad
      \frac{i+K}{k+K}\ge \frac13.
\]
Also $i+K\le k+K\le 3k/2$, so
\[
      \alpha_i=\frac{a}{i+K}\ge \frac{2a}{3k}.
\]
Using this and \eqref{eq:product-lower} in the definition of $b_{i,k}$ yields
\[
      b_{i,k}\ge \frac{2a}{3k}e^{-1}3^{-a}
      =\frac{2ae^{-1}}{3^{a+1}}\frac{1}{k}.
\]
This proves \eqref{eq:recent-weight-lower}.
\end{proof}

\subsection{Second and fourth moments}\label{app:moments}

Let
\begin{equation}\label{eq:Sk-def}
      S_k=\sum_{i=0}^{k-1}b_{i,k}Z_i,
\end{equation}
so that $x_k=-\sigma S_k$ by \cref{lem:linear-filter}. Since the chain starts stationarily, $\E Z_i=0$ and
\begin{equation}\label{eq:covariance-two-state}
      \E[Z_iZ_j]=\rho^{\abs{i-j}},
      \qquad i,j\ge0.
\end{equation}

\begin{lemma}[Variance lower bound]\label{lem:variance-lower}
Under the conditions of \cref{lem:recent-weight-lower}, if in addition $k\ge16/\varepsilon$, then
\begin{equation}\label{eq:variance-lower}
      \operatorname{Var}(S_k)\ge \frac{\beta_a^2}{128}\frac{1}{\varepsilon k}.
\end{equation}
\end{lemma}

\begin{proof}
All weights $b_{i,k}$ are nonnegative because $\alpha_j\le1/4$. From \eqref{eq:covariance-two-state},
\begin{equation}\label{eq:var-full-sum}
      \operatorname{Var}(S_k)=\sum_{i=0}^{k-1}\sum_{j=0}^{k-1}b_{i,k}b_{j,k}\rho^{\abs{i-j}}.
\end{equation}
Restrict the double sum to $I_k\times I_k$. By \cref{lem:recent-weight-lower},
\begin{equation}\label{eq:var-restrict}
      \operatorname{Var}(S_k)
      \ge \frac{\beta_a^2}{k^2}
          \sum_{i\in I_k}\sum_{j\in I_k}\rho^{\abs{i-j}}.
\end{equation}
Let $n=\abs{I_k}$. Since $I_k$ contains at least $k/2$ indices, $n\ge k/2$. Let
\[
      h=\left\lfloor\frac{1}{8\varepsilon}\right\rfloor.
\]
Because $\varepsilon\le1/8$, $h\ge1/(16\varepsilon)$. The condition $k\ge16/\varepsilon$ implies $n\ge k/2\ge8/\varepsilon\ge 2h+1$. For every integer $0\le r\le h$,
\[
      \rho^r=(1-2\varepsilon)^r\ge 1-2\varepsilon r\ge 1-2\varepsilon h\ge \frac34,
\]
where we used Bernoulli's inequality $(1-u)^r\ge1-ur$ for $u\in[0,1]$. For at least $n-h\ge n/2$ choices of $i\in I_k$, the indices $i,i+1,\ldots,i+h$ all belong to $I_k$. Hence
\begin{align}\label{eq:cov-sum-lower}
      \sum_{i\in I_k}\sum_{j\in I_k}\rho^{\abs{i-j}}
      &\ge \sum_{\substack{i\in I_k:\ i+h\le k-1}}\sum_{r=0}^{h}\rho^r \\
      &\ge \frac{n}{2}(h+1)\frac34
       \ge \frac{3}{8}\cdot\frac{k}{2}\cdot\frac{1}{16\varepsilon}
       \ge \frac{k}{128\varepsilon}.
\end{align}
Combining \eqref{eq:var-restrict} and \eqref{eq:cov-sum-lower} proves \eqref{eq:variance-lower}.
\end{proof}

\begin{lemma}[Fourth moment upper bound]\label{lem:fourth-moment}
For every deterministic nonnegative sequence $b_0,\ldots,b_{k-1}$ and the stationary two-state chain \eqref{eq:two-state-kernel},
\begin{equation}\label{eq:fourth-bound}
      \E\left(\sum_{i=0}^{k-1}b_iZ_i\right)^4
      \le 24\left\{\operatorname{Var}\left(\sum_{i=0}^{k-1}b_iZ_i\right)\right\}^2.
\end{equation}
\end{lemma}

\begin{proof}
Let $S=\sum_{i=0}^{k-1}b_iZ_i$. We first compute mixed fourth moments. The two-state stationary Markov chain can be represented as
\begin{equation}\label{eq:innovation-representation}
      Z_i=Z_0\prod_{r=1}^{i}Y_r,
\end{equation}
where $Z_0$ is a Rademacher random variable, $Y_1,Y_2,\ldots$ are independent of $Z_0$, and
\[
      \Prob(Y_r=1)=1-\varepsilon,
      \qquad
      \Prob(Y_r=-1)=\varepsilon.
\]
Then $\E Y_r=1-2\varepsilon=\rho$. If $i_1\le i_2\le i_3\le i_4$, the product $Z_{i_1}Z_{i_2}Z_{i_3}Z_{i_4}$ equals
\[
      \left(\prod_{r=i_1+1}^{i_2}Y_r\right)
      \left(\prod_{r=i_3+1}^{i_4}Y_r\right),
\]
because all other factors in \eqref{eq:innovation-representation} appear an even number of times and cancel. Therefore
\begin{equation}\label{eq:mixed-fourth}
      \E[Z_{i_1}Z_{i_2}Z_{i_3}Z_{i_4}]
      =\rho^{i_2-i_1+i_4-i_3}.
\end{equation}
Since the coefficients $b_i$ are nonnegative, expanding $S^4$ and grouping ordered quadruples by their nondecreasing rearrangement gives
\begin{align}\label{eq:S4-ordering}
      \E S^4
      &\le 24\sum_{0\le i\le j\le \ell\le m\le k-1}
          b_i b_j b_\ell b_m\rho^{j-i+m-\ell}.
\end{align}
The factor $24$ is an upper bound on the number of permutations of four indices; if some indices coincide, this only overcounts and preserves the inequality.

Define
\begin{equation}\label{eq:Q-def}
      Q=\sum_{0\le i\le j\le k-1}b_i b_j\rho^{j-i}.
\end{equation}
The ordered quadruple sum in \eqref{eq:S4-ordering} is bounded by $Q^2$, because $Q^2$ sums $b_i b_j b_\ell b_m\rho^{j-i+m-\ell}$ over all pairs $(i,j)$ and $(\ell,m)$ with $i\le j$ and $\ell\le m$, while the ordered quadruple sum restricts to the subset satisfying $j\le\ell$. Hence
\begin{equation}\label{eq:S4-Q}
      \E S^4\le24Q^2.
\end{equation}
On the other hand,
\begin{align}\label{eq:var-Q}
      \operatorname{Var}(S)
      &=\sum_{i=0}^{k-1}\sum_{j=0}^{k-1}b_i b_j\rho^{\abs{i-j}} \\
      &=\sum_{i=0}^{k-1}b_i^2+2\sum_{0\le i<j\le k-1}b_i b_j\rho^{j-i} \\
      &=Q+\sum_{0\le i<j\le k-1}b_i b_j\rho^{j-i}
      \ge Q.
\end{align}
Combining \eqref{eq:S4-Q} and \eqref{eq:var-Q} proves \eqref{eq:fourth-bound}.
\end{proof}

\subsection{Completion of the lower-bound proof}\label{app:complete-lower}

\begin{proof}[Proof of \Cref{thm:linear-lower}]
The assumptions are verified in \cref{lem:lower-instance-assumptions}, and the mixing-time claim is
\cref{lem:two-state-mixing}.  It remains to prove the expectation and probability lower bounds.

By \cref{lem:linear-filter}, \(x_k=-\sigma S_k\).  Since \(f(x)=x^2/2\) and \(f^\star=0\),
\begin{equation}\label{eq:Delta-S}
      f(x_k)-f^\star=\frac{\sigma^2S_k^2}{2}.
\end{equation}
The chain starts in stationarity, so \(\E S_k=0\) and \(\E S_k^2=\operatorname{Var}(S_k)\).
Using \cref{lem:variance-lower},
\begin{align}\label{eq:expectation-proof}
      \E[f(x_k)-f^\star]
      &=\frac{\sigma^2}{2}\operatorname{Var}(S_k)
      \ge \frac{\sigma^2\beta_a^2}{256}\frac{1}{\varepsilon k}.
\end{align}

For the probability statement, set \(Y=S_k^2\).  The fourth-moment estimate in
\cref{lem:fourth-moment} gives
\begin{equation}\label{eq:Y-second}
      \E Y^2=\E S_k^4\le 24\{\E S_k^2\}^2.
\end{equation}
Applying Paley--Zygmund to the nonnegative random variable \(Y\), with threshold one half of its
mean, yields
\begin{equation}\label{eq:paley-step}
      \Prob\left(Y\ge\frac12\E Y\right)
      \ge (1-1/2)^2\frac{(\E Y)^2}{\E Y^2}
      \ge \frac{1}{96}.
\end{equation}
On this event, \eqref{eq:Delta-S} and \cref{lem:variance-lower} imply
\begin{equation}\label{eq:prob-proof-threshold}
      f(x_k)-f^\star
      =\frac{\sigma^2Y}{2}
      \ge \frac{\sigma^2}{4}\operatorname{Var}(S_k)
      \ge \frac{\sigma^2\beta_a^2}{512}\frac{1}{\varepsilon k}.
\end{equation}
Since
\[
      \beta_a^2=\frac{4a^2e^{-2}}{3^{2a+2}},
\]
both \eqref{eq:expectation-proof} and \eqref{eq:prob-proof-threshold} hold with the common
constant
\[
      c_a=\frac{a^2e^{-2}}{2^{14}3^{2a+2}},
\]
which is smaller than both \(\beta_a^2/256\) and \(\beta_a^2/512\).  The probability constant is
\(c_0=1/96\), proving \eqref{eq:expectation-lower-main} and
\eqref{eq:probability-lower-main}.

Finally, \cref{lem:two-state-mixing} gives a valid geometric mixing time \(\tau_\rho\) satisfying
\(\tau_\rho\le c/\varepsilon\) and \(\tau_\rho\ge c'/\varepsilon\) for universal constants
\(c,c'>0\).  Equivalently, \(1/\varepsilon\) is bounded below by a universal constant times this
valid mixing time.  Replacing \(1/\varepsilon\) in the preceding two lower bounds by that constant
multiple of \(\tau_\rho\) gives \eqref{eq:tmix-lower-main}, with \(c'_a\) adjusted by the same
universal factor.
\end{proof}

\section{Proofs for the heavy-tailed extension}\label{app:heavy-tail}

This appendix proves \Cref{thm:heavy-upper,cor:heavy-geometric,thm:heavy-lower}. 
Appendix~\ref{app:heavy-aux} proves the clipped-gradient concentration and 
deterministic inexact-gradient tools. Appendix~\ref{app:heavy-upper-proof} 
proves the robust upper bound. Appendix~\ref{app:heavy-geometric-proof} derives 
the transition-budget corollary under geometric mixing. Appendix~\ref{app:heavy-lower-proof} 
proves the matching heavy-tailed sticky-chain lower bound.

Throughout this appendix, $c,c_p,c_{\lambda,p}$ denote positive constants depending only on $p$, unless another dependence is explicitly stated.
\subsection{Auxiliary estimates for clipped Markovian blocks}
\label{app:heavy-aux}
\begin{lemma}[Clipped stationary bias and variance]\label{lem:clipped-stationary-bias-var}
Fix $x\in\R^d$ and suppose Assumption~\ref{ass:heavy-tail} holds.  Let $g:=\nabla f(x)$, $v:=\sigma_pR_p(x)$, and assume $\lambda\ge2\|g\|_\infty$.  For a coordinate $j$, define
\[
        \psi_j(z,u):=[\mathsf T_\lambda({\mathcal G}(x,z,u))]_j,
        \qquad
        \bar\psi_j:=\int \E_U[\psi_j(z,U)]\,\pi(dz).
\]
Then
\begin{align}
        |\bar\psi_j-g_j|&\le c_p v^p\lambda^{1-p},\label{eq:clip-bias-coordinate}\\
        \int \E_U\big[(\psi_j(z,U)-\bar\psi_j)^2\big] \pi(dz)&\le c_p v^p\lambda^{2-p}.
        \label{eq:clip-var-coordinate}
\end{align}
\end{lemma}

\begin{proof}
Let \((Z,U)\) have the stationary law \(\pi(dz)\) together with the auxiliary randomness of the
oracle, and write
\[
        \eta={\mathcal G}(x,Z,U)-g,
        \qquad
        \eta_j=[\eta]_j .
\]
The assumption gives \(\E_\pi\|\eta\|^p\le v^p\), hence
\(\E_\pi|\eta_j|^p\le v^p\).  Since \(\lambda\ge2\|g\|_\infty\), if
\(|\eta_j|<\lambda/2\), then \(|g_j+\eta_j|<\lambda\) and clipping does not change the
\(j\)th coordinate.  Therefore clipping bias can occur only on the event
\(\{|\eta_j|\ge\lambda/2\}\).  Using Markov's inequality and the identity
\(\E\{|\eta_j|\1_{|\eta_j|\ge a}\}\le a^{1-p}\E|\eta_j|^p\),
\begin{align*}
        |\bar\psi_j-g_j|
        &=\left|\E_\pi\{[\mathsf T_\lambda(g+\eta)]_j-(g_j+\eta_j)\}\right| \\
        &\le \E_\pi\left[|g_j+\eta_j|\1_{\{|\eta_j|\ge\lambda/2\}}\right] \\
        &\le |g_j|\Pp_\pi(|\eta_j|\ge\lambda/2)
             +\E_\pi\left[|\eta_j|\1_{\{|\eta_j|\ge\lambda/2\}}\right] \\
        &\le |g_j|\frac{\E_\pi|\eta_j|^p}{(\lambda/2)^p}
             +\frac{\E_\pi|\eta_j|^p}{(\lambda/2)^{p-1}}
        \le c_p v^p\lambda^{1-p},
\end{align*}
where the first term is also of order \(v^p\lambda^{1-p}\) because \(|g_j|\le\lambda/2\).
This proves \eqref{eq:clip-bias-coordinate}.

For the second moment, split again according to \(|\eta_j|\le\lambda/2\).  On this event the
clipped coordinate differs from \(g_j\) by exactly \(\eta_j\).  On the complement, both
\([\mathsf T_\lambda(g+\eta)]_j\) and \(g_j\) have magnitude at most \(\lambda\) and
\(\lambda/2\), respectively, so their difference is bounded by a numerical multiple of
\(\lambda\).  Hence
\begin{align*}
        \E_\pi\big([\mathsf T_\lambda(g+\eta)]_j-g_j\big)^2
        &\le
        \E_\pi\big[|\eta_j|^2\1_{\{|\eta_j|\le\lambda/2\}}\big]
        +c\lambda^2\Pp_\pi(|\eta_j|>\lambda/2) \\
        &\le
        (\lambda/2)^{2-p}\E_\pi|\eta_j|^p
        +c\lambda^2\frac{\E_\pi|\eta_j|^p}{(\lambda/2)^p}
        \le c_p v^p\lambda^{2-p}.
\end{align*}
Finally, variance around \(\bar\psi_j\) is no larger than the second moment around any fixed
center, so
\[
        \E_\pi[(\psi_j-\bar\psi_j)^2]
        \le \E_\pi[(\psi_j-g_j)^2]
        \le c_pv^p\lambda^{2-p}.
\]
This proves \eqref{eq:clip-var-coordinate}.
\end{proof}

\begin{lemma}[Consecutive clipped-block concentration]\label{lem:consecutive-block-concentration}
Fix a deterministic point $x$ satisfying $\Delta(x)\le\mathcal R$, a block length $b\ge1$, and an integer $m\in\{1,\ldots,b\}$.  Let the Markov chain start from an arbitrary state and let
\[
        Y_i={\mathcal G}(x,Z_i,U_i),\qquad i=1,\ldots,b,
\]
where the auxiliary variables $U_i$ are conditionally independent given the Markov trajectory.  Define
\[
        u:=\log\left(\frac{16d(m+1)}{\eta}\right),
        \qquad
        s:=\frac{mu}{b},
\]
for some $\eta\in(0,1)$.  Suppose $\varepsilon(m)\le \eta/(8db)$, set
\[
        \lambda=c_{\lambda,p}\sigma_pB_{\mathcal R}s^{-1/p},
\]
with $c_{\lambda,p}$ sufficiently large, and assume $\lambda\ge2\sqrt{2L\mathcal R}$.  Then, with probability at least $1-\eta$,
\begin{equation}\label{eq:consecutive-block-gradient-error}
        \left\|\frac1b\sum_{i=1}^b\mathsf T_\lambda(Y_i)-\nabla f(x)\right\|
        \le c_p\sqrt d\,\sigma_pB_{\mathcal R}s^{\frac{p-1}{p}}.
\end{equation}
The same conclusion holds conditionally on any sigma-field \(\mathcal A\) with respect to which \(x\) and the initial Markov state are measurable, provided that after conditioning on \(\mathcal A\) the future chain evolves with transition kernel \(P\) from that initial state and the auxiliary variables \((U_i)\) are conditionally independent given the future Markov trajectory, with the same conditional oracle kernel as in \Cref{ass:heavy-tail}.
\end{lemma}

\begin{proof}
We first justify the conditional version.  After conditioning on such a sigma-field \(\mathcal A\), the query point \(x\) and the starting state are deterministic, the future Markov process still has transition kernel \(P\), and the auxiliary variables remain conditionally independent with the prescribed conditional oracle law.  Hence all conditional expectations, mixing-bias estimates, and martingale-difference properties below hold under the regular conditional law.  It is therefore enough to prove the result for a fixed deterministic starting state.  Since \(\Delta(x)\le\mathcal R\),
\Cref{lem:grad-upper} gives \(\|\nabla f(x)\|\le\sqrt{2L\mathcal R}\).  Thus
\(R_p(x)\le B_{\mathcal R}\).  Put
\[
        g:=\nabla f(x),
        \qquad
        v:=\sigma_pB_{\mathcal R}.
\]
Then \Cref{ass:heavy-tail} gives the stationary moment bound needed in
\Cref{lem:clipped-stationary-bias-var}, and the condition
\(\lambda\ge2\sqrt{2L\mathcal R}\) implies \(\lambda\ge2\|g\|_\infty\).

Fix a coordinate \(j\).  Let \(\psi_j\) and \(\bar\psi_j\) be the clipped coordinate and its
stationary mean from \Cref{lem:clipped-stationary-bias-var}.  We first control the centered block
average
\[
        S_j:=\frac1b\sum_{i=1}^b\{\psi_j(Z_i,U_i)-\bar\psi_j\}.
\]
The first \(m\) samples are not delayed by a full mixing window.  Since
\(|\psi_j|\le\lambda\) and \(|\bar\psi_j|\le\lambda\), their total contribution is bounded
deterministically by
\[
        \left|\frac1b\sum_{i=1}^{m}\{\psi_j(Z_i,U_i)-\bar\psi_j\}\right|
        \le \frac{2\lambda m}{b}.
\]

For \(i>m\), define
\[
        D_i:=\psi_j(Z_i,U_i)-\bar\psi_j
        -\E[\psi_j(Z_i,U_i)-\bar\psi_j\mid\mathcal F_{i-m}],
\]
where \(\mathcal F_i\) contains the chain and auxiliary variables up to time \(i\).  The conditional
mean term is small because, given \(\mathcal F_{i-m}\), the law of \(Z_i\) is
\(P^m(Z_{i-m},\cdot)\), while the auxiliary variable is sampled from the same conditional kernel as
in the stationary definition.  Since \(\int\E_U\psi_j(z,U)\pi(dz)=\bar\psi_j\), the total-variation
inequality \eqref{eq:tv-dual} and \(|\psi_j|\le\lambda\) give
\begin{equation}\label{eq:conditional-mean-heavy-block}
        \left|\E[\psi_j(Z_i,U_i)-\bar\psi_j\mid\mathcal F_{i-m}]\right|
        \le 2\lambda\varepsilon(m).
\end{equation}

The variables \((D_i)_{i>m}\) are not martingale differences in the ordinary time order.  Split the
indices \(\{m+1,\ldots,b\}\) into residue classes modulo \(m\).  If
\(i_1<i_2<\cdots<i_s\) are the indices in one residue class, then \(i_{r+1}-m=i_r\).  Hence
\[
        \E[D_{i_{r+1}}\mid\mathcal F_{i_r}]
        =\E[D_{i_{r+1}}\mid\mathcal F_{i_{r+1}-m}]=0,
\]
so this subsequence is a martingale-difference sequence with respect to the filtration
\((\mathcal F_{i_r})_{r=1}^s\).  Moreover, \(|D_i|\le4\lambda\), because both
\(\psi_j-\bar\psi_j\) and its conditional mean are bounded by \(2\lambda\).

We also need a conditional variance bound.  Since subtracting a conditional mean can only decrease
conditional second moment,
\[
        \E[D_i^2\mid\mathcal F_{i-m}]
        \le
        \E[(\psi_j(Z_i,U_i)-\bar\psi_j)^2\mid\mathcal F_{i-m}].
\]
The function \((\psi_j-\bar\psi_j)^2\) is bounded by \(4\lambda^2\).  Combining the stationary
variance bound in \Cref{lem:clipped-stationary-bias-var} with \eqref{eq:tv-dual} gives
\begin{equation}\label{eq:cond-var-heavy-block}
        \E[D_i^2\mid\mathcal F_{i-m}]
        \le c_pv^p\lambda^{2-p}+8\lambda^2\varepsilon(m).
\end{equation}

Apply Freedman's inequality separately on each residue class, using range bound \(4\lambda\) and the
variance bound \eqref{eq:cond-var-heavy-block}.  Let \(I_r\subseteq\{m+1,\ldots,b\}\) denote the
indices in residue class \(r\).  For a fixed residue class, Freedman's inequality gives, with
probability at least \(1-2e^{-u}\),
\[
        \left|\sum_{i\in I_r}D_i\right|
        \le
        c\sqrt{u |I_r|\left(v^p\lambda^{2-p}+\lambda^2\varepsilon(m)\right)}
        +c\lambda u .
\]
Taking a union bound over the at most \(m\) residue classes, this holds for every residue class
simultaneously with probability at least \(1-2me^{-u}\).  On this event,
\[
\begin{aligned}
        \left|\sum_{i=m+1}^bD_i\right|
        &\le \sum_{r=0}^{m-1}\left|\sum_{i\in I_r}D_i\right| \\
        &\le
        c\sum_{r=0}^{m-1}
        \sqrt{u |I_r|\left(v^p\lambda^{2-p}+\lambda^2\varepsilon(m)\right)}
        +c\lambda m u .
\end{aligned}
\]
By Cauchy--Schwarz,
\[
\begin{aligned}
        \sum_{r=0}^{m-1}
        \sqrt{|I_r|\left(v^p\lambda^{2-p}+\lambda^2\varepsilon(m)\right)}
        &\le
        \sqrt{m\sum_{r=0}^{m-1}|I_r|\left(v^p\lambda^{2-p}+\lambda^2\varepsilon(m)\right)} \\
        &\le
        \sqrt{mb\left(v^p\lambda^{2-p}+\lambda^2\varepsilon(m)\right)} .
\end{aligned}
\]
Therefore
\begin{align}\label{eq:residue-heavy-bound}
        \left|\sum_{i=m+1}^b D_i\right|
        &\le c\sqrt{mu\,b\left(v^p\lambda^{2-p}+\lambda^2\varepsilon(m)\right)}
        +c\lambda m u .
\end{align}
The factor \(m\) under the square root is the price of the residue-class union and summation, not a
loss of samples: the total number of summands remains \(b-m\).

Combining the initial-window bound, the accumulated conditional-mean bound from
\eqref{eq:conditional-mean-heavy-block}, and \eqref{eq:residue-heavy-bound}, and then dividing by
\(b\), yields for the fixed coordinate
\begin{align*}
        |S_j|
        &\le c\left[
        \sqrt{\frac{mu}{b}v^p\lambda^{2-p}}
        +\lambda\sqrt{\frac{mu}{b}\varepsilon(m)}
        +\frac{\lambda m u}{b}
        +\frac{\lambda m}{b}
        +\lambda\varepsilon(m)\right]
\end{align*}
with probability at least \(1-2me^{-u}\).  Since \(u\ge1\) and
\(\varepsilon(m)\le\eta/(8db)\le1/b\le mu/b\), the last four terms are bounded by a constant
multiple of \(\lambda mu/b=\lambda s\).  Therefore
\[
        |S_j|
        \le c\left[\sqrt{s v^p\lambda^{2-p}}+\lambda s\right].
\]
Adding the stationary clipping bias from \Cref{lem:clipped-stationary-bias-var} gives
\[
        \left|\frac1b\sum_{i=1}^b[\mathsf T_\lambda(Y_i)]_j-g_j\right|
        \le c_p\left[
        \sqrt{s v^p\lambda^{2-p}}+
        \lambda s+v^p\lambda^{1-p}\right].
\]
With \(\lambda=c_{\lambda,p}v s^{-1/p}\), each of the three terms on the right is of order
\(v s^{(p-1)/p}\), after choosing \(c_{\lambda,p}\) sufficiently large.  Finally,
\(2dme^{-u}\le\eta\) by the definition of \(u\), so a union bound over the \(d\) coordinates gives
coordinatewise control simultaneously.  Taking the Euclidean norm multiplies the coordinate bound
by at most \(\sqrt d\), proving \eqref{eq:consecutive-block-gradient-error}.
\end{proof}

\begin{lemma}[PL descent with deterministic gradient errors]\label{lem:deterministic-inexact-pl}
Assume that $f$ is $L$-smooth and $\mu$-PL.  Let \(J\ge1\) be an integer and let
\[
        x_{r+1}=x_r-\frac1{4L}(\nabla f(x_r)+e_r)
\]
for \(0\le r<J\).  Suppose that $\|e_r\|\le\epsilon$ for all $0\le r<J$.  Then
\begin{equation}\label{eq:det-inexact-rate}
        \Delta(x_J)
        \le \left(1-\frac{\mu}{16L}\right)^J\Delta(x_0)
        +\frac{8\epsilon^2}{\mu}.
\end{equation}
\end{lemma}

\begin{proof}
Fix one step and write \(g_r:=\nabla f(x_r)\).  By \(L\)-smoothness and the update rule,
\begin{align*}
\Delta(x_{r+1})
&\le \Delta(x_r)-\frac1{4L}\langle g_r,g_r+e_r\rangle
        +\frac{L}{2}\left\|\frac1{4L}(g_r+e_r)\right\|^2 \\
&= \Delta(x_r)-\frac1{4L}\|g_r\|^2
        -\frac1{4L}\langle g_r,e_r\rangle
        +\frac1{32L}\|g_r+e_r\|^2 .
\end{align*}
Using \(\|g_r+e_r\|^2\le2\|g_r\|^2+2\|e_r\|^2\) and
\(-\langle g_r,e_r\rangle\le \|g_r\|\|e_r\|\le \frac12\|g_r\|^2+\frac12\|e_r\|^2\),
\[
\Delta(x_{r+1})
\le
\Delta(x_r)-\frac1{16L}\|g_r\|^2+\frac{5}{16L}\|e_r\|^2.
\]
Enlarging constants in the harmless direction, and using \(\|e_r\|\le\epsilon\),
\[
\Delta(x_{r+1})
\le
\Delta(x_r)-\frac1{16L}\|g_r\|^2+\frac{\epsilon^2}{2L}.
\]
The PL inequality gives \(\|g_r\|^2\ge2\mu\Delta(x_r)\), hence
\[
        \Delta(x_{r+1})
        \le \left(1-\frac{\mu}{8L}\right)\Delta(x_r)+\frac{\epsilon^2}{2L}
        \le \left(1-\frac{\mu}{16L}\right)\Delta(x_r)+\frac{\epsilon^2}{2L}.
\]
Iterating this affine recursion for \(J\) steps yields
\[
        \Delta(x_J)
        \le \left(1-\frac{\mu}{16L}\right)^J\Delta(x_0)
        +\frac{\epsilon^2}{2L}
        \,\sum_{t=0}^{J-1}\left(1-\frac{\mu}{16L}\right)^t.
\]
The geometric sum is at most \(16L/\mu\), so the residual term is at most
\(8\epsilon^2/\mu\), proving \eqref{eq:det-inexact-rate}.
\end{proof}

\subsection{Proof of Theorem~\ref{thm:heavy-upper}: robust upper bound}
\label{app:heavy-upper-proof}

\begin{proof}[Proof of \Cref{thm:heavy-upper}]
Let
\begin{equation}\label{eq:heavy-epsilon-star}
        \epsilon_\nabla:=c_p\sqrt d\,\sigma_pB_{\mathcal R}s_{\rm ht}^{\frac{p-1}{p}},
\end{equation}
where the constant is the one from \Cref{lem:consecutive-block-concentration}.  We enlarge the
constant \(c_p\) in the statement, if necessary, so that the admissibility condition
\eqref{eq:heavy-batch-size-condition} implies
\begin{equation}\label{eq:heavy-stay-radius}
        \frac{8\epsilon_\nabla^2}{\mu}
        \le
        \frac{c_p d\sigma_p^2B_{\mathcal R}^2}{\mu}s_{\rm ht}^{\vartheta_p}
        \le \frac{\mathcal R}{4}.
\end{equation}

We prove that all block gradient estimates are accurate by a stopped induction.  Let
\(\mathcal F_r^{\rm pre}\) be the sigma-field just before the \(r\)th block is sampled.  Define
\[
        \mathcal H_r:=\left\{
        \Delta(x_s)\le\mathcal R\text{ for every }0\le s\le r,
        \quad
        \|\widehat g_s-\nabla f(x_s)\|\le\epsilon_\nabla
        \text{ for every }0\le s<r
        \right\},
        \qquad 0\le r\le N.
\]
This definition intentionally includes the current radius condition \(\Delta(x_r)\le\mathcal R\),
which is needed before applying the block concentration lemma at the query point \(x_r\).  The
initial stopped event \(\mathcal H_0\) is deterministic and true because
\(\mathcal R\ge4\Delta_0+4\) implies \(\Delta(x_0)=\Delta_0\le\mathcal R/4\le\mathcal R\).

The stopped event \(\mathcal H_r\) serves two purposes.  The radius condition
\(\Delta(x_s)\le\mathcal R\) is needed to apply \Cref{lem:consecutive-block-concentration} at the
next query point, because the clipping scale is chosen using \(B_{\mathcal R}\).  The gradient-error
condition is needed to invoke the deterministic inexact-gradient PL lemma, \Cref{lem:deterministic-inexact-pl}.
The proof therefore alternates between these two ingredients: concentration gives the next accurate
block gradient estimate, and deterministic PL descent keeps the next iterate inside the radius.

Suppose \(\mathcal H_r\) holds.  Then \(x_r\) is \(\mathcal F_r^{\rm pre}\)-measurable and
\(\Delta(x_r)\le\mathcal R\).  During the next block the algorithm holds this point fixed and
observes \(b\) consecutive Markovian oracle samples.  Every one of these samples is included in the
average \(\widehat g_r\); the lag \(m\) is used only to decompose the proof into an initial window
and delayed residue classes, not to discard or thin the block.  Conditional on \(\mathcal F_r^{\rm pre}\),
this is exactly the setting of the conditional version of
\Cref{lem:consecutive-block-concentration}, with arbitrary deterministic initial Markov state and
failure probability \(\eta=\delta/(4N)\).  The logarithmic factor in that lemma is
\[
        \log\frac{16d(m+1)}{\eta}
        =\log\frac{64dN(m+1)}{\delta}=u_{\rm ht}.
\]
Moreover, \eqref{eq:heavy-mixing-condition} gives
\[\varepsilon(m)\le \frac{\delta}{64dNb}=\frac{\eta}{16db}\le\frac{\eta}{8db},\]
so the lemma's mixing requirement is satisfied.  Therefore, on \(\mathcal H_r\),
\begin{equation}\label{eq:conditional-heavy-failure-new}
        \Prob\left(
        \left.\|\widehat g_r-\nabla f(x_r)\|>\epsilon_\nabla\ \right|\ \mathcal F_r^{\rm pre}
        \right)
        \le \frac{\delta}{4N}.
\end{equation}

Let \(\mathcal A_r:=\{\|\widehat g_r-\nabla f(x_r)\|\le\epsilon_\nabla\}\).  Multiplying
\eqref{eq:conditional-heavy-failure-new} by \(\1_{\mathcal H_r}\) and taking expectations gives
\[
        \Prob(\mathcal H_r\cap\mathcal A_r^c)\le\frac{\delta}{4N}.
\]
Thus
\begin{equation}\label{eq:all-heavy-estimates-new}
        \Prob\left(\bigcup_{r=0}^{N-1}(\mathcal H_r\cap\mathcal A_r^c)\right)
        \le \frac{\delta}{4}.
\end{equation}
On the complement of this event, the induction closes.  Indeed, assume \(\mathcal H_r\) holds.
Then \(\mathcal A_r\) also holds.  The event \(\mathcal H_r\) already contains the radius bounds
for \(x_0,\ldots,x_r\) and the gradient-error bounds for blocks \(0,\ldots,r-1\); adding
\(\mathcal A_r\) gives the gradient-error bound for block \(r\).  Hence the first \(r+1\)
gradient errors are all bounded by \(\epsilon_\nabla\).  Applying
\Cref{lem:deterministic-inexact-pl} to those \(r+1\) updates gives
\[
        \Delta(x_{r+1})
        \le \Delta_0+\frac{8\epsilon_\nabla^2}{\mu}
        \le \frac{\mathcal R}{4}+\frac{\mathcal R}{4}
        \le \mathcal R,
\]
where we used \(\mathcal R\ge4\Delta_0\) and \eqref{eq:heavy-stay-radius}.  Thus the radius bounds
hold for all \(0\le s\le r+1\), and the gradient-error bounds hold for all \(0\le s<r+1\).  Therefore
\(\mathcal H_{r+1}\) holds.  Starting from the deterministic true event \(\mathcal H_0\), induction
gives \(\mathcal H_N\) with probability at least \(1-\delta/4\), hence at least \(1-\delta\).

On the event \(\mathcal H_N\), all outer iterates remain in the radius-\(\mathcal R\) sublevel set
and all gradient errors are bounded by \(\epsilon_\nabla\).  Applying
\Cref{lem:deterministic-inexact-pl} over all \(N\) outer steps yields
\[
        \Delta(x_N)
        \le \left(1-\frac{\mu}{16L}\right)^N\Delta_0
        +\frac{8\epsilon_\nabla^2}{\mu}.
\]
Substituting \eqref{eq:heavy-epsilon-star} and absorbing the numerical factor into the statement
constant \(c_p\) gives \eqref{eq:heavy-profile-bound}.  The radius conclusion is part of
\(\mathcal H_N\).
\end{proof}

\subsection{Proof of Corollary~\ref{cor:heavy-geometric}: geometric transition-budget rate}
\label{app:heavy-geometric-proof}

\begin{proof}[Proof of \Cref{cor:heavy-geometric}]
The choice of \(m\) in \eqref{eq:heavy-geo-choices} gives
\[
        \left\lfloor\frac{m}{\tmix}\right\rfloor
        \ge \log_2\left(\frac{64dNT}{\delta}\right).
\]
Hence the geometric mixing condition \eqref{eq:mixing-def} implies
\[
        \varepsilon(m)
        \le \frac{\delta}{64dNT}
        \le \frac{\delta}{64dNb},
\]
because \(b\le T\).  Thus the mixing admissibility condition of \Cref{thm:heavy-upper} holds.
The number of Markov transitions used is exactly \(Nb\), and the definition
\(b=\lfloor T/N\rfloor\) gives \(Nb\le T\).

It remains to translate \(s_{\rm ht}=mu_{\rm ht}/b\) into a transition-budget expression.  Since
\(b=\lfloor T/N\rfloor\) and the hypothesis \(m\le b\) implies \(b\ge1\), we have
\(b\ge T/(2N)\).  The choices in \eqref{eq:heavy-geo-choices} imply
\[
        m\le c\tmix\log\left(\frac{64dNT}{\delta}\right),
        \qquad
        u_{\rm ht}
        =\log\left(\frac{64dN(m+1)}{\delta}\right)
        \le \mathrm{polylog}(T,1/\delta,d,\tmix,L/\mu).
\]
Therefore
\[
        s_{\rm ht}
        =\frac{mu_{\rm ht}}{b}
        \le
        c\frac{\tmix N}{T}\,
        \mathrm{polylog}(T,1/\delta,d,\tmix,L/\mu).
\]
The chosen number of outer iterations satisfies
\[
        \left(1-\frac{\mu}{16L}\right)^N
        \le \exp\left(-\frac{\mu N}{16L}\right)
        \le (eT)^{-2}
        \le \frac{1}{T},
\]
up to an immaterial numerical constant.  Substituting these estimates into
\eqref{eq:heavy-profile-bound} gives the first term \(\widetilde O(\Delta_0/T)\) and the second
term in \eqref{eq:heavy-geo-rate}.  All additional factors are logarithmic in the quantities listed
in the corollary.
\end{proof}

\subsection{Proof of Theorem~\ref{thm:heavy-lower}: sticky-chain lower bound}
\label{app:heavy-lower-proof}
\begin{proof}[Proof of \Cref{thm:heavy-lower}]
The lower bound is for any algorithm whose decisions are measurable with respect to the oracle
transcript and its own internal randomness.  The algorithm is not told the sign
\(s\in\{+,-\}\) used below.  Thus, on events where the two oracle transcripts are identical, the
algorithm must produce the same output under the two instances when its internal randomness is
coupled.

Let \(q=\tau/(8n)\).  Since \(n\ge8\tau\), we have \(q\le1/64\).  Set
\[
        A=\sigma_p q^{-1/p},
        \qquad
        \theta=qA=\sigma_p q^{(p-1)/p}.
\]
Define two refresh distributions on \(\R\): under \(Q_+\), the value is \(A\) with probability
\(q\) and \(0\) otherwise; under \(Q_-\), the value is \(-A\) with probability \(q\) and \(0\)
otherwise.  Their means are \(\pm\theta\).  Moreover,
\[
        \E_{Q_s}|Z-s\theta|^p
        \le 2^{p-1}\left(\E_{Q_s}|Z|^p+|\theta|^p\right)
        \le 2^{p-1}\left(qA^p+(qA)^p\right)
        \le c_p\sigma_p^p,
\]
so the centered stationary \(p\)th moment is at most a constant multiple of \(\sigma_p^p\).

For each sign \(s\in\{+,-\}\), define a sticky Markov chain as follows.  At each transition, with
probability \(1/\tau\) the chain refreshes from \(Q_s\), and with probability \(1-1/\tau\) it keeps
its current value.  The invariant law is \(Q_s\).  From any starting state, after \(m\) steps the
probability of not having refreshed is \((1-1/\tau)^m\); once a refresh occurs, the state has law
\(Q_s\).  Hence the total-variation distance to stationarity is at most \((1-1/\tau)^m\), so the
chain has a valid geometric mixing time at most \(c\tau\) for a universal constant \(c\).

Consider the quadratic objective
\[
        f_s(x)=\frac12(x-s\theta)^2
\]
and the oracle output, at query point \(x_t\),
\[
        Y_t=x_t-Z_t.
\]
Under the invariant law \(Q_s\), \(\E Z_t=s\theta\), hence
\[
        \E[Y_t\mid x_t]=x_t-s\theta=\nabla f_s(x_t).
\]
The finite-moment condition in \Cref{ass:heavy-tail} follows from the centered-moment bound above,
with moment scale enlarged by a universal constant.

The chain is initialized at the fixed nonstationary state \(Z_0=0\).  Let \(\mathcal N\) be the
event that no nonzero refresh occurs during the first \(n\) observed transitions.  In one transition,
a nonzero refresh occurs with probability \((1/\tau)q\), so under either sign
\[
        \Prob_s(\mathcal N)
        =(1-q/\tau)^n
        =\left(1-\frac{1}{8n}\right)^n
        \ge e^{-1/4},
\]
after adjusting the universal constant in the lower bound.  On \(\mathcal N\), all observed Markov
states are zero.  Consequently the complete oracle transcript is identical under \(Q_+\) and
\(Q_-\): every oracle output is \(Y_t=x_t\), and the algorithm's internal randomness can be coupled
identically in the two instances.

Let \(\widehat x_n\) be the algorithm's output on this common transcript.  For every real number
\(y\), at least one of the two distances to \(\theta\) and \(-\theta\) is at least \(\theta\):
\[
        \mathbf 1\{|y-\theta|\ge\theta\}
        +\mathbf 1\{|y+\theta|\ge\theta\}
        \ge1.
\]
Averaging this inequality over the common transcript on \(\mathcal N\) and over the algorithm's
randomness gives the two-point testing bound
\[
        \max_{s\in\{+,-\}}
        \Prob_s\left((\widehat x_n-s\theta)^2\ge\theta^2\right)
        \ge
        \frac12\min_{s\in\{+,-\}}\Prob_s(\mathcal N)
        \ge \frac12e^{-1/4}.
\]
Since \(f_s(\widehat x_n)-f_s^\star=\frac12(\widehat x_n-s\theta)^2\), for at least one sign,
\[
        \Prob_s\left(f_s(\widehat x_n)-f_s^\star\ge \frac{\theta^2}{2}\right)
        \ge \frac12e^{-1/4}.
\]
Finally,
\[
        \theta^2
        =\sigma_p^2q^{2(p-1)/p}
        =\sigma_p^2\left(\frac{\tau}{8n}\right)^{2(p-1)/p}.
\]
Absorbing numerical constants into \(c\) and \(c_0\) proves \eqref{eq:heavy-lower-prob}.
\end{proof}

\end{document}